\documentclass[12pt]{article}

\usepackage[doublespacing,numericcite]{mystyle}
\usepackage{array}
\usepackage{tabularx}
\usepackage{titletoc}
\graphicspath{{./}}

\newcolumntype{Y}{>{\raggedright\arraybackslash}X}
\providecommand{\vbeta}{{\bm{\beta}}}
\providecommand{\vgamma}{{\bm{\gamma}}}
\providecommand{\vphi}{{\bm{\phi}}}
\providecommand{\evgamma}{{\gamma}}
\providecommand{\evphi}{{\phi}}

\providecommand{\vtau}{{\bm{\tau}}}

\providecommand{\vT}{\bm{T}}

\providecommand{\vDelta}{{\bm{\Delta}}}

\definecolor{darkpowderblue}{rgb}{0.0, 0.2, 0.6}
\newcommand{\authorfootref}[2]{\textsuperscript{\hyperlink{#1}{\textcolor{black}{#2}}}}
\newcommand{\authoraffil}[3]{\textsuperscript{\hypertarget{#1}{#2}}#3}
\titlecontents{suppsection}
  [0pt]
  {\normalsize}
  {\contentslabel{2em}}
  {}
  {\titlerule*[0.6em]{.}\contentspage}
\titlecontents{suppsubsection}
  [2em]
  {\small}
  {\contentslabel{2.8em}}
  {}
  {\titlerule*[0.6em]{.}\contentspage}
\titlecontents{suppsubsubsection}
  [4.8em]
  {\small}
  {\contentslabel{3.6em}}
  {}
  {\titlerule*[0.6em]{.}\contentspage}
\newcommand{\printsupplementarycontents}{%
  \begingroup
  \spacingset{1.0}
  \hypersetup{linkcolor=black}
  \begin{center}
  \begin{minipage}{0.9\textwidth}
  \noindent{\bf Contents}\par\smallskip
  \printcontents[supplement]{supp}{1}[3]{}
  \end{minipage}
  \end{center}
  \endgroup
}

\title{First-Order Efficiency for Probabilistic Value Estimation via A Statistical Viewpoint}
\addauthor{%
  {\large
  Ziqi Liu\authorfootref{fn:cmu}{1},
  Kiljae Lee\authorfootref{fn:osu}{2},
  Yuan Zhang\authorfootref{fn:osu}{2},
  Weijing Tang\authorfootref{fn:cmu}{1} \authorfootref{fn:corresponding}{*}%
  }%
}{%
  \normalsize
  \authoraffil{fn:cmu}{1}{Department of Statistics \& Data Science, Carnegie Mellon University.}\\
  \authoraffil{fn:osu}{2}{Department of Statistics, The Ohio State University.}%
}
\date{}

\begin{document}

\maketitle
\begingroup
\renewcommand{\thefootnote}{*}
\footnotetext{\hypertarget{fn:corresponding}{Corresponding author email: \texttt{weijingt@andrew.cmu.edu}.}}
\endgroup

\begingroup
\spacingset{1.25}
\begin{abstract}
Probabilistic values, including Shapley values and semivalues, provide a model-agnostic framework to attribute the behavior of a black-box model to data points or features, with a wide range of applications including explainable artificial intelligence and data valuation. 
However, their exact computation requires utility evaluations over exponentially many coalitions, making Monte Carlo approximation essential in modern machine learning applications. 
Existing estimators are often developed through different representation strategies, including weighted averages, self-normalized weighting, regression adjustment, and weighted least squares. 
Our key observation is that these seemingly distinct constructions share a common first-order expansion, in which the leading term is determined by the sampling law and a working surrogate function. 
This first-order representation yields an explicit expression for the leading mean squared error (MSE), which characterizes how the sampling law and the surrogate jointly determine statistical efficiency. 
Guided by this criterion, we propose an \emph{Efficiency-Aware Surrogate-adjusted Estimator (EASE)} that directly chooses the sampling law and  surrogate  to minimize the first-order MSE.
We demonstrate that EASE consistently outperforms existing estimators for various probabilistic values. 


\end{abstract}
\endgroup

\newpage

\section{Introduction}

Modern machine learning has created a growing class of attribution problems, in which one seeks to quantify the contribution of individual units to a collective outcome. 
In data attribution, the units are training examples and the outcome is the predictive performance of a model trained on a selected subset of data\citep{Jia2019-wa,Kwon2021-bw,Wang2022-iz}. 
In explainable artificial intelligence, the units are input features and the outcome is the prediction of a black-box model using a selected subset of features~\citep{Lundberg2017-gl}. 
Similar attribution questions also arise when the units are data sources, modalities, model components, or other contributors whose joint use determines a downstream utility. 
Although these problems arise in different contexts, they share a similar structure: a utility function is evaluated on subsets of units, and the goal is to assign each unit a contribution score.

The Shapley value~\citep{Shapley1953-ru} provides a canonical axiomatic framework for such attribution problems. 
It is the unique attribution rule satisfying four axioms: \textit{efficiency, symmetry, linearity}, and the \textit{null-player property}; we provide the formal statement of these axioms in Appendix~\ref{app:axioms}.
Probabilistic values extend the Shapley framework by relaxing the \textit{efficiency} and \textit{symmetry} axioms, thereby covering a broader class of attribution rules, including semivalues~\citep{Dubey1981-if, Weber1988-zf}.
Formally, for a game with \(n\) players, let \([n]:={1,\dots,n}\), and let \(u:2^{[n]}\to\mathbb R\) denote the utility function. Depending on the application, \(u(S)\) may represent the predictive performance of a model trained on the data points in \(S\), or the output of a model evaluated using only the features in \(S\). 
Given~\(u\), a probabilistic value assigns to each player \(i\in[n]\) a weighted average of its marginal contributions over all coalitions that do not contain \(i\):
\begin{equation}
    \evphi_i(u)
    =
    \sum_{S \subseteq [n]\setminus\{i\}}
    \alpha_{i}^{(n)}(S)
    \bigl\{
        u(S\cup\{i\}) - u(S)
    \bigr\},
\label{eq:semivalue-def}
\end{equation}
where the weights are nonnegative and satisfy \(\sum_{S\subseteq[n]\setminus \{i\}} \alpha_{i}^{(n)}(S)=1\).
When the weights depend on \(S\) only through the coalition size \(|S|\), 
the probabilistic value reduces to a semivalue. As a  further special case, choosing \(\alpha_{i}^{(n)}(S)=\{n\binom{n-1}{|S|}\}^{-1}\) recovers the Shapley~value.

Despite their broad applicability for attribution problems, probabilistic values are often difficult to use in practice because exact computation is combinatorial.
Direct evaluation of~\eqref{eq:semivalue-def} requires utility evaluations over up to all \(2^n\) coalitions. Moreover, in modern machine learning applications,  a single utility evaluation may itself be expensive. 
In data attribution, it can require retraining a model on a selected subset of training data, while in feature attribution it can require repeated inference procedure under different feature subsets.
Exact computation therefore becomes infeasible even for moderately large \(n\).

Monte Carlo approximation has consequently become the standard computational strategy, and a broad literature has developed estimators for probabilistic values; see, for example, the surveys~\citep{Chen2023-ka,Deng2025-tp}.
Existing estimators can be organized according to how they represent the target value under a sampling distribution.
The first family represents the target directly as a weighted average of coalition utilities, which leads to weighted sample means~\citep{Covert2021-di,Fumagalli2023-lo}, self-normalized weighted sample means~\citep{Wang2022-iz,Li2024-gp,Kolpaczki2024-eq}, and, more recently, regression-adjusted estimators~\citep{Witter2025-gz}.
The second family represents the target through a weighted least-squares (WLS) projection of the utility function onto a structured feature space. This class includes empirical WLS estimators for Shapley values~\citep{Lundberg2017-gl,Musco2024-cf,Fumagalli2026-hu,Chen2026-ai} and for general semivalues~\citep{Li2024-yr}.

Although these estimators appear to arise from different constructions, our key observation is that they share a common first-order error structure.
For both weighted-average and least-squares estimators, the first-order error term can be written as an augmented inverse-probability-weighted influence term determined by two design components: the sampling law and a working surrogate function (see Section~\ref{sec:regular} for details).
The surrogate serves as an approximation to the utility function, which may be specified explicitly, as in regression-adjusted estimators, or induced implicitly, as in self-normalized and WLS estimators.
This representation gives a unified first-order characterization of many existing estimators and, more importantly, yields an explicit expression for their leading mean squared error (MSE). 

Specifically, as the Monte Carlo sample size \(m\to\infty\), the first-order MSE is given by the variance of the corresponding residual influence term divided by \(m\).
For probabilistic-value estimation, however, the classical asymptotic regime in \(m\) is not the primary regime of interest. 
The space of possible coalitions has size \(2^n\), whereas feasible sampling budgets are typically much smaller than \(2^n\) and often grow only polynomially with \(n\). Thus, the central question is whether the first-order approximation remains valid at the polynomial sample sizes used in practice. 
To address this gap, we establish non-asymptotic remainder bounds for representative weighted-average and WLS estimators (see Section~\ref{sec:first-order-risk} for details). 
Our results show that the first-order MSE approximation has a small relative error when \(m=\Omega(n\operatorname{polylog}(n))\), aligning with the regime studied in the literature~\citep{Li2024-yr, Li2024-gp, Musco2024-cf}. These theoretical results justify using the first-order MSE as a practical design criterion and make explicit how the sampling law and the surrogate jointly determine statistical efficiency.

Motivated by this criterion, we propose the \emph{Efficiency-Aware Surrogate-adjusted Estimator (EASE)}, which directly targets to minimize the first-order MSE by choosing the surrogate and the sampling law.
This optimization problem is challenging for two reasons. 
First, the objective depends on the unknown utility function, and therefore must be learned from sampled utility evaluations. 
Second, without structural restrictions, optimizing the sampling law would require searching over all distributions on \(2^{[n]}\), which is computationally infeasible.
EASE makes this design problem tractable through a two-stage procedure. 
The initialization stage uses pilot samples to fit an initial surrogate and construct a residual-aware sampling law within a structured class. 
The estimation stage then draws new coalitions from the learned law and constructs a cross-fitted augmented inverse-probability-weighted estimator, with the surrogate chosen to minimize the empirical first-order MSE criterion.
Across a range of probabilistic-value targets, our numerical results show that EASE achieves lower MSE than existing estimators under the same utility-evaluation budget.

\section{Problem Setup for Probabilistic Value Estimation}
\label{sec:formulation}

In this section, we formulate the Monte Carlo estimation problem for a broad class of probabilistic-value targets. 
This framework covers the estimation of individual probabilistic values as well as more general aggregate targets, such as the total attribution assigned to a group of players. 
Let $\vphi(u)=(\evphi_1(u),\dots,\evphi_n(u))^\top$, where $\evphi_i(u)$ is the probabilistic value of player $i$ defined in \eqref{eq:semivalue-def}. 
For any coefficient vector $\va=(\eva_1,\dots,\eva_n)^\top \in \mathbb{R}^n$, define the target
\begin{equation}
    \tau_{\va}(u) := \va^\top \vphi(u) = \sum_{i=1}^n \eva_i \evphi_i(u).
    \label{eq:tau-a-def}
\end{equation}
If $\va=\ve_i \in \mathbb{R}^n$ denotes the $i$-th standard basis vector, then the target \(\tau_{\va}(u)=\evphi_i(u)\) recovers the probabilistic value of player $i$.
More generally, if $\va$ is the indicator vector of a subset of players, then $\tau_{\va}(u)$ is the total probabilistic value over that subset, which is useful in group data attribution applications~\citep{Lee2025-ht}.

We estimate $\tau_{\va}(u)$ using Monte Carlo sampling. 
In the main text, we focus on the single-coalition sampling scheme.
Let $q$ be a sampling distribution on $2^{[n]}$, and suppose that we observe i.i.d.\ coalitions
$S_1,\dots,S_m \sim q$,
together with their exact utility values.  
The observed data are therefore
$ \mathcal D_m=\{(S_t,u(S_t))\}_{t=1}^m$. 
This sampling regime is widely used in practice because of its flexibility. Once the sampled coalitions have been evaluated, the same data can be reused to estimate different linear summaries. Changing the target from \(\tau_{\va}\) to \(\tau_{\vb}\) only changes the coefficient vector from \({\va}\) to \({\vb}\). 
Structured sampling schemes such as bundled coalition designs are discussed in Section~\ref{sec:discussion} and Appendix~\ref{app:bundled}. 

For a fixed target \(\tau_{\va}\) and a given sampling law \(q\), there exists more than one way to represent the target as a functional of the observed-data distribution. 
These different representations motivate different Monte Carlo estimators. 
We next revisit two common representations that underlie many consistent estimators for probabilistic-value targets. 

\textbf{Weighted average representation.} 
    The first representation writes the target directly as a weighted average of coalition utilities. For each player $i$, define $\rho_i(S)
        :=
        \alpha_i^{(n)}(S\setminus\{i\})\1\{i\in S\}
        -
        \alpha_i^{(n)}(S)\1\{i\notin S\},$ for 
        $S\subseteq [n].$
    Then $\evphi_i(u) = \sum_{S \subseteq [n]} \rho_i(S)u(S)$.
    Therefore, any linear target \(\tau_{\va}(u)\) in~\eqref{eq:tau-a-def} can be written as
    \begin{equation}
        \tau_{\va}(u)
        =
        \sum_{S \subseteq [n]} \rho_{\va}(S)u(S),
        \qquad
        \rho_{\va}(S):=\sum_{i=1}^n \eva_i \rho_i(S).
        \label{eq:tau-a-linear}
    \end{equation}
    Under single-coalition sampling from the law \(q\), this representation becomes
    \begin{equation}
        \tau_{\va}(u)
        =
        \E_q\!\left[\gamma_{\va, q}(S)u(S)\right],
        \qquad
        \gamma_{\va,q}(S):=\frac{\rho_{\va}(S)}{q(S)},
        \label{eq:tau-a-ipw}
    \end{equation}
    provided that \(q(S)>0\) whenever \(\rho_{\va}(S)\neq 0\). Thus, the target is represented as a weighted expectation of the observed utility under the sampling distribution~$q$. 
    This representation naturally leads to weighted sample-mean estimators~\citep{Covert2021-di,Fumagalli2023-lo} and underlies related self-normalized~\citep{Wang2022-iz,Kolpaczki2024-eq, Li2024-gp} and regression-adjusted estimators~\citep{Witter2025-gz}. 

\textbf{Weighted least-squares representation.} 
    The second representation writes the target through a weighted least-squares projection of the utility function. 
    For Shapley values~\citep{Lundberg2017-gl,Musco2024-cf,Fumagalli2026-hu,Chen2026-ai} and more general semivalues~\citep{Li2024-yr}, one can find a nonnegative weight function \(w:2^{[n]}\to[0,\infty)\), a feature map \(\vz:2^{[n]}\to\mathbb R^{d_\vz}\), and a readout vector \(\vc_{\va}\in\mathbb R^{d_\vz}\) such that
    \begin{equation}
        \tau_{\va}(u)=\vc_{\va}^\top\vbeta^\star(u),
        \qquad
        \text{for every }\vbeta^\star(u)\in
        \argmin_{\vbeta \in \mathbb R^{d_\vz}}
        \sum_{S\subseteq[n]}
        w(S)\{u(S)-\vz(S)^\top\vbeta\}^2.
        \label{eq:wls-id}
    \end{equation}
    In this representation, the pre-specified tuple \((w,\vz,\vc_{\va})\) determines how the utility function is projected and how the target is read out from the projected coefficients. 
    This tuple is not necessarily unique. 
    For example, KernelSHAP~\citep{Lundberg2017-gl} and PolySHAP~\citep{Fumagalli2026-hu} both represent Shapley values through WLS problems, but they use different tuples. 
    Under single-coalition sampling from \(q\), the same population criterion can be rewritten as
    \begin{equation}
        \tau_{\va}(u)=\vc_{\va}^\top\vbeta^\star(u),
        \qquad
        \text{for every }\vbeta^\star(u)\in
        \argmin_{\vbeta\in\mathbb R^{d_\vz}}
        \E_q\left[
            \frac{w(S)}{q(S)}
            \{u(S)-\vz(S)^\top\vbeta\}^2
        \right],
        \label{eq:wls-id-q}
    \end{equation}
    provided that \(q(S)>0\) whenever \(w(S)>0\).
    Unlike the weighted-average representation, this form does not write the target directly as a weighted average of \(u(S)\). 
    Instead, \(u\) is first projected onto a structured linear class and the target is then obtained by applying the readout \(\vc_{\va}\).
    The corresponding Monte Carlo estimator replaces the population WLS objective in \eqref{eq:wls-id-q} by its empirical counterpart and then applies the same readout to the fitted coefficient vector. 

The two representations lead to estimators with seemingly different structures. 
The weighted-average estimators reweight sampled utilities, possibly with normalization or regression adjustments. 
The WLS estimators instead solve a weighted regression problem and plug in the fitted coefficients to the target. 
Nevertheless, as we shall show in the next section, these estimators admit a common first-order expansion, whose leading term is an augmented weighted-average estimator indexed by a working surrogate. The surrogate may be specified explicitly, as in regression adjustment, or induced implicitly by self-normalization or least-squares fitting. 

\section{First-Order Representation via Working Surrogates}
\label{sec:regular}
The previous section describes two representations of the same target.
We now show that the corresponding estimators, despite their different forms, share a common first-order structure. This structure is indexed by a working surrogate function \(h:2^{[n]}\to\mathbb R\), which serves as an approximation to the utility function $u$.
For any working surrogate \(h\), let \(\tau_{\va}(h)\) denote the same probabilistic-value target applied to \(h\) in place of \(u\).
Since \(\E_q[\gamma_{\va,q}(S)f(S)]=\tau_{\va}(f)\) for any function \(f\), the target admits the decomposition \(\tau_{\va}(u)=\tau_{\va}(h)+\E_q[\gamma_{\va,q}(S)\{u(S)-h(S)\}]\).
Accordingly, the augmented inverse-probability-weighted term \(\tau_{\va}(h)+\gamma_{\va,q}(S)\{u(S)-h(S)\}\), which combines the target value of the working surrogate with an inverse-probability-weighted correction for the residual \(u-h\), has mean \(\tau_{\va}(u)\) under \(q\).
Define its centered version
\[
    \psi_h(S;u)
    :=
    \gamma_{\va,q}(S)\{u(S)-h(S)\}
    +
    \tau_{\va}(h)
    -
    \tau_{\va}(u),
\]
which is mean zero under \(q\).
We use this mean-zero quantity to define a regular class of estimators and to show that existing estimators fall into this class.

Let $\mathcal H:= \{h: 2^{[n]} \to \mathbb{R}\}$ be a working surrogate class consisting of candidate surrogates attainable by the estimation procedure.
We define regularity relative to \(\mathcal H\) under the sampling law $q$ through the following first-order  expansion.

\begin{definition}[Regular estimator sequence]\label{def:regular}
Let \(S_1,\ldots,S_m \overset{\mathrm{iid}}{\sim} q\), and let
\(\hat\tau_{\va, m}\) denote an estimator based on the sampled utility evaluations $ \mathcal D_m=\{(S_t,u(S_t))\}_{t=1}^m$.
A sequence of estimators $\{\hat\tau_{\va, m}\}_{m\ge 1}$ is \emph{regular} relative to  \(\mathcal H\) under \(q\) if there exists a sequence of nonrandom functions $h_m\in\mathcal H$ such that
\begin{equation}
    \hat\tau_{\va, m}(u)-\tau_{\va}(u)
    =
    \frac1m\sum_{j=1}^m \psi_{h_m}(S_j; u) + r_m \quad \text{with } \E_q(r_m^2)=o(m^{-1}).
    \label{eq:regular-expansion}
\end{equation}
\end{definition}
In the expansion \eqref{eq:regular-expansion}, the surrogate \(h_m\) is the population-level working surrogate associated with the estimator, and the mean-zero summand \(\psi_{h_m}\) determines the estimator's first-order behavior.
In this sense, \(\psi_h\) plays the role of an influence function, and we refer to it as the \emph{influence term} associated with the sampling law \(q\) and the working surrogate \(h\).
We first show the simplest case with the zero surrogate. The Horvitz--Thompson estimator~\citep{Horvitz1952-ex}, motivated by the weighted-average representation in~\eqref{eq:tau-a-ipw}, is
    \begin{equation}    
        \hat{\tau}_{\va, m}^{\mathrm{HT}}(u) = \frac{1}{m}\sum_{t=1}^m \gamma_{\va,q}(S_t)u(S_t).
        \label{eq:ht-est}
    \end{equation}
    By the definition that \(\psi_0(S;u)=\gamma_{\va,q}(S)u(S)-\tau_{\va}(u)\), we have $\hat{\tau}_{\va, m}^{\mathrm{HT}}(u) - \tau_{\va}(u) = \frac{1}{m}\sum_{t=1}^m \psi_{0}(S_t; u).$
    Thus, \(\hat{\tau}_{\va, m}^{\mathrm{HT}}\) is regular relative to \(\mathcal H=\{0\}\), with \(r_m=0\) for every \(m\). In the Shapley-value literature, this estimator coincides with unbiased KernelSHAP~\citep{Covert2021-di} and the order-one special case of SHAP-IQ~\citep{Fumagalli2023-lo}. 
    In our framework, both correspond to regular estimators with the zero surrogate.

Our framework also covers many estimators associated with nontrivial working surrogates.
The surrogate may be specified explicitly, as in regression-adjusted estimators, or induced implicitly by the estimation procedure, as in self-normalized weighted averages or WLS estimators. 
The following subsections identify the corresponding working surrogate classes and establish their regularity.

\subsection{Regression-Adjusted Estimators with Explicit Surrogates}
Regression-adjusted estimators for probabilistic values use an explicit surrogate $h$ to approximate the utility function and then correct the residual by Monte Carlo sampling~\citep{Witter2025-gz}.
They can be written in the augmented inverse-probability-weighted (AIPW) form~\citep{Rubin1974-bm}:
\begin{equation}\label{eq:aipw-est}
        \hat{\tau}_{\va, m}^{\mathrm{AIPW}}(u) = \tau_{\va}(h_m) + \frac{1}{m}\sum_{t=1}^m \gamma_{\va,q}(S_t)\{u(S_t)-h_m(S_t)\},
    \end{equation}
where $h_m$ is an explicitly specified surrogate. 
The estimator first computes the probabilistic value of the surrogate and then uses Monte Carlo sampling only to correct the residual. 

This form imposes a constraint that is specific to probabilistic value estimation.
The surrogate must not only be evaluable at sampled coalitions, but also have a tractable value $\tau_{\va}(h_m)$. 
Otherwise, computing $\tau_{\va}(h_m)$ would itself be as difficult as computing the original target $\tau_{\va}(u)$. 

Accordingly, let $\mathcal H^{\mathrm{sur}}$ be an explicit working surrogate class consisting of functions $h:2^{[n]}\to\mathbb R$ for which $\tau_{\va}(h)$ is computationally tractable; see Remark~\ref{rmk:working-surrogate-classes} for further discussion.
For any $h_m\in\mathcal H^{\mathrm{sur}}$, the surrogate in the first-order representation is exactly the surrogate used in the regression adjustment. 
Indeed, by the definition of $\psi_{h_m}$, we have $\hat{\tau}_{\va, m}^{\mathrm{AIPW}}(u)-\tau_{\va}(u)=\frac{1}{m}\sum_{t=1}^m \psi_{h_m}(S_t; u).$
Thus, $\hat{\tau}_{\va,m}^{\mathrm{AIPW}}$ is regular relative to $\mathcal H^{\mathrm{sur}}$, with $r_m=0$ for every $m$. 
The form in~\eqref{eq:aipw-est} is also the with-replacement analogue of the generalized-difference estimator in model-assisted survey sampling~\citep{Sarndal1992-fm}.

\begin{remarknum}[Explicit Working Surrogate Classes]\label{rmk:working-surrogate-classes}
    The class \(\mathcal H^{\mathrm{sur}}\) is a working class used for surrogate adjustment, rather than a correctly specified model class for $u$. In particular, we do not require that \(u\in \mathcal H^{\mathrm{sur}}\). 
    Instead, its choice is  constrained by two practical considerations.
    First,  \(\tau_{\va}(h)\) should be computable without excessive additional utility evaluations. Otherwise, even a useful approximation \(h\approx u\) may be impractical because evaluating the term \(\tau_{\va}(h)\) is itself too costly.
    A simple example is the additive surrogate
\(
h(S)=\xi_0+\sum_{i\in S}\xi_i,
\)
for which every probabilistic value assigns player \(i\) the constant marginal contribution \(\xi_i\), so
\(
\tau_{\va}(h)=\sum_{i=1}^n \eva_i\xi_i.
\)
More expressive examples include tree ensembles in TreeSHAP~\citep{Lundberg2018-xx}, RKHS models in RKHS-SHAP~\citep{Chau2022-rg}, and low-degree polynomial classes in PolySHAP~\citep{Fumagalli2026-hu}, for which probabilistic values can be computed in closed form or by specialized algorithms.
Second, the surrogate $h$ should be practically obtainable under the available evaluation budget. 
It may be specified a priori when structural knowledge of the utility function is available, or fitted from evaluated utilities. In the latter case, a highly expressive class \(\mathcal H^{\mathrm{sur}}\) can be unstable when the Monte Carlo budget is small. 
 These desiderata parallel the requirements for auxiliary information used in model-assisted survey sampling: auxiliary variables should have tractable population totals~\citep{Sarndal1992-fm}, and the working model is used for variance reduction rather than assumed to be correctly specified~\citep{Dagdoug2023-nl}.
\end{remarknum}

\subsection{Self-Normalized Weighted-Average Estimators}
Several existing estimators proposed for Shapley and probabilistic value estimation use stratified weighted averages~\citep{Wang2022-iz,Kolpaczki2024-eq,Li2024-gp}.
For example, the Banzhaf-value estimator of \citet{Wang2022-iz} normalizes the weighted averages separately over coalitions that contain or exclude a given player, while stratified SVARM \citep{Kolpaczki2024-eq} further stratifies by coalition size.
Although these normalizations arise naturally from stratified sampling schemes in their original formulations, they can be viewed as instances of a H{\'a}jek-type estimator~\citep{hajek1971comment}. 
As we  show below, their first-order behavior can also be represented through a working surrogate. In this case, the surrogate is not specified explicitly but induced by the partition used for normalization.

To formalize this connection, suppose the coalition space \(2^{[n]}\) is partitioned into \(K\) cells \(\mathscr C=\{C_1,\dots,C_K\}\). 
Denote the design probability of cell \(C_k\) by \(\pi_k:=\PP_q(S\in C_k)\) and the empirical frequency of cell \(C_k\) in the sample by \(\hat\pi_k = N_k/m\), where \(N_k = \sum_{t=1}^m \1\{S_t \in C_k\}\).
Using this partition, the Horvitz--Thompson (HT) estimator in \eqref{eq:ht-est} can be rewritten as
\[
    \hat{\tau}_{\va, m}^{\mathrm{HT}}(u)
    =
    \sum_{k:N_k>0}
    \hat\pi_k
    \left\{
    \frac{1}{N_k}
    \sum_{t:S_t\in C_k}
    \gamma_{\va,q}(S_t)u(S_t)
    \right\}.
\]
This representation reveals that the HT estimator weights each within-cell average by the \emph{random} empirical frequency \(\hat\pi_k\).
A H{\'a}jek-type estimator instead replaces these empirical frequencies with the known design probabilities:
\begin{equation}
    \hat\tau_{\va, m}^{\mathrm{Hajek}}(u)
    =
    \sum_{k:N_k>0}
    \pi_k
    \left\{
    \frac{1}{N_k}
    \sum_{t:S_t\in C_k}
    \gamma_{\va,q}(S_t)u(S_t)
    \right\}.
    \label{eq:hajek-est}
\end{equation}
This substitution eliminates the variance introduced by random cell frequencies, while preserving the information gathered from the sampled utilities within those cells. 
When \(K=1\), the partition consists of the entire coalition space, and the
H{\'a}jek-type estimator in~\eqref{eq:hajek-est} reduces to the HT estimator in \eqref{eq:ht-est}.

The estimator in~\eqref{eq:hajek-est} does not explicitly construct a surrogate. Nevertheless, its first-order linearization centers the weighted utility $\gamma_{\va,q}u$ within each cell.
This is equivalent to using a working surrogate whose weighted version $\gamma_{\va,q}h$ reproduces the
cellwise conditional mean of $\gamma_{\va,q}u$.
Accordingly, consider the working surrogate class
\begin{equation*}
    \mathcal{H}_{\mathscr{C}, \gamma_{\va,q}} := \left\{
        h: \gamma_{\va,q}(S)h(S) = \sum_{k=1}^K \omega_k \ind(S \in C_k) \text{ for some } \omega_1, \dots, \omega_K \in \mathbb{R}
    \right\}.
\end{equation*}
The following proposition shows that the H{\'a}jek-type estimator is regular relative to this implicitly induced working surrogate class.

\begin{proposition}[Regularity of H{\'a}jek-type Weighted Average] \label{prop:aug-hajek-regular}
    Suppose that, on each cell \(C_k\in\mathscr C\), the weight \(\gamma_{\va,q}\) is either identically zero or everywhere nonzero. 
    Let $\mu_k
        :=
        \E_q\!\left[
            \gamma_{\va,q}(S)u(S)\mid S\in C_k
        \right]$ if \(\pi_k>0\) and  \(\mu_k
        :=0\) otherwise.
        Define \(h_{\mathscr C}\) by setting
    \(h_{\mathscr C}(S)=\mu_k/\gamma_{\va,q}(S)\) for
\(S\in C_k\) on cells where \(\gamma_{\va,q}\not\equiv0\), and
\(h_{\mathscr C}(S)=0\) on cells where \(\gamma_{\va,q}\equiv0\). Then \(h_{\mathscr C}\in\mathcal H_{\mathscr C,\gamma_{\va,q}}\), and the H{\'a}jek-type estimator \(\hat\tau_{\va, m}^{\mathrm{Hajek}}(u)\) in
\eqref{eq:hajek-est} satisfies 
\begin{equation}
   \hat\tau_{\va, m}^{\mathrm{Hajek}}(u)-\tau_{\va}(u)
    =
    \frac1m\sum_{t=1}^m \psi_{h_{\mathscr C}}(S_t;u)
    +
    r_m,\quad \text{with } \E_q(r_m^2) = o(m^{-1}).
    \label{eq:hajek-regular}
\end{equation}
Therefore, \(\hat\tau_{\va, m}^{\mathrm{Hajek}}(u)\)  is regular relative to \(\mathcal H_{\mathscr C,\gamma_{\va,q}}\). 
\end{proposition}
This proposition gives a common interpretation of several existing self-normalized weighted-average estimators. 
The estimator of the Banzhaf value for player \(i\) in \citet{Wang2022-iz} corresponds to the partition $\mathscr C=\{\{S:i\in S\}, \{S:i\notin S\}\}$. 
Stratified SVARM~\citep{Kolpaczki2024-eq}
refines this partition by further splitting cells based on coalition size, i.e., $\mathscr C=\{\{S:i\in S, |S|=s\}, \{S:i\notin S, |S|=s\}: s=0,\dots,n\}$.
In a related vein, 
\citet{Li2024-gp} use a similar membership-size partition and normalization idea to estimate more general probabilistic values.
In each case, the normalization determines the surrogate appearing in their first-order~representation.

\subsection{Weighted Least-Squares Estimators}
Motivated by the WLS representation in~\eqref{eq:wls-id-q}, the widely used KernelSHAP estimator and related variants estimate the target by fitting an empirical WLS regression and then applying the fixed linear functional \(\vc_{\va}^\top\) to the fitted coefficients~\citep{Lundberg2017-gl,Musco2024-cf,Chen2026-ai,Fumagalli2026-hu}. 
In our framework, the corresponding population WLS projection serves as the working surrogate for $u$.
To make this precise, denote the feature span by
\[
\mathcal H_{\mathrm{WLS}}
=
\left\{
h:
h(S)=\vz(S)^\top\vbeta
\text{ for some }
\vbeta\in\mathbb R^{d_\vz}
\right\}.
\]
Let \(h_w \in \argmin_{h\in\mathcal H_{\mathrm{WLS}}}
        \sum_{S \subseteq [n]} w(S)\{u(S)-h(S)\}^2\)
    be a population WLS projection of \(u\) onto the feature span.
Thus, unlike the explicit-surrogate estimators, the working surrogate is not
chosen directly. It is induced by the regression criterion used by the estimator.

The empirical version replaces the population WLS criterion with its sampled analogue.
Under limited observations, the sampled design may be rank deficient. To handle this issue and improve numerical stability, we study the ridge-penalized version, which is well posed and provides a stable object for the first-order analysis.
For \(\lambda>0\), define
\begin{equation}
    \hat{\tau}_{\va, m}^{\mathrm{WLS}, \lambda}(u)
    :=
    \vc_{\va}^\top \hat{\vbeta}_{m, \lambda},
    \quad \text{where } \hat{\vbeta}_{m, \lambda}
    :=
    \argmin_{\vbeta\in\mathbb R^{d_\vz}}
    \left\{
    \sum_{t=1}^m \frac{w(S_t)}{q(S_t)}\{u(S_t)-\vz(S_t)^\top\vbeta\}^2
    +
    \lambda \|\vbeta\|_2^2
    \right\}.
    \label{eq:wls-ridge-est}
\end{equation}
The following proposition shows that this estimator is regular relative to the feature-span surrogate class, with the working surrogate given by the population
WLS projection $h_w$.

\begin{proposition}
    [Regularity of Weighted Least-Squares Estimators] \label{prop:rr-regular}
    Let \(h_w\)  be a population WLS projection of \(u\) onto $\mathcal H_{\mathrm{WLS}}$.  Then
    \begin{equation*}
        \hat{\tau}_{\va, m}^{\mathrm{WLS}, \lambda}(u) - \tau_{\va}(u)
        =
        \frac{1}{m} \sum_{t=1}^m \psi_{h_w}(S_t; u)
        +
        r_m, \quad \text{with } \E_q(r_m^2) = o(m^{-1}).
    \end{equation*}
    Therefore, \(\hat{\tau}_{\va, m}^{\mathrm{WLS}, \lambda}(u)\)  is regular relative to \(\mathcal H_{\mathrm{WLS}}\). 
\end{proposition}
From this viewpoint, using a richer feature span enriches the set of working
approximations to $u$ available to the estimator. For example, \citet{Fumagalli2026-hu} use a more expressive feature span in the WLS regression, which can be interpreted as expanding $\mathcal H_{\mathrm{WLS}}$. 
Our focus next is complementary. Rather than expanding the surrogate class, we take a working
class as given and study how the choice of $q$ and $h$ affects first-order risk.

\section{Error Expansion and Design Principles} \label{sec:first-order-risk}

The regularity framework in Section~\ref{sec:regular} connects many estimator constructions to a common first-order expansion indexed by a sampling law \(q\) and a working surrogate~\(h\). 
In this section, we study the leading term in the mean squared error implied by this expansion and its joint dependence on \(q\) and \(h\). This dependence provides a design criterion for choosing sampling laws and surrogates
that reduce first-order risk.

If \(m\to\infty\) with the number of players $n$ fixed, the  expansion in \eqref{eq:regular-expansion} suggests that the leading MSE is the variance of \(\psi_h\) divided by \(m\) under some regularity conditions.
In probabilistic-value estimation, however, this fixed-$n$ asymptotic regime is not the main one of interest. 
The utility function is defined on \(2^n\) coalitions, while feasible sampling budgets are usually much smaller than \(2^n\), often only polynomial in \(n\)~\citep{Castro2009-yu, Jia2019-wa, Li2024-yr, Li2024-gp, Musco2024-cf}.
Although the remainder $r_m$ in \eqref{eq:regular-expansion} is lower order in $m$, its constants may depend on $n$. 
A first-order risk criterion is therefore informative only when this remainder is small relative to the leading term at the sample sizes used in practice.

We proceed in two steps. First, we derive the leading MSE term for a regular estimator and characterize its joint dependence on $q$ and $h$. 
This reveals that the sampling law  and surrogate induced by existing estimators do not necessarily minimize first-order risk. 
Second, we provide non-asymptotic conditions under which the leading term dominates the remainder for polynomial sampling budgets.

\subsection{First-Order MSE as a Design Criterion} \label{subsec:first-order-risk}

We state the results for the general setting in which the same sampled coalitions \(\mathcal{D}_m=\{(S_t,u(S_t))\}_{t=1}^m\) are used to estimate multiple linear transformations of probabilistic values, such as the probabilistic values of all players. The single target is a special case.

Let \(\tau_{\va_1}(u),\dots,\tau_{\va_d}(u)\) be the targets of interest, where \(d\) is independent of \(m\), and let \(\mA=[\va_1,\dots,\va_d]\in \R^{n \times d}\).
The target vector
\((\tau_{\va_1}(u),\dots,\tau_{\va_d}(u))^\top\) can then be written as
\(\vtau_{\mA}(u) := \mA^\top \vphi(u)\), where \(\vphi(u)\) is the vector of
probabilistic values.
Consider the estimator \(\hat{\vtau}_{\mA,m}(u)=\bigl(
        \hat\tau_{\va_1,m}(u),\dots,
        \hat\tau_{\va_d,m}(u)
    \bigr)^\top\) constructed from $\mathcal{D}_m$. 
Suppose each coordinate \(\hat\tau_{\va_j,m}\) is regular relative to the working surrogate class
\(\mathcal H_j\), for \(j=1,\dots,d\). Then there exists
\(h_{j}\in\mathcal H_j\) such that 
\begin{equation}
    \hat\tau_{\va_j,m}(u)-\tau_{\va_j}(u)
    :=
    \frac1m\sum_{t=1}^m
    \psi_{\va_j,q,h_{j}}(S_t;u)
    +
    r_{m,j},
    \qquad
    \E_q(r_{m,j}^2)=o(m^{-1}),
    \label{eq:coordinate-regular-expansion}
\end{equation}
where  \(r_{m,j}\) is the remainder term.  
Let \(\vr_m=(r_{m,1},\dots,r_{m,d})\) and \(\vh=(h_{1},\dots,h_{d})\).
Define
\begin{equation} \label{eq:variance-criterion}
    V(\mA;q,\vh)
    :=
    \sum_{j=1}^d \Var_q\!\left[
        \gamma_{\va_j,q}(S)\{u(S)-h_j(S)\}
    \right]
\end{equation}
as the summation of the variances of the coordinatewise influence terms $\psi_{\va_j,q,h_{j}}$.
The following lemma gives the first-order expansion of MSE. 

\begin{lemma}[First-order Error Expansion]
\label{lem:regular-mse-vector}
Assume that the coordinatewise expansion
\eqref{eq:coordinate-regular-expansion} holds for all \(1\le j\le d\). Suppose that \(V^*(\mA;q):= \inf_{h_j \in\mathcal H_j} V(\mA;q,\vh)>0\), then
\begin{equation}
    \left|
    \frac{
        \E_q\!\left[
            \left\|
                \hat{\vtau}_{\mA,m}(u)-\vtau_{\mA}(u)
            \right\|_2^2
        \right]
    }{
        V(\mA;q,\vh)/m
    }
    -1
    \right|
    \le
    2\sqrt{\frac{m \E_q(\|\vr_m\|_2^2)}{V(\mA;q,\vh)}}
    +
    \frac{m \E_q(\|\vr_m\|_2^2)}{V(\mA;q,\vh)}.
    \label{eq:mse-ratio-bound}
\end{equation}
Consequently, for any \(0<\epsilon\le1\), the relative error is at most
\(\epsilon\) whenever
\(
    \E_q(\|\vr_m\|_2^2)
    \le C \epsilon^2  \frac{V(\mA;q,\vh)}{m}
\)
for some universal constant \(C > 0\). 
\end{lemma}

The condition \(V^*(\mA;q)>0\) excludes degenerate cases in which the surrogate class can eliminate the first-order variance under the given sampling law. This condition is mild in practice. As discussed in Remark~\ref{rmk:working-surrogate-classes}, surrogate classes are typically constrained by computational and statistical feasibility and are not assumed to contain the true utility~function~\(u\).

Since \(\E_q(\|\vr_m\|_2^2)=\sum_{j=1}^d \E_q(r_{m,j}^2) = o(m^{-1})\), Lemma~\ref{lem:regular-mse-vector} implies that \(V(\mA;q,\vh)/m\) is the first-order MSE term.
This expansion motivates using first-order MSE as a design criterion. It suggests choosing \(q\) and \(\vh\) to reduce \(V(\mA;q,\vh)\). Existing estimators often induce surrogate adjustments through criteria internal to their construction, such as self-normalization or weighted least squares, or through generic distance metrics between \(u\) and \(h\). These choices need not minimize the first-order MSE criterion.

\subsection{Non-Asymptotic Bounds for the Remainder}
The bound in~\eqref{eq:mse-ratio-bound} shows that the accuracy of the first-order MSE approximation is governed by the scaled remainder \(m\E_q(\|\vr_m\|_2^2)\). The scale of this term depends on both the specific estimator design and the target. This subsection establishes explicit non-asymptotic upper bounds for \(\E_q(\|\vr_m\|_2^2)\) for two estimator classes: H{\'a}jek-type weighted averages and weighted least-squares (WLS) estimators. We then apply these bounds to representative Shapley-value estimators~\citep{Li2024-gp, Kolpaczki2024-eq, Lundberg2017-gl, Musco2024-cf} and show that, for \(0<\epsilon\le1\), the first-order MSE approximation has relative error at most \(\epsilon\) once \(m=\Omega(n\operatorname{polylog}(n)/\epsilon^2)\), which aligns with the regime of primary interest in the literature~\citep{Li2024-yr, Li2024-gp, Musco2024-cf}. 

\paragraph{H{\'a}jek-type weighted average estimators.} Consider the H{\'a}jek-type estimator for a scalar target in~\eqref{eq:hajek-est} with partition \(\mathscr C=\{C_k\}_{k=1}^K\), and let \(h_{\mathscr C}\) be the corresponding cellwise surrogate defined in Proposition~\ref{prop:aug-hajek-regular}.
Denote the number of cells with positive sampling probability by \(K_+:=|\{k:\pi_k>0\}|\), and the smallest nonzero cell probability by \(\pi_{\min}:=\min_{k:\pi_k>0}\pi_k\).
Define the average cellwise squared weight and the largest magnitude by
\(
    \Gamma_{2,\mathscr C}^2
    :=
    \frac{1}{K_+}
    \sum_{k:\pi_k>0}
    \E_q\!\left[\gamma_{\va,q}(S)^2\mid S\in C_k\right]
\)
and \(
    \Gamma_\infty
    :=
    \sup_{S:q(S)>0} |\gamma_{\va,q}(S)|,
\) respectively.

\begin{theorem}[Non-asymptotic remainder bound for H{\'a}jek-type estimators] \label{prop:non-asymptotic-hajek}
    Assume the conditions in Proposition~\ref{prop:aug-hajek-regular} hold.
    There exist universal constants \(C_1,C_2>0\) such that
    \begin{equation}
        \E_q(r_m^2)
        \le
        C_1 \|u\|_\infty^2
        \left\{
            \frac{K_+ \; \Gamma_{2,\mathscr C}^2}{m^2}
        +
        \Gamma_\infty^2 e^{-C_2 m \pi_{\min}}
        \right\}.
        \label{eq:aug-hajek-finite-remainder}
    \end{equation}
\end{theorem}

Theorem~\ref{prop:non-asymptotic-hajek} bounds the remainder by two terms.
The first term is the main polynomial term in \(m\). It reflects the cost of estimating separate normalizing factors across the active cells in \(\mathscr C\), and scales with both \(K_+\) and \(\Gamma_{2,\mathscr C}^2\). The second term controls the rare event that some active cell receives no samples. It decays exponentially in \(m\pi_{\min}\) and is governed by the largest weight \(\Gamma_\infty\).
We interpret the bound through two representative H{\'a}jek-type estimators, OFA~\citep{Li2024-gp} and Stratified SVARM~\citep{Kolpaczki2024-eq}, specialized to Shapley-value estimation through their corresponding partitions and sampling laws. Details are given in Appendix~\ref{app:non-asymptotic-hajek}.

\begin{corollary}[Remainder bound for OFA and Stratified SVARM]\label{corollary:hajek-ofa-svarm-rates}
    Consider estimating the Shapley value \(\evphi_i\) of a fixed player \(i\). There exist universal constants \(C_1,C_2>0\) such that
\[
    \E_q(r_m^2)
    \le
    C_1\|u\|_\infty^2
    \begin{cases}
        \displaystyle
        \frac{n\log n}{m^2}
        +
        n e^{-C_2 m/n^{3/2}},
        & \text{for OFA~\citep{Li2024-gp}},\\[1ex]
        \displaystyle
        \frac{n\log^2 n}{m^2}
        +
        \log^2 n\, e^{-C_2 m/(n\log n)},
        & \text{for Stratified SVARM~\citep{Kolpaczki2024-eq}}.
    \end{cases}
\]
\end{corollary}

The polynomial terms in Corollary~\ref{corollary:hajek-ofa-svarm-rates} determine the leading remainder scales. If \(\|u\|_\infty\) is bounded and the first-order MSE is nondegenerate, then, for any \(0<\epsilon\le1\), the first-order MSE approximation has relative error at most \(\epsilon\) once \(m=\Omega(n\log n/\epsilon^2)\) for OFA and \(m=\Omega(n\log^2 n/\epsilon^2)\) for Stratified SVARM. For OFA, this \(n\log n\) dependence implies the non-asymptotic query complexity required for its Shapley-value approximation guarantee (see Appendix~\ref{app:non-asymptotic-hajek} for more details).

\paragraph{WLS estimators.}
Consider a vector-valued extension of the WLS estimators. Suppose \(\vz(S)\in\mathbb R^{d_{\vz}}\) and the target vector \(\vtau_{\mA}(u)\) is represented through the WLS projection
\[
    \vtau_{\mA}(u)=\mC_{\mA}^\top\vbeta^\star(u), \quad \text{for every } \vbeta^\star(u) \in \argmin_{\vbeta\in\R^{d_{\vz}}} \E_q\!\left[
        \frac{w(S)}{q(S)}\{u(S)-\vz(S)^\top\vbeta\}^2
    \right].
\]
We estimate this target by \(\hat\vtau_{\mA,m}^{\mathrm{WLS},\lambda}=\mC_{\mA}^\top\hat\vbeta_{m,\lambda}\), where \(\hat\vbeta_{m,\lambda}\) is the ridge estimator in \eqref{eq:wls-ridge-est}. 

Each coordinate admits the regular expansion in~\eqref{eq:coordinate-regular-expansion} with the same surrogate \(h_j=h_w\) for all \(j\), where \(h_w\) is  defined in Proposition~\ref{prop:rr-regular}. 
The next result establishes a non-asymptotic bound for the remainder \(\E_q(\|\vr_m\|_2^2)\), which depends on the leverage scores of the weighted feature design. 
Define the Gram matrix and the leverage scores as
\begin{equation*}
    \mG =
    \E_q\!\left[
        \tilde w(S)\vz(S)\vz(S)^\top
    \right],
    \quad
    \ell(S) = \tilde w(S)\vz(S)^\top \mG^+\vz(S),
\end{equation*}
where \(\tilde w(S):=w(S)/q(S)\) with $\operatorname{supp}(w) \subseteq \operatorname{supp}(q)$ and \(\mG^+\) is the Moore--Penrose pseudoinverse. 
Large values of \(\ell(S)\) indicate that the WLS fit is more sensitive to the observation~at~\(S\).

\begin{theorem}
    [Non-asymptotic remainder bound for WLS estimators] \label{prop:rr-regular-exact}
    Let \(L = \sup_{S:q(S)>0} \ell(S)\) be the maximum leverage score. Assume that \(\mG\) has nonzero rank. Let \(\sigma_{\min}^+(\mG)\) and \(\sigma_{\max}(\mG)\) denote the minimum nonzero and maximum eigenvalues of \(\mG\), respectively, and define \(\kappa_{\mG}:=\sigma_{\max}(\mG)/\sigma_{\min}^+(\mG)\). 
    Define
    \(R_{\mA}^2:=\sigma_{\max}(\mC_{\mA}^\top\mG^+\mC_{\mA})\),
    \(\|u\|_{2,\tilde{w}}^2:=\E_q[\tilde w(S)u(S)^2]\), and
    \(\|u\|_{\infty,\tilde{w}}^2:=\sup_{S:q(S)>0}\tilde w(S)u(S)^2\).
    If the ridge penalty parameter satisfies \(\lambda \asymp \sigma_{\min}^+(\mG)\), then there exist universal constants \(C_1,C_2>0\) such that
    \begin{align}
        \E_q\left(\|\vr_m\|_2^2\right) &\le
        C_1 R_{\mA}^2 L^2 \left[ \frac{1}{m^2}\log d_{\vz} \left(1+\frac{L\log d_{\vz}}{m}\right) \|u\|_{2,\tilde w}^2 +
        e^{-\frac{C_2m}{L}} \cdot d_{\vz} m^2 \kappa_{\mG}^2 \|u\|_{\infty,\tilde w}^2 \right].
    \label{eq:rr-delta-bound}
    \end{align}
\end{theorem}

The bound in Theorem~\ref{prop:rr-regular-exact} also has two terms. 
The first term is polynomial in \(m\) and is scaled by the maximum leverage \(L\) and by \(R_{\mA}^2\). 
The factor \(R_{\mA}^2\) 
captures how strongly the target readout \(\mC_{\mA}^\top\vbeta\) can amplify errors in weakly identified WLS directions. 
The second term is exponentially small in \(m/L\) and accounts for the event that the empirical weighted design is poorly conditioned.

We next specialize this bound to two representative WLS-based Shapley estimators: KernelSHAP~\citep{Lundberg2017-gl} and LeverageSHAP~\citep{Musco2024-cf}. We focus on the centered Shapley vector, \(\vtau_{\mA}(u)=\mA\vphi(u)\), where \(\mA=\mI-n^{-1}\vone\vone^\top\) and \(\vphi(u)\) is the Shapley vector. 
This centered target $\vtau_{\mA}(u)$ can be represented as a linear readout \(\mC_{\mA}^\top\vbeta\) with certain coefficients $\mC_{\mA}$.

\begin{corollary}[Remainder bound for KernelSHAP and LeverageSHAP]\label{corollary:ls-kernel-leverage-rates}
Consider estimating the centered Shapley vector \(\mA\vphi(u)\) through the WLS characterization. There exist universal constants \(C_1,C_2>0\) such that
\[
    \E_q(\|\vr_m\|_2^2)
    \le
    C_1\|u\|_\infty^2
    \begin{cases}
        \displaystyle
        \frac{n^2\log^4 n}{m^2}
        \left(1+\frac{n\log^2 n}{m}\right)
        +
        n^3 m^2\log^3 n\,
        e^{-\frac{C_2 m}{n\log n}},
        & \text{KernelSHAP},\\[2ex]
        \displaystyle
        \frac{n^2\log^2 n}{m^2}
        \left(1+\frac{n\log n}{m}\right)
        +
        n^4 m^2
        e^{-C_2 \frac{m}{n}},
        & \text{LeverageSHAP}.
    \end{cases}
\]
\end{corollary}

The polynomial terms in Corollary~\ref{corollary:ls-kernel-leverage-rates} yield sufficient sample sizes for making the remainder negligible relative to the first-order MSE. In particular, if \(\|u\|_\infty\) is bounded and the first-order MSE has the natural full-vector scale \(V(\mA;q,\vh)=\sum_{j=1}^n V(\va_j;q,h_w)=O(n)\), then it is sufficient to take 
\(m=\Omega(n\log^4 n)\) for KernelSHAP and \(m=\Omega(n\log^2 n)\) for LeverageSHAP. 
For LeverageSHAP, this scale implies the non-asymptotic \(n\log n\) query guarantee up to one additional logarithmic factor (see Appendix~\ref{app:non-asymptotic-rr} for more details). 
\section{Efficiency-Aware Surrogate-Adjusted Estimation}
\label{sec:ease}

Existing estimators induce different combinations of sampling laws and working surrogates, typically use a prespecified sampling law and do not choose these components jointly by directly optimizing the first-order MSE. This leaves room for improving statistical efficiency. Motivated by this gap, we take
\(
V(\mA;q,\vh)
\)
itself as the design criterion and jointly optimize the sampling law \(q\) and surrogate \(\vh\). 

For a given pair of a sampling law \(q\) and a surrogate vector \(\vh=(h_1,\dots,h_d)\), consider the AIPW estimator based on
\(S_1,\dots,S_m\stackrel{\mathrm{iid}}{\sim}q\) and \(Y_t=u(S_t)\):
\begin{equation}
    \hat\tau_{\va_j}^{\mathrm{AIPW}}(u)
    =
    \tau_{\va_j}(h_j)
    +
    \frac{1}{m}
    \sum_{t=1}^m
    \frac{\rho_{\va_j}(S_t)}{q(S_t)}
    \{Y_t-h_j(S_t)\},
    \qquad j=1,\dots,d.
    \label{eq:aipw-vector}
\end{equation}
This estimator is unbiased, and its Euclidean MSE is exactly
$
    \E_q\!\left[
        \left\|
            \hat\vtau_{\mA}^{\mathrm{AIPW}}(u)-\vtau_{\mA}(u)
        \right\|_2^2
    \right]
    =
    V(\mA;q,\vh) / {m}.
$
Thus, for a given pair \((q,\vh)\), the AIPW estimator attains the same first-order MSE associated with a regular estimator indexed by that pair. 
For purposes of first-order efficiency, it suffices to restrict attention to the AIPW form and focus on the choice of~\((q,\vh)\).

Since both $q$ and $\vh$ affect the first-order MSE, we propose \emph{Efficiency-Aware Surrogate-Adjusted Estimation (EASE)}, which jointly chooses \((q,\vh)\) to target $V(\mA;q,\vh)$. 
To make this criterion amenable to data-driven optimization, we first rewrite it as an equivalent profiled squared-loss.
Introduce an auxiliary centering vector \(\vmu=(\evmu_1,\dots,\evmu_d)\) and define
\begin{equation}
    \mathscr V(\mA;q,\vh,\vmu)
    :=
    \sum_{j=1}^d
    \E_q\!\left[
        \left\{
            \frac{\rho_{\va_j}(S)}{q(S)}
            \{u(S)-h_j(S)\}
            -
            \evmu_j
        \right\}^2
    \right].
    \label{eq:profiled-vector-risk}
\end{equation}
Centering each coordinate at its mean gives \(V(\mA;q,\vh)=\min_{\vmu\in\R^d}\mathscr V(\mA;q,\vh,\vmu)\). 
Thus, the variance criterion can be viewed as a profiled squared-loss expectation, which suggests empirical risk minimization over both \(q\) and \(\vh\). 
Directly implementing this idea raises two challenges. First, learning \(q\) creates a circularity. 
Evaluating a candidate sampling law requires utility observations, while the sampling law determines which utilities are observed.  
Second, an unrestricted sampling law is a distribution over \(2^n\) coalitions, making direct optimization~infeasible.

We address these challenges in two ways. To resolve the circularity, we use a two-stage design: pilot samples drawn from a reference law are first used to learn the sampling law, which is then used to collect an independent estimation sample. To make sampling-law optimization tractable, we restrict \(q\) to a structured class \(\mathcal Q\). We use a partition-based class, under which coalitions of the same size share a sampling probability; see Appendix~\ref{app:candidate-class} for details. 
We next describe the two-stage procedure.

\textbf{Stage 1: Sampling-law selection based on pilot samples.}
Let \(q_0\) denote a prespecified reference law whose support contains that of every \(q\in\mathcal Q\). 
We first draw pilot coalitions \(S_1^{(0)},\dots,S_{m_0}^{(0)}\stackrel{\mathrm{iid}}{\sim}q_0\) and observe \(Y_t^{(0)}=u(S_t^{(0)})\). 
By importance weighting, we obtain an unbiased estimator of $ \mathscr V(\mA;q,\vh,\vmu)$ in~\eqref{eq:profiled-vector-risk} for every candidate \((q,\vh,\vmu)\):
\begin{equation}
    \widehat{\mathscr V}_{m_0}(\mA;q,\vh,\vmu)
    :=
    \frac{1}{m_0}
    \sum_{t=1}^{m_0}
    \sum_{j=1}^d
    \frac{
        \left[
            \rho_{\va_j}(S_t^{(0)})
            \{Y_t^{(0)}-h_j(S_t^{(0)})\}
            -
            q(S_t^{(0)})\evmu_j
        \right]^2
    }{
        q_0(S_t^{(0)})q(S_t^{(0)})
    }.
   \label{eq:empirical-profiled-training-pilot}
\end{equation}
We select the sampling law by minimizing the profiled empirical risk:
\begin{equation}
\hat q
\in
\argmin_{q\in\mathcal Q}
\min_{\vh\in\mathcal H,,\vmu\in\R^d}
\widehat{\mathscr V}_{m_0}(\mA;q,\vh,\vmu).
\label{eq:pilot-design-optimization}
\end{equation}

\begin{remarknum}[Pilot Optimization]
	Although the optimization problem in \eqref{eq:pilot-design-optimization} is generally nonconvex, it is blockwise convex when \(\mathcal Q\) and \(\mathcal H\) are convex. For fixed \(q\), the objective is convex in \((\vh,\vmu)\), whereas for fixed \((\vh,\vmu)\), it is convex in \(q\). 
	This structure suggests an alternating procedure. Starting from an initial sampling law, we update \((\vh,\vmu)\) while holding \(q\) fixed, then update \(q\) while holding \((\vh,\vmu)\) fixed, and repeat these two steps until the empirical criterion stabilizes. 
	The alternating procedure guarantees monotonic descent of the empirical objective.
\end{remarknum}

\textbf{Stage 2: Estimation with cross-fitting.}
After selecting the sampling law $\hat q$, we draw an independent estimation sample
\(\{(S_t,Y_t)\}_{t=1}^m\), where \(S_t\stackrel{\mathrm{iid}}{\sim}\hat q\) and $Y_t=u(S_t)$.
We learn the surrogate using this sample with cross-fitting, so that each held-out observation used to construct the estimator is independent of the surrogate fitted on the training fold.
Let \(\{I_b\}_{b=1}^B\) partition \(\{1,\dots,m\}\), and define \(\mathcal T_b=\{1,\dots,m\}\setminus I_b\). On each fold, EASE~solves
\begin{equation}
    (\hat\vh^{(-b)},\hat\vmu^{(-b)})
    \in
    \argmin_{\vh\in\mathcal H,\,\vmu\in\R^d}
    \sum_{t\in\mathcal T_b}
    \sum_{j=1}^d
    \left[
        \frac{\rho_{\va_j}(S_t)}{\hat q(S_t)}
        \{Y_t-h_j(S_t)\}
        -
        \evmu_j
    \right]^2.
    \label{eq:empirical-profiled-training}
\end{equation}
The fold-specific surrogate \(\hat\vh^{(-b)}\) is then plugged into the AIPW estimator on the corresponding held-out fold.
The complete procedure is summarized in Algorithm~\ref{alg:vector-aipw-profiled}, with further implementation details provided in Appendix~\ref{app:algorithm}.

\begin{algorithm}[!htbp]
\caption{Efficiency-Aware Surrogate-Adjusted Estimation (EASE)}
\label{alg:vector-aipw-profiled}
\KwIn{Target matrix \(\mA\); candidate classes \(\mathcal Q,\mathcal H\); reference law \(q_0\); budgets \(m_0,m\); folds \(\{I_b\}_{b=1}^B\); }
\KwOut{Estimator \(\hat\vtau_{\mA}\)}
\Comment{Stage 1: Sampling-law selection}
Draw \(m_0\) pilot observations from \(q_0\) and obtain $\hat q$ by solving \eqref{eq:pilot-design-optimization};\;
\Comment{Stage 2: Cross-fitting estimation}
Draw \(m\) independent estimation observations from \(\hat q\)\;
\For{\(b=1,\dots,B\)}{
Fit \((\hat\vh^{(-b)},\hat\vmu^{(-b)})\) on \(\mathcal T_b\) using \eqref{eq:empirical-profiled-training}\;
\For{\(j=1,\dots,d\)}{
Set
\(
\hat\tau_{\va_j}^{(b)}
=
\tau_{\va_j}(\hat h_j^{(-b)})
+
\frac{1}{|I_b|}
\sum_{t\in I_b}
\frac{\rho_{\va_j}(S_t)}{\hat q(S_t)}
\{Y_t-\hat h_j^{(-b)}(S_t)\}
\)\;
}
}
\Return
\(
\hat\vtau_{\mA}
=
\sum_{b=1}^B
\frac{|I_b|}{m}
\bigl(
\hat\tau_{\va_1}^{(b)},\dots,\hat\tau_{\va_d}^{(b)}
\bigr)^\top
\)\;
\end{algorithm}

\begin{remarknum}[Comparison with Existing Surrogate-adjusted Estimators]
	Cross-fitted AIPW estimation has also been used by existing methods, including RegressionMSR~\citep{Witter2025-gz}.
	The key distinction of EASE is not the AIPW correction itself, but how the sampling law \(q\) and working surrogate \(\vh\) are chosen.
	RegressionMSR uses a prespecified sampling law that is independent of the utility function and surrogate fit, and learns its surrogate by minimizing an unweighted approximation loss. 
	In contrast, our first-order MSE analysis suggests that both choices should be efficiency-aware.
	For a fixed sampling law, EASE fits the surrogate by minimizing the profiled weighted least-squares criterion in \eqref{eq:empirical-profiled-training}, rather than to approximate (u) uniformly over the sampled coalitions. For a fixed surrogate \(\vh\), the oracle best-in-class sampling law satisfies
\begin{equation}
    q_{\mathcal Q}^\star(\vh)
    \in
    \argmin_{q\in\mathcal Q}
    \sum_{S\subseteq[n]}
    \frac{
        \sum_{j=1}^d
        \rho_{\va_j}(S)^2
        \{u(S)-h_j(S)\}^2
    }{
        q(S)
    }.
    \label{eq:ease-oracle-q-general}
\end{equation}
Thus, subject to the restrictions defining \(\mathcal Q\), EASE allocates more probability to coalitions with larger target-weighted residuals; see details in Appendix~\ref{app:sampling-details}. EASE therefore uses the same first-order efficiency criterion to guide surrogate fitting and residual-aware sampling.
\end{remarknum}

\begin{remarknum}
	This residual-aware sampling law in~\eqref{eq:ease-oracle-q-general} also connects to classical optimal sampling designs. 
If \(\mathcal Q\) were unrestricted, the minimizer would satisfy
\(q^\star(S)\propto\bigl[\sum_{j=1}^d\rho_{\va_j}(S)^2\{u(S)-h_j(S)\}^2\bigr]^{1/2}\),
which has the familiar form of a variance-minimizing importance-sampling law~\citep{Kahn1953-mc}. 
More directly relevant to EASE, within the partition-based class $\mathcal Q$ considered here, the oracle becomes a Neyman-type allocation~\citep{Neyman1934-wp}. Classical Neyman allocation assigns more samples to strata with larger variability, whereas here the relevant variability is the residual variation left after surrogate approximation. Because the working surrogate class need not contain the true utility function, these residuals are generally non-negligible and can play an important role in sampling design. 
Therefore, EASE brings together surrogate adjustment and the classical principle of variance-reducing sampling.
\end{remarknum}

\providecommand{\revised}[1]{{\color{darkpurple}#1}} % marks text revised in this round

\section{Numerical Experiments}
\label{sec:numerical}

We evaluate EASE against existing  estimators for probabilistic values in two settings.
The first is a controlled benchmark based on sum-of-unanimity (SOU) games, where the target values are available in closed form and how well low-order surrogates approximate the utility can be varied directly. This setting allows us to examine whether EASE benefits from efficiency-aware design under varying surrogate quality. 
The second setting consists of real-world data- and feature-attribution benchmarks, where the utilities arise from model fitting and prediction on real-world datasets. 

In the experiments, all methods are evaluated under the same total utility-evaluation budgets for fair comparison. 
EASE allocates 20\% of the total utility-evaluation budget to the pilot stage, and the remaining budget is used for final estimation. 
We observe similar performance when varying the pilot fraction from 10\% to 40\%; see the sensitivity analyses in Appendix~\ref{app:pilot-fraction-sensitivity}.
To estimate the sampling law \(\hat q\) in the pilot stage, we perform three iterations, each consisting of  surrogate fitting followed by an update of the sampling law according to criterion~\eqref{eq:empirical-profiled-training-pilot}. The final AIPW estimator uses two-fold cross-fitting. Our implementation also reuses  pilot observations in the final estimator; further details are provided in Appendix~\ref{app:algorithm}. 
Code for reproducing the experiments is available in the \href{https://github.com/zqliu01/ease_probabilistic_values}{EASE GitHub repository}.

\subsection{Controlled SOU Benchmarks}
\label{sec:sou-benchmarks}

SOU games are commonly used as controlled benchmarks for probabilistic-value estimation~\citep{Li2024-gp,Lee2025-ht}.
We consider SOU utilities of the form
\(
u(S)
=
\sum_{T\in\mathscr S_{\mathrm{low}}}
\theta_T\mathbb{I}(T\subseteq S)
+
\sum_{T\in\mathscr S_{\mathrm{high}}}
\xi_T\mathbb{I}(T\subseteq S),
\)
where 
\(
\mathscr S_{\mathrm{low}}
=
\{T\subseteq[n]:1\leq |T|\leq2\}
\)
contains all singleton and pairwise interactions among players, and \(\mathscr S_{\mathrm{high}}\) is a multiset of \(n^2\) randomly generated coalitions with cardinalities between 3 and 40. Probabilistic values for such SOU utilities are available in closed form, which we use as ground truth; the formulas and details on generating \(\mathscr S_{\mathrm{high}}\) are provided in Appendix~\ref{app:additional-experiments}.
To vary how well low-order surrogates approximate the utility, 
we draw the coefficients independently as \(\theta_T\sim N(0,\eta\sigma^2/|\mathscr S_{\mathrm{low}}|)\) and \(\xi_T\sim N(0,(1-\eta)\sigma^2/|\mathscr S_{\mathrm{high}}|)\). Thus, \(\eta\) is the fraction  of the total coefficient variance assigned to singleton and pairwise interactions. 
A larger \(\eta\) makes the utility more amenable to approximation by low-order surrogates, whereas a smaller \(\eta\) leaves more variation in higher-order interactions.

The main experiments use \(n=40\) and \(\eta\in\{0.25,0.5,0.75\}\). Results for \(n\in\{80,160\}\), which lead to the same qualitative conclusions, are reported in Appendix~\ref{app:additional-aucc-sou}. We first compare EASE with representative baselines under matched working classes and then benchmark it against a broader collection of estimators across several probabilistic values.

\textbf{Matched comparisons.}
The matched comparisons are designed to isolate the effect of EASE's efficiency-aware design. We focus on Shapley-value estimation and 
compare EASE with one representative from each of three estimator families: the surrogate-adjusted RegressionMSR estimator~\citep{Witter2025-gz}, the self-normalized weighted-average OFA estimator~\citep{Li2024-gp}, and the WLS-based PolySHAP estimator~\citep{Fumagalli2026-hu}. 
In each comparison, EASE uses the same working surrogate class as the corresponding baseline. It uses the first-order class spanned by  the player-inclusion indicators \(\{\mathbb{I}(i\in S): i\in[n]\}\) for RegressionMSR, the size-player class spanned by the indicators \(\{\mathbb{I}(|S|=s,i\in S): i\in[n],\, s=0,\ldots,n\}\) for OFA, and the second-order class spanned by \(\{\mathbb{I}(T\subseteq S):|T|\leq2\}\) for the second-order PolySHAP.
Thus, differences in their performance reflect the efficiency-aware sampling and adjustment strategy rather than a richer working class.

\begin{figure}[!hbt]
    \centering
    \includegraphics[width=0.95\linewidth]{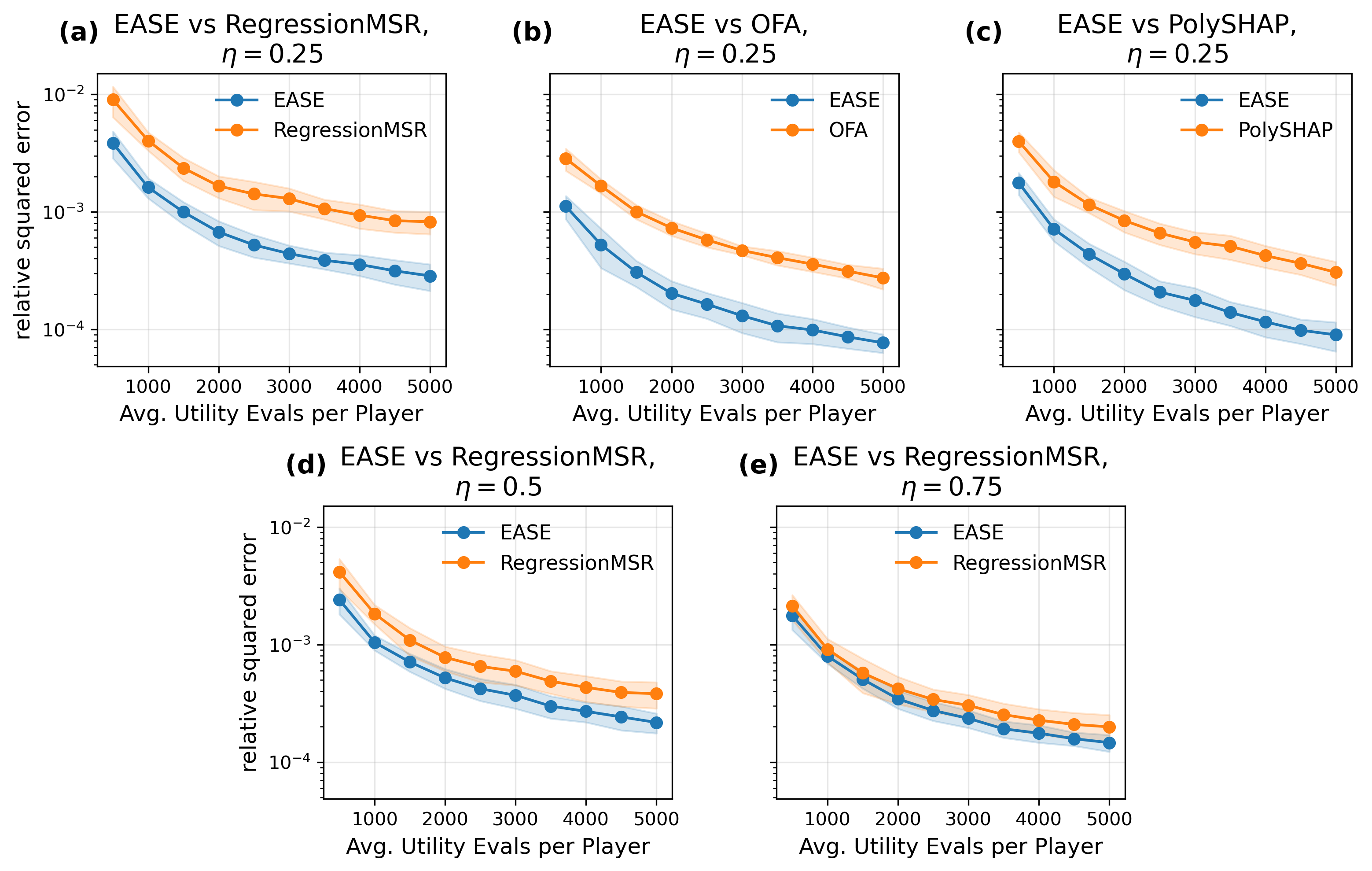}
    \caption{Matched comparisons for Shapley-value estimation on SOU games with \(n=40\). Panels (a)--(c) compare EASE with RegressionMSR, OFA, and PolySHAP at \(\eta=0.25\); panels (d)--(e) compare EASE with RegressionMSR at \(\eta=0.5\) and \(0.75\). Within each panel, EASE uses the same working class and matched sampling conventions as the baseline. The horizontal axis is the average number of utility evaluations per player, and the vertical axis is \(\|\hat{\vphi}-\vphi\|_2^2/\|\vphi\|_2^2\) on a logarithmic scale. Curves are pointwise arithmetic means over 10 estimator runs on one fixed game for each \(\eta\). }
    \label{fig:sou-matched-comparisons}
\end{figure}

Figure~\ref{fig:sou-matched-comparisons} reports the matched comparisons. Panels (a)--(c) compare EASE with RegressionMSR, OFA, and PolySHAP at \(\eta=0.25\). EASE achieves lower relative squared error throughout the  budget range in all three comparisons, supporting that its gain arises from the efficiency-aware design. 
Panels (a), (d), and (e) further compare EASE with RegressionMSR, the closest surrogate-adjusted comparator, as the low-order signal strength varies. 
The gap between EASE and RegressionMSR narrows as \(\eta\) increases. This pattern is consistent with residual-aware sampling being more beneficial when the shared working class leaves substantial utility structure unexplained. 

\textbf{Performance across probabilistic values.}
We next evaluate EASE across multiple probabilistic values and a broader set of competing estimators. We consider six targets: the Shapley value, Beta Shapley~\citep{Kwon2021-bw} with parameters \((4,1)\) and \((1,4)\), and weighted Banzhaf values~\citep{Wang2022-iz} with \(p\in\{0.25,0.5,0.75\}\).

We include two EASE variants that differ in their working surrogate classes. 
EASE-FO uses a compact first-order class spanned by an intercept, the player-inclusion indicators \(\{\mathbb{I}(i\in S):i\in[n]\}\), and two coalition-size terms \(\log(1+|S|)\) and \((|S|/n)^2\). 
EASE-SP uses a the richer size-player class spanned by \(\{\mathbb{I}(|S|=s,i\in S):i\in[n],\,s=1,\ldots,n\}\), which allows player effects to vary with coalition size. 
The working class used in EASE-SP is more flexible but also more computationally expensive. We compare these variants with OFA~\citep{Li2024-gp}, sampling lift and weighted sampling lift~\citep{Moehle2021-cm,Kwon2021-bw}, SHAP-IQ~\citep{Fumagalli2023-lo}, GELS and WGELS~\citep{Li2024-yr}, KernelSHAP~\citep{Lundberg2017-gl,Covert2021-di}, permutation and weighted-permutation estimators~\citep{Castro2009-yu}, complement sampling~\citep{Zhang2023-ne}, group testing~\citep{Jia2019-wa}, AME~\citep{Lin2022-ul}, RegressionMSR~\citep{Witter2025-gz}, and PolySHAP~\citep{Fumagalli2026-hu}. Estimators defined only for particular probabilistic values are included where applicable.

Following existing probabilistic-value benchmarks~\citep{Li2024-gp,Lee2025-ht}, we summarize each convergence trajectory by the area under the convergence curve (AUCC), defined as \(\operatorname{AUCC}=|\mathcal B|^{-1}\sum_{m\in\mathcal B}\|\widehat{\vphi}_m-\vphi\|_2/\|\vphi\|_2\), where \(\mathcal B=\{80,160,\ldots,4000\}\) and \(m\) is the average number of utility evaluations per player. Lower AUCC indicates greater sample efficiency because it corresponds to lower estimation error across the tracked evaluation budgets.

\begin{figure}[!hbt]
    \centering
    \includegraphics[width=0.9\linewidth]{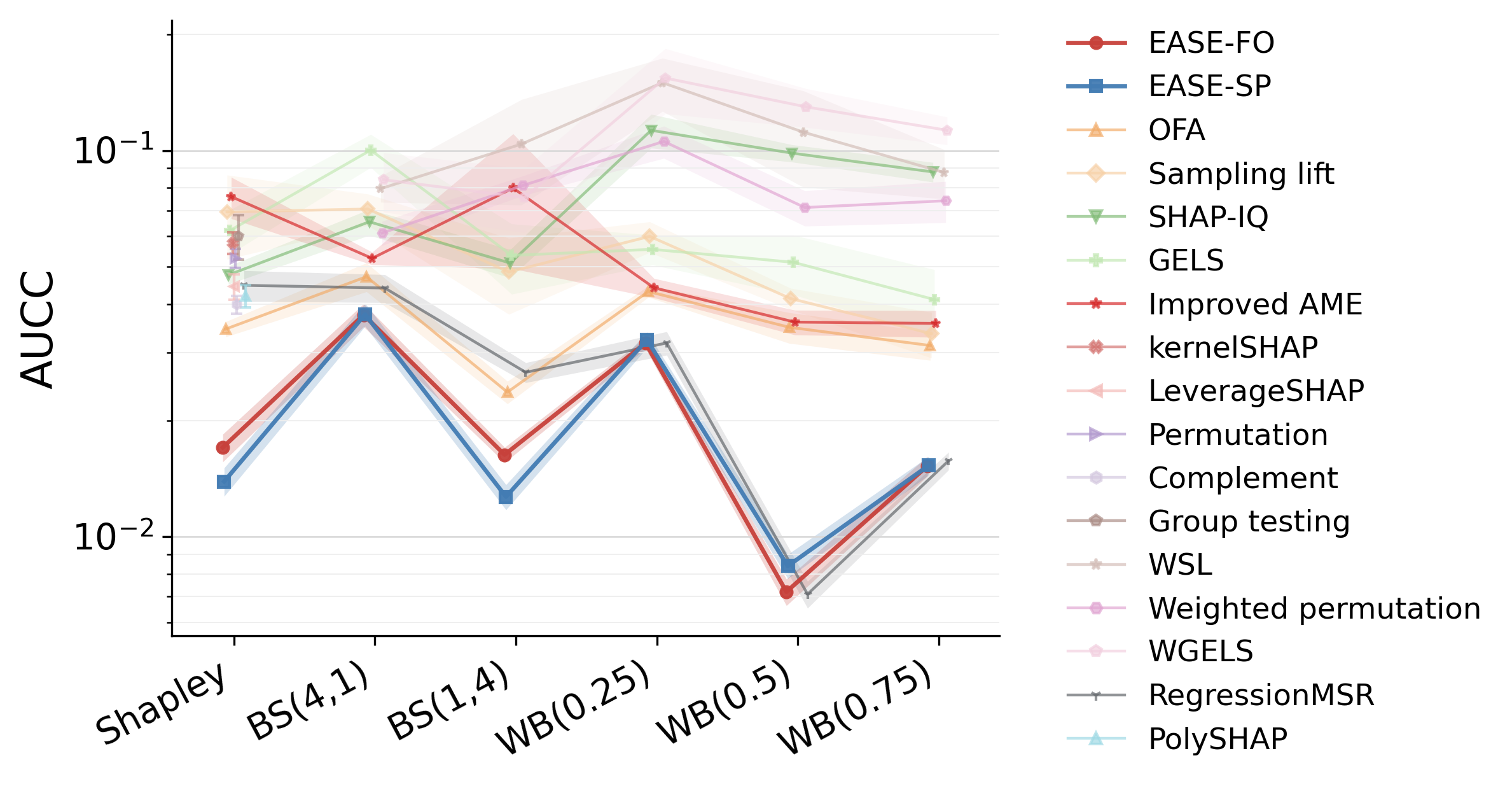}
    \caption{AUCC on the SOU game with \(\eta=0.25\) for six probabilistic-value
    targets: Shapley, Beta Shapley (BS), and weighted Banzhaf (WB). Markers
    report the arithmetic mean of the run-level AUCC over 10 estimator runs;
    lower values are better.}
    \label{fig:sou-aucc-benchmark}
\end{figure}

Figure~\ref{fig:sou-aucc-benchmark} summarizes the main SOU benchmark at \(\eta=0.25\) and \(n=40\). EASE-FO and EASE-SP attain the two lowest mean AUCC values for the Shapley and both Beta-Shapley targets, with EASE-SP generally benefiting from its richer working class. For the weighted Banzhaf targets, the EASE variants remain competitive with the strongest target-specific estimators, although they do not uniformly attain the lowest AUCC. Appendix~\ref{app:additional-aucc-sou} reports additional results for \(\eta\in\{0.5,0.75\}\) and for larger games with \(n\in\{80,160\}\). The rankings are stable across these settings, except that EASE-FO outperforms EASE-SP as \(n\) grows and the per-player budget shrinks, making the richer surrogate harder to fit.

\subsection{Real-World Benchmarks}
\label{sec:real-world}

We next evaluate EASE on six real-world Shapley-value estimation benchmarks, summarized in Table~\ref{tab:real-world-tasks}. 
The benchmarks cover two common attribution problems in machine learning: data attribution, where players are training-data sources, and local feature attribution, where players are input features or image regions. 
In both cases, Shapley values provide principled measures of contribution, but exact computation is expensive or infeasible. 
We therefore compare EASE with existing estimators in terms of both estimation accuracy and computational cost.

\begin{table}[!tbp]
    \centering
\small
\renewcommand{\arraystretch}{1.12}
\begin{tabularx}{\linewidth}{@{}l c Y Y@{}}
\toprule
Benchmark & \(n\) & Application & Model \\
\midrule
ACSIncome--logistic
& 50 & State-source data valuation & Logistic regression \\
ACSIncome--XGBoost
& 50 & State-source data valuation & XGBoost \\
CIFAR-10
& 16 & Image-patch attribution & Vision Transformer (ViT) \\
Breast Cancer
& 30 & Feature attribution & Random forest \\
Communities and Crime
& 101 & Feature attribution & Random forest \\
NHANES I
& 79 & Feature attribution & Random forest \\
        \bottomrule
    \end{tabularx}
    \caption{Real-world Shapley-value benchmarks. The two ACSIncome benchmarks use
    the same state-source valuation game with different downstream learners. The four
    feature-attribution tasks each comprise $30$ local explanation games, one
    for each explained instance.}
    \label{tab:real-world-tasks}
\end{table}

\textbf{Data attribution.} 
We use the Folktables ACSIncome task~\citep{Ding2021-gl}, based on the 2018 American Community Survey. The task is to predict whether an individual's personal income exceeds \$50,000. Taking Pennsylvania residents as the target population, we treat the $50$ U.S. states as players and measure how much each state's training data contributes to predictive performance on Pennsylvania. 
For each coalition \(S\), we pool $500$ training observations from each state in \(S\), fit either ridge logistic regression or XGBoost, and define \(u(S)\) as the negative logarithmic-transformed loss on a fixed held-out set of $1,000$ Pennsylvania observations. 
Because exact Shapley value computation is infeasible for \(n=50\), we use an independent high-budget OFA estimate, computed using \(4000n\) utility evaluations, as the reference value.

\textbf{Feature attribution.}
For local feature attribution, we follow common benchmark settings used in~\citep{Muschalik2024-shapiq,Musco2024-cf,Witter2025-gz,Fumagalli2026-hu,Fumagalli2026-rf}. We consider four datasets: CIFAR-10~\citep{Krizhevsky2009-cifar}, Breast Cancer~\citep{Street1993-breast}, Communities and Crime~\citep{Redmond2009-communities}, and NHANES~I~\citep{Dinh2019-nhanes}. For a given test instance, the utility \(u(S)\) is the model output when the components in \(S\) are retained and all remaining components are replaced by task-specific reference values. For CIFAR-10, each image is partitioned into a \(4\times4\) grid of 16 patch players; excluded patches are replaced by gray pixels, and the utility is the vision transformer's logit for the class predicted from the original image. For the three tabular tasks, the players are input variables; excluded variables are replaced by their featurewise means, and the utility is the prediction of a random forest. Reference Shapley values are computed by exhaustive enumeration for CIFAR-10 and by the TreeProb algorithm~\citep{Witter2025-gz} for the tabular random-forest models.

\begin{figure}[!htb]
    \centering
    \includegraphics[width=\linewidth]{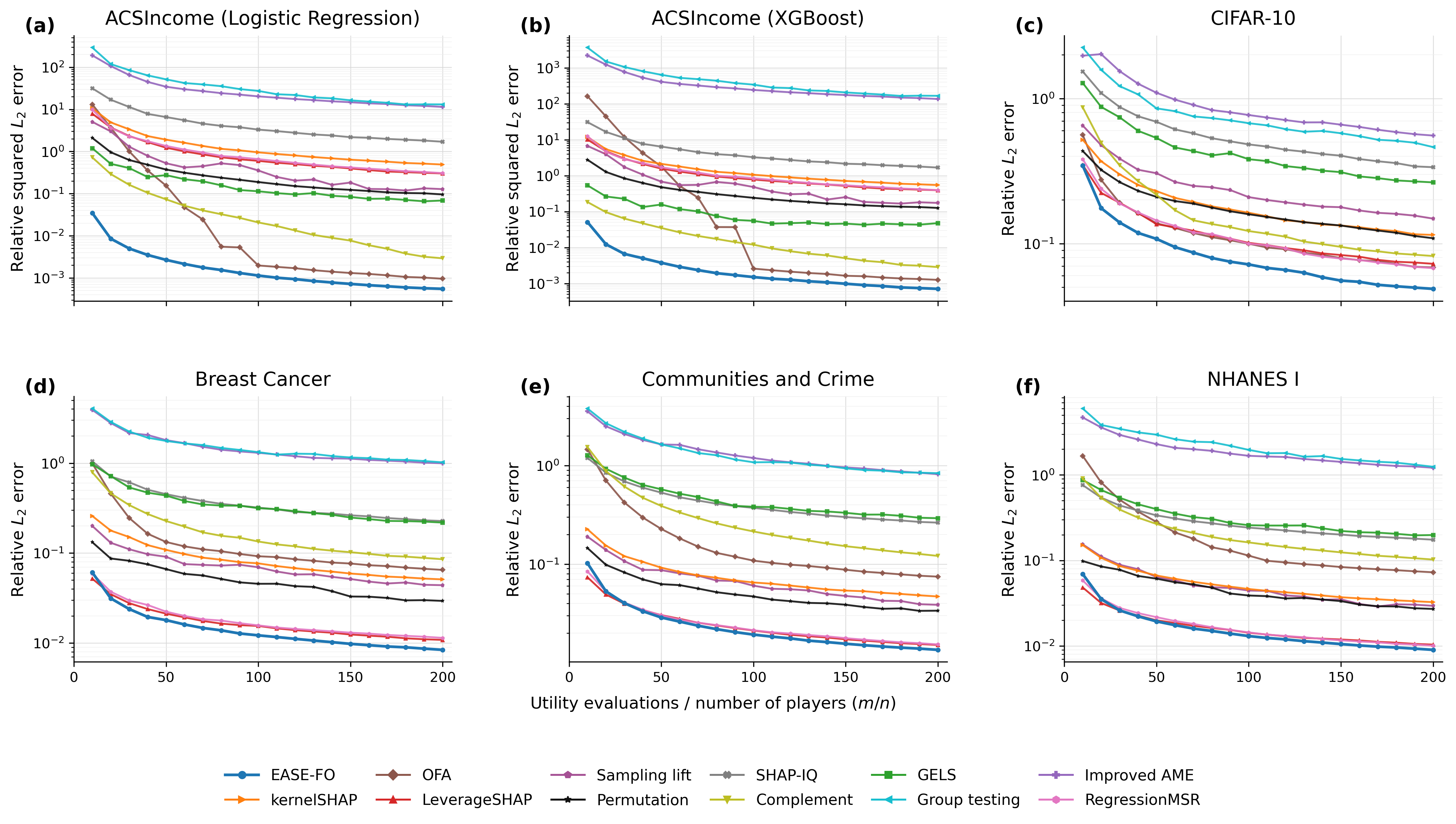}
    \caption{Shapley-value estimation on six real-world games. Panels (a)--(b)
    are data-attribution games and panels (c)--(f) are feature-attribution
    games. The x-axis reports utility evaluations per player,
    \(m/n\). The y-axis reports relative squared estimation error on
    logarithmic scales. Curves are pointwise averages over 20
    estimator seeds in panels (a)--(b) and 30 explained test instances in panels
    (c)--(f); lower values are better.}
    \label{fig:real-world-shapley}
\end{figure}

\textbf{Estimation accuracy.}
We evaluate each method up to a total budget of \(200n\) utility evaluations. Figure~\ref{fig:real-world-shapley} reports the estimation error at $20$ equally spaced checkpoints. 
Across six benchmarks, EASE-FO attains the lowest error throughout most of the budget range. 
The gains are especially clear on the two ACSIncome benchmarks and CIFAR-10.
On the other three tabular feature-attribution benchmarks, EASE-FO remains the best method at the final budget, althrough the gains are smaller, with RegressionMSR as the closest competitor.
The smaller gaps between EASE-FO and RegressionMSR are consistent with the controlled SOU results at larger \(\eta\). In these settings, a first-order surrogate appears to capture most of the utility variation, leaving less residual structure for EASE to exploit through its residual-aware sampling design. The sampling-allocation diagnostics in Appendix~\ref{app:allocation-diagnostics} provide further support for this~interpretation.

\begin{figure}[!htb]
    \centering
    \includegraphics[width=\linewidth]{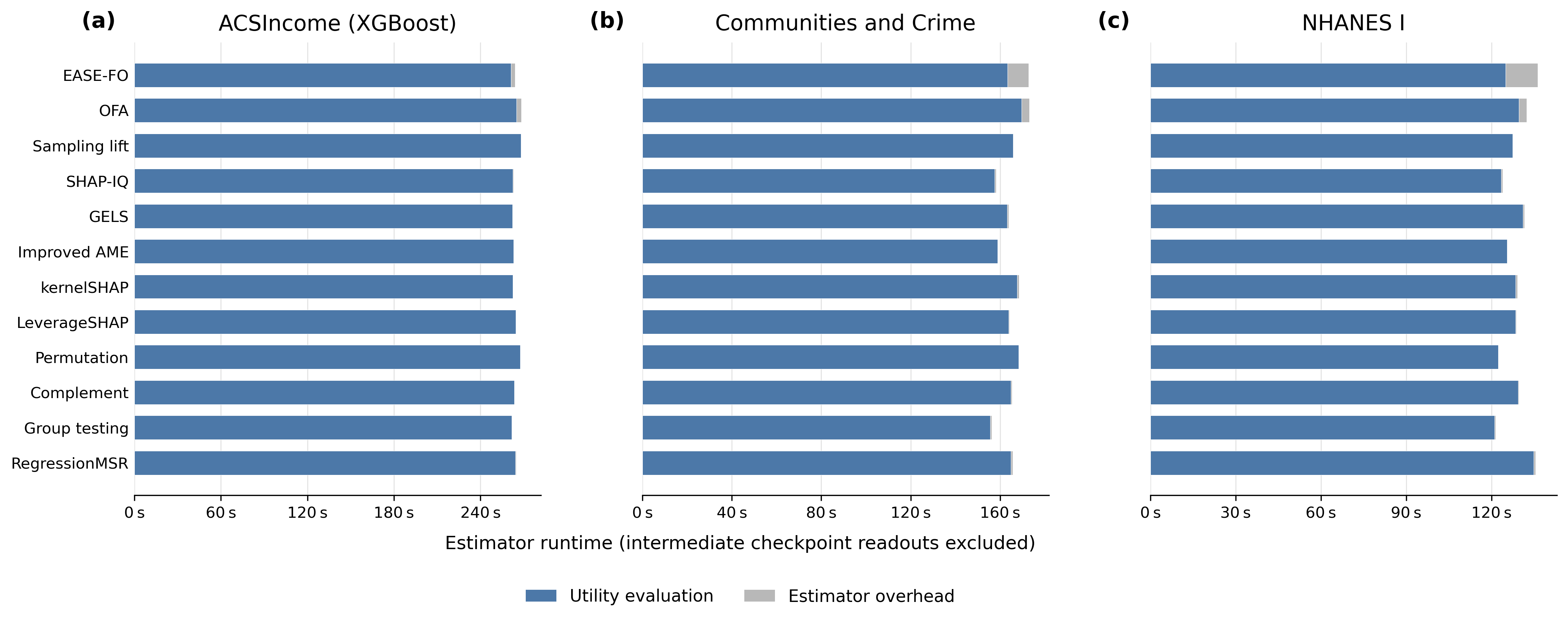}
    \caption{Runtime decomposition for ACSIncome with XGBoost, Communities and
    Crime, and NHANES~I. Stacked bars report componentwise averages 
    over estimator seeds for ACSIncome and explained instances for the
    feature-attribution tasks. Intermediate checkpoint readouts are excluded.}
    \label{fig:real-world-runtime}
\end{figure}

\textbf{Computational cost.}
Figure~\ref{fig:real-world-runtime} shows the runtime decomposition for the ACSIncome benchmark with XGBoost and the two large-\(n\) feature-attribution benchmarks, Communities and Crime and NHANES~I. We separate the time spent on utility evaluation from estimator overhead, which includes coalition sampling, surrogate fitting, sampling-law updates, and final estimation. Across all three benchmarks, utility evaluation dominates the total runtime. EASE-FO incurs modest additional overhead, particularly on the two large-\(n\) tasks, but this overhead remains small relative to the cost of utility evaluation, and its total runtime is comparable to those of the competing methods. Thus, on these benchmarks, the accuracy gains of EASE-FO do not come at a substantial additional computational cost.

\section{Discussion: Beyond Single-Coalition Sampling} \label{sec:discussion}

\textbf{Bundled sampling.} 
The target $\tau_{\va}(u)$ can also be estimated under \emph{bundled sampling}, where each Monte Carlo draw returns multiple coalitions \(\vT=(S_1,\ldots,S_r)\sim Q\).
Examples include edge pairs \((S,S\cup\{i\})\)~\citep{Moehle2021-cm,Kwon2021-bw}, complement pairs \((S, S^c)\)~\citep{Zhang2023-ne,Musco2024-cf,Fumagalli2026-hu}, and entire permutation paths~\citep{Castro2009-yu,Castro2017-pf,Maleki2013-zq,Wu2023-hb}. 
In Appendix~\ref{app:bundled}, we extend the regular-estimator framework and first-order efficiency analysis to the bundled sampling.
Unlike single-coalition sampling, where the weighting under a single-coalition law \(q\) is uniquely determined as  \( \gamma_{\va,q}(S)=\rho_{\va}(S)/q(S)\), bundled sampling introduces another design choice. 
A feasible weighting vector \(\vgamma(\vT)=(\gamma_1(\vT),\ldots,\gamma_r(\vT))\) only needs to
satisfy \(
    \rho_{\va}(S)
    =
    \E_Q\!\left[
   \sum_{\ell=1}^r \gamma_\ell(\vT)\1\{S_\ell=S\}
    \right],
    \) for any \(S\subseteq[n]\). 
Thus, this system is generally underdetermined, and the first-order MSE depends jointly on the surrogate \(h\), the bundle law \(Q\), and the feasible $\vgamma$.
In Appendix~\ref{app:bundled}, we study complement-pair sampling \(\vT=(S,S^c)\) with \(S\sim q\) as a tractable example. 
For semivalues whose weights are symmetric with respect to coalition size, 
we show that the weighting \(\vgamma(\vT) = \left(\frac{\rho_{\va}(S)}{q(S)+q(S^c)}, -\frac{\rho_{\va}(S)}{q(S)+q(S^c)} \right)\),
which has been implicitly used in the literature~\citep{Zhang2023-ne, Witter2025-gz, Fumagalli2026-hu}, 
minimizes the first-order MSE for a fixed complement-pair law and surrogate.
A promising direction for future work is to generalize this understanding to broader bundled sampling designs.

\textbf{Adaptive sampling.} 
Our framework allows the sampling law to be learned from pilot observations but holds it fixed during the estimation stage. Independent recent work has considered sequentially adaptive coalition selection. In particular, ShaplEIG updates a Gaussian-process surrogate using the utility evaluations observed so far and selects each subsequent coalition by its expected information gain about the Shapley values~\citep{Rundel2026-se}. Extending our first-order efficiency framework to sequentially adaptive designs is another promising direction for future~work.

\section*{Disclosure Statement}
The authors report no potential conflicts of interest. 

\section*{Declaration of Generative AI Use}
The authors used Generative AI  tools (ChatGPT 5.5 and ChatGPT 5.6) for coding assistance and language refinement. The authors reviewed and edited all AI-assisted outputs and take full responsibility for the accuracy, integrity, and originality of the content of this manuscript. 

\section*{Code Availability Statement}
Code for reproducing the experiments is available in the \href{https://github.com/zqliu01/ease_probabilistic_values}{EASE GitHub repository}.

\section*{Acknowledgements}
Liu and Tang were supported by NSF DMS-2412853 and Jane Street Group, LLC.
Lee and Zhang were supported by NSF DMS-2311109. 
We thank Raylen Teal Witter and Maximilian Muschalik for helpful feedback on an earlier version of this work, including suggestions on benchmarking pipelines and evaluation metrics, and thank Weida Li for comments on the theoretical results.

\spacingset{1}

\bibliographystyle{bibstyle}
\bibliography{reference-clean}

\clearpage
\hypersetup{pageanchor=false}
\setcounter{page}{1}
\phantomsection\label{supplementary-material}
\bigskip

\begin{center}

{\large\bf SUPPLEMENTARY MATERIAL}

\end{center}
\spacingset{1.5}
\appendix
\startcontents[supplement]
\printsupplementarycontents

\section{Axioms of the Shapley Value}\label{app:axioms}

This appendix states the four axioms referenced in the introduction. An \emph{attribution rule} \(\vphi\) assigns to each utility function \(u:2^{[n]}\to\mathbb R\) a vector \(\vphi(u)=(\evphi_1(u),\dots,\evphi_n(u))^\top\in\mathbb R^n\), where \(\evphi_i(u)\) represents the contribution attributed to player \(i\). The following axioms are required to hold for all utility functions \(u,v:2^{[n]}\to\mathbb R\), all scalars \(a,b\in\mathbb R\), and all players \(i,j\in[n]\).

\begin{axiom}[Efficiency]\label{ax:efficiency}
\(\sum_{i=1}^{n}\evphi_i(u)=u([n])-u(\emptyset)\).
\end{axiom}

\begin{axiom}[Symmetry]\label{ax:symmetry}
If \(u(S\cup\{i\})=u(S\cup\{j\})\) for all \(S\subseteq[n]\setminus\{i,j\}\), then \(\evphi_i(u)=\evphi_j(u)\).
\end{axiom}

\begin{axiom}[Linearity]\label{ax:linearity}
\(\evphi_i(au+bv)=a\,\evphi_i(u)+b\,\evphi_i(v)\).
\end{axiom}

\begin{axiom}[Null player]\label{ax:null-player}
If \(u(S\cup\{i\})=u(S)\) for all \(S\subseteq[n]\setminus\{i\}\), then \(\evphi_i(u)=0\).
\end{axiom}

These axioms formalize natural requirements on an attribution rule: the realized total utility \(u([n])-u(\emptyset)\) is fully distributed among the players (efficiency); interchangeable players receive equal attributions (symmetry); the attribution acts linearly on the utility function (linearity); and a player whose marginal contribution to every coalition is zero receives zero attribution (null player).

The classical result of \citet{Shapley1953-ru} shows that the Shapley value, given by the choice \(\alpha_i^{(n)}(S)=\{n\binom{n-1}{|S|}\}^{-1}\) in~\eqref{eq:semivalue-def}, is the unique attribution rule satisfying Axioms~\ref{ax:efficiency}--\ref{ax:null-player}. Every probabilistic value of the form~\eqref{eq:semivalue-def} satisfies linearity and the null-player property, but in general neither efficiency nor symmetry. Imposing symmetry in addition forces the weights \(\alpha_i^{(n)}(S)\) to depend on \(S\) only through the coalition size \(|S|\), which yields the class of semivalues; see \citet{Dubey1981-if} and \citet{Weber1988-zf} for axiomatic characterizations of semivalues and probabilistic values.

\section{Proofs}

\subsection{Proof of Theorem~\ref{prop:non-asymptotic-hajek}}

\begin{proof}
We prove the result for the scalar H{\'a}jek-type estimator
\(\hat\tau_{\va,m}^{\mathrm{Hajek}}(u)\) in \eqref{eq:hajek-est}.
The proof has four steps. We first recall the cellwise surrogate from
Proposition~\ref{prop:aug-hajek-regular}, then derive the exact remainder
decomposition, and finally bound the normalization and missing-cell pieces
separately.

\paragraph{Step 1: Recall the cellwise centering.}
For each cell \(C_k\) with \(\pi_k>0\), define
\[
    \mu_k
    :=
    \E_q\!\left[
        \gamma_{\va,q}(S)u(S)\mid S\in C_k
    \right]
    =
    \frac{1}{\pi_k}
    \E_q\!\left[
        \gamma_{\va,q}(S)u(S)\1\{S\in C_k\}
    \right].
\]
For cells with \(\pi_k=0\), set \(\mu_k:=0\). By the condition in
Proposition~\ref{prop:aug-hajek-regular}, on each cell \(C_k\) the weight
\(\gamma_{\va,q}\) is either identically zero or everywhere nonzero. Hence the
cellwise surrogate \(h_{\mathscr C}\) satisfies
\[
    \gamma_{\va,q}(S)h_{\mathscr C}(S)
    =
    \sum_{k=1}^K \mu_k\1\{S\in C_k\},
\]
with the convention \(h_{\mathscr C}(S)=0\) on cells where
\(\gamma_{\va,q}\equiv0\).

Taking expectations gives
\[
    \E_q\!\left[
        \gamma_{\va,q}(S)h_{\mathscr C}(S)
    \right]
    =
    \sum_{k=1}^K \pi_k\mu_k
    =
    \E_q\!\left[
        \gamma_{\va,q}(S)u(S)
    \right]
    =
    \tau_{\va}(u),
\]
where the last identity is the weighted-average representation
\eqref{eq:tau-a-ipw}. Therefore \(\tau_{\va}(h_{\mathscr C})=\tau_{\va}(u)\),
and the influence term in Definition~\ref{def:regular} simplifies to
\[
    \psi_{h_{\mathscr C}}(S;u)
    =
    \gamma_{\va,q}(S)u(S)-\mu_{k(S)}.
\]

\paragraph{Step 2: Decompose the estimator error.}
Let
\[
    \xi_t
    :=
    \psi_{h_{\mathscr C}}(S_t;u)
    =
    \gamma_{\va,q}(S_t)u(S_t)-\mu_{k(S_t)},
\]
where \(k(S_t)\) denotes the cell index of \(S_t\). By construction,
\[
    \E(\xi_t\mid S_t\in C_k)=0
    \qquad
    \text{for every } k.
\]
Set
\[
    A_{mk}
    :=
    \frac{1}{m}
    \sum_{t=1}^m
    \1\{S_t\in C_k\}\xi_t,
\]
and, only within this proof,
\[
    \hat\lambda_k
    :=
    \begin{cases}
        \pi_k/\hat\pi_k, & N_k>0,\\
        0, & N_k=0.
    \end{cases}
\]
Define
\[
    r_{m,\mathrm{norm}}
    :=
    \sum_{k=1}^K
    (\hat\lambda_k-1)A_{mk},
    \qquad
    r_{m,\mathrm{miss}}
    :=
    \sum_{k:N_k=0} \pi_k\mu_k.
\]
Also, by the definition of \(\hat\lambda_k\),
\[
    \hat\lambda_k\hat\pi_k
    =
    \pi_k\1\{N_k>0\}.
\]
Using the definition of \(\hat\tau_{\va,m}^{\mathrm{Hajek}}(u)\), the identity
\[
    \hat\pi_k
    =
    \frac{1}{m}\sum_{t=1}^m \1\{S_t\in C_k\},
\]
and \(\tau_{\va}(u)=\sum_{k=1}^K \pi_k\mu_k\), we obtain
\begin{align*}
    \hat\tau_{\va,m}^{\mathrm{Hajek}}(u)-\tau_{\va}(u)
    &=
    \sum_{k=1}^K
    \hat\lambda_k
    \frac{1}{m}
    \sum_{t=1}^m
    \1\{S_t\in C_k\}\gamma_{\va,q}(S_t)u(S_t)
    -
    \sum_{k=1}^K \pi_k\mu_k
    \\
    &=
    \sum_{k=1}^K
    \hat\lambda_k
    \frac{1}{m}
    \sum_{t=1}^m
    \1\{S_t\in C_k\}
    \bigl\{
        \gamma_{\va,q}(S_t)u(S_t)-\mu_k
    \bigr\}
    +
    \sum_{k=1}^K
    (\hat\lambda_k\hat\pi_k-\pi_k)\mu_k
    \\
    &=
    \sum_{k=1}^K \hat\lambda_k A_{mk}
    -
    r_{m,\mathrm{miss}}
    \\
    &=
    \sum_{k=1}^K A_{mk}
    +
    r_{m,\mathrm{norm}}
    -
    r_{m,\mathrm{miss}}
    \\
    &=
    \frac1m\sum_{t=1}^m \xi_t
    +
    r_{m,\mathrm{norm}}
    -
    r_{m,\mathrm{miss}}.
\end{align*}
Since \(\xi_t=\psi_{h_{\mathscr C}}(S_t;u)\), this yields the exact decomposition
\begin{equation}
    \hat\tau_{\va,m}^{\mathrm{Hajek}}(u)-\tau_{\va}(u)
    =
    \frac1m\sum_{t=1}^m \psi_{h_{\mathscr C}}(S_t;u)
    +
    r_m,
    \qquad
    r_m:=r_{m,\mathrm{norm}}-r_{m,\mathrm{miss}}.
    \label{eq:aug-hajek-proof-decomp}
\end{equation}

\paragraph{Step 3: Bound the normalization remainder.}
Write the average cellwise second moment in the theorem display as
\[
    \Gamma_{2,\mathscr C}^2
    :=
    \frac{1}{K_+}
    \sum_{k:\pi_k>0}
    \E_q\!\left[
        \gamma_{\va,q}(S)^2\mid S\in C_k
    \right],
    \qquad
    \Gamma_\infty
    :=
    \sup_{S:q(S)>0}|\gamma_{\va,q}(S)|.
\]
Cells with \(\pi_k=0\) satisfy \(N_k=0\) and \(A_{mk}=0\) almost surely, so
only cells with \(\pi_k>0\) contribute to the bounds.

For \(N_k>0\), write
\[
    \bar\xi_{mk}
    :=
    \frac{1}{N_k}
    \sum_{t:S_t\in C_k}\xi_t,
\]
and set \(\bar\xi_{mk}:=0\) when \(N_k=0\). Then
\(A_{mk}=\hat\pi_k\bar\xi_{mk}\), and the normalization factor cancels:
\[
    (\hat\lambda_k-1)A_{mk}
    =
    (\pi_k-\hat\pi_k)\bar\xi_{mk}.
\]
When \(N_k=0\), both sides are zero by definition. When \(N_k>0\),
\[
    (\hat\lambda_k-1)A_{mk}
    =
    \left(
        \frac{\pi_k}{\hat\pi_k}-1
    \right)
    \hat\pi_k\bar\xi_{mk}
    =
    (\pi_k-\hat\pi_k)\bar\xi_{mk}.
\]

Let \(\mathcal F_m\) be the sigma-field generated by the cell labels
\(\{k(S_t):1\le t\le m\}\). Conditional on \(\mathcal F_m\), the within-cell
averages \(\bar\xi_{mk}\) have mean zero, and the cross terms vanish. Hence
\[
    \E(r_{m,\mathrm{norm}}^2)
    =
    \sum_{k:\pi_k>0}
    \E\!\left[
        (\pi_k-\hat\pi_k)^2\bar\xi_{mk}^2
    \right].
\]
Moreover, conditional on \(N_k\), the \(N_k\) observations falling in \(C_k\) are
i.i.d.\ from \(q(\cdot\mid C_k)\), so
\[
\begin{aligned}
    \E(\bar\xi_{mk}^2\mid N_k)
    &=
    \frac{1}{N_k}
    \Var_q\!\left\{
        \gamma_{\va,q}(S)u(S)\mid S\in C_k
    \right\}
    \1\{N_k>0\}
    \\
    &\le
    \frac{\|u\|_\infty^2}{N_k}
    \E_q\!\left[
        \gamma_{\va,q}(S)^2\mid S\in C_k
    \right]
    \1\{N_k>0\}.
\end{aligned}
\]
For \(N_k\sim\mathrm{Bin}(m,\pi_k)\), the following binomial moment bound
is uniform in \(m\) and \(k\):
\[
    \E\!\left[
        \frac{(N_k-m\pi_k)^2}{N_k}\1\{N_k>0\}
    \right]
    \le
    C.
\]
Indeed, on \(\{N_k\ge m\pi_k/2\}\), the expectation is bounded by
\(2(m\pi_k)^{-1}\Var(N_k)\le 2\), while on
\(\{0<N_k<m\pi_k/2\}\) it is bounded by
\((m\pi_k)^2P(N_k<m\pi_k/2)\le (m\pi_k)^2e^{-m\pi_k/8}\le C\).
Therefore,
\[
    \E\!\left[
        (\pi_k-\hat\pi_k)^2\bar\xi_{mk}^2
    \right]
    \le
    \frac{C\|u\|_\infty^2}{m^2}
    \E_q\!\left[
        \gamma_{\va,q}(S)^2\mid S\in C_k
    \right],
\]
and hence
\[
    \E(r_{m,\mathrm{norm}}^2)
    \le
    \frac{C\|u\|_\infty^2 K_+
    \Gamma_{2,\mathscr C}^2}{m^2}.
\]

\paragraph{Step 4: Bound the missing-cell remainder.}
Since \(P(N_k=0)=(1-\pi_k)^m\le e^{-m \pi_k}\),
\[
    |r_{m,\mathrm{miss}}|
    \le
    \|u\|_\infty\Gamma_\infty
    \sum_{k:\pi_k>0}\pi_k\1\{N_k=0\}.
\]
Because \(\sum_{k:\pi_k>0}\pi_k\1\{N_k=0\}\le1\), we have
\[
    \E(r_{m,\mathrm{miss}}^2)
    \le
    \|u\|_\infty^2\Gamma_\infty^2
    \sum_{k:\pi_k>0}\pi_kP(N_k=0)
    \le
    \|u\|_\infty^2\Gamma_\infty^2 e^{-m \pi_{\min}}.
\]

Combining the two bounds with \((x+y)^2\le 2x^2+2y^2\), and absorbing
numerical constants into \(C_1,C_2>0\), yields
\[
\begin{aligned}
    \E(r_m^2)
    &\le
    2\E(r_{m,\mathrm{norm}}^2)+2\E(r_{m,\mathrm{miss}}^2)
    \\
    &\le
    C_1\|u\|_\infty^2
    \left\{
        \frac{K_+\Gamma_{2,\mathscr C}^2}{m^2}
        +
        \Gamma_\infty^2 e^{-C_2 m \pi_{\min}}
    \right\}
\end{aligned}
\]
for universal constants \(C_1,C_2>0\). This is
\eqref{eq:aug-hajek-finite-remainder}.
\end{proof}

\subsection{Proof of Theorem~\ref{prop:rr-regular-exact}}

The proof has two parts. We first establish a same-sample product bound that
controls the interaction between the empirical Gram fluctuation and the
empirical residual vector. We then apply this lemma to the whitened WLS design
to bound the remainder vector \(\vr_m\). Both steps rely on a concentration
fact for centered isotropic rows, which we state first; its proof is deferred
to \Secref{app:proof-isotropic-row-concentration}.

\begin{lemma}[Concentration for centered isotropic rows]
\label{lem:isotropic-row-concentration}
Let \(\vx_1,\dots,\vx_m\) be independent random vectors in \(\mathbb R^r\)
with \(\E(\vx_t\vx_t^\top)=\mI_r\) and \(\|\vx_t\|_2^2\le L\) almost surely,
and write \(\mB_t:=\vx_t\vx_t^\top-\mI_r\). There exist numerical constants
\(c,C>0\) such that:
\begin{enumerate}
    \item[(a)] for every \(t\in(0,1]\),
    \[
        P\!\left(
            \left\|\frac1m\sum_{s=1}^m\mB_s\right\|_{\mathrm{op}}>t
        \right)
        \le
        2r\exp\{-cmt^2/L\};
    \]
    \item[(b)] for every deterministic \(I\subseteq\{1,\dots,m\}\) with
    \(|I|=n_I\),
    \[
        \E\left\|\sum_{i\in I}\mB_i\right\|_{\mathrm{op}}^2
        \le
        C\left\{n_IL\log(er)+L^2\log^2(er)\right\}.
    \]
\end{enumerate}
\end{lemma}

The good-event analysis below also involves the product of two empirical
averages computed from the same sampled coalitions. The next lemma isolates
the bound needed for that product. We state it here and defer its proof to
\Secref{app:proof-same-sample-product}.

\begin{lemma}[Same-sample covariance--residual product]\label{lem:same-sample-leverage-product}
Let \((\vx,\vb)\) be a random pair in an \(r\)-dimensional Euclidean space. Assume
\[
    \E(\vx\vx^\top)=\mI_r,
    \qquad
    \|\vx\|_2^2\le L,
    \qquad
    \E(\vb)=0,
    \qquad
    M_{\vb}:=\E\|\vb\|_2^2<\infty.
\]
For iid copies \((\vx_t,\vb_t)_{t=1}^m\), define
\[
    \bar\mB_m
    :=
    \frac1m\sum_{t=1}^m
    \left(
        \vx_t\vx_t^\top-\mI_r
    \right),
    \qquad
    \bar\vb_m
    :=
    \frac1m\sum_{t=1}^m\vb_t.
\]
Then, for a numerical constant \(C>0\),
\begin{equation}
    \E\!\left[
        \|\bar\mB_m\|_{\mathrm{op}}^2
        \|\bar\vb_m\|_2^2
    \right]
    \le
    C
    \left\{
        \frac{L\log(e r)}{m^2}
        +
        \frac{L^2\log^2(e r)}{m^3}
    \right\}
    M_{\vb}.
    \label{eq:same-sample-product}
\end{equation}
\end{lemma}

\begin{proof}[Proof of Theorem~\ref{prop:rr-regular-exact}]
We organize the proof into four steps. The first step identifies the common
first-order term and the vector remainder. The remaining steps bound that
remainder on a good conditioning event and on its complement.

\paragraph{Step 1: Setup and first-order decomposition.}
Fix \(m\), write \(\lambda_m:=\lambda/m\), and let
\[
    \mM:=\operatorname{range}(\mG).
\]
All inverses below are taken on \(\mM\). Let \(\vbeta^\star\) be the
minimum-Euclidean-norm element of
\[
    \argmin_{\vbeta\in\mathbb R^{d_{\vz}}}
    \E_q\!\left[
        \tilde w(S)\{u(S)-\vz(S)^\top\vbeta\}^2
    \right],
\]
and define
\[
    h^\star(S):=\vz(S)^\top\vbeta^\star,
    \qquad
    \varepsilon(S):=u(S)-h^\star(S).
\]
Also define
\[
\begin{aligned}
    \vb_u
    &:=
    \E_q\!\left[
        \tilde w(S)\vz(S)u(S)
    \right],
    \qquad
    \hat\mG_m
    :=
    \frac1m\sum_{t=1}^m
    \tilde w(S_t)\vz(S_t)\vz(S_t)^\top,
    \\
    \hat\vb_m
    &:=
    \frac1m\sum_{t=1}^m
    \tilde w(S_t)\vz(S_t)u(S_t),
    \qquad
    \hat\delta
    :=
    \frac1m\sum_{t=1}^m
    \tilde w(S_t)\vz(S_t)\varepsilon(S_t).
\end{aligned}
\]
Since \(\vb_u\in\mM\), the minimum-norm normal-equation solution is
\[
    \vbeta^\star=\mG^+\vb_u,
    \qquad
    \E_q\!\left[
        \tilde w(S)\vz(S)\varepsilon(S)
    \right]=0.
\]

Let \(\va_j\) and \(\vc_j\) denote the \(j\)-th columns of \(\mA\) and
\(\mC_{\mA}\), respectively. For an arbitrary set function \(v\), let
\(\vbeta^\star(v)\) be the minimum-norm population WLS solution. Then
\[
    \vbeta^\star(v)
    =
    \mG^+
    \E_q\!\left[
        \tilde w(S)\vz(S)v(S)
    \right].
\]
By the assumed WLS representation, this implies
\[
    \tau_{\va_j}(v)
    =
    \vc_j^\top
    \mG^+
    \E_q\!\left[
        \tilde w(S)\vz(S)v(S)
    \right]
    \qquad
    \text{for every } j \text{ and every } v.
\]
Comparing with the weighted-average representation \eqref{eq:tau-a-ipw} for
the scalar target \(\tau_{\va_j}\) yields
\begin{equation}
    \gamma_{\va_j,q}(S)
    =
    \tilde w(S)\vz(S)^\top\mG^+\vc_j
    \qquad
    \text{for every } S \text{ with } q(S)>0.
    \label{eq:rr-direct-score-coordinate}
\end{equation}
Applying the same readout identity to \(v=h^\star\) gives
\[
    \vtau_{\mA}(h^\star)
    =
    \mC_{\mA}^\top\vbeta^\star
    =
    \vtau_{\mA}(u),
\]
so \(h^\star\) is the common surrogate for all coordinates.

Since \(u(S)=\vz(S)^\top\vbeta^\star+\varepsilon(S)\), the empirical right-hand
side in the ridge normal equation decomposes as
\begin{align*}
    \hat\vb_m
    &=
    \frac1m\sum_{t=1}^m
    \tilde w(S_t)\vz(S_t)u(S_t)
    \\
    &=
    \left\{
        \frac1m\sum_{t=1}^m
        \tilde w(S_t)\vz(S_t)\vz(S_t)^\top
    \right\}\vbeta^\star
    +
    \frac1m\sum_{t=1}^m
    \tilde w(S_t)\vz(S_t)\varepsilon(S_t)
    \\
    &=
    \hat\mG_m\vbeta^\star+\hat\delta .
\end{align*}
The ridge estimator satisfies
\[
    \left(
        \hat\mG_m+\lambda_m\mI_{d_{\vz}}
    \right)\hat\vbeta_{m,\lambda}
    =
    \hat\vb_m .
\]
Therefore, using
\(\hat\mG_m=(\hat\mG_m+\lambda_m\mI_{d_{\vz}})-\lambda_m\mI_{d_{\vz}}\),
\begin{align*}
    \hat\vbeta_{m,\lambda}-\vbeta^\star
    &=
    \left(
        \hat\mG_m+\lambda_m\mI_{d_{\vz}}
    \right)^{-1}
    \left(
        \hat\mG_m\vbeta^\star+\hat\delta
    \right)
    -
    \vbeta^\star
    \\
    &=
    \left[
        \left(
            \hat\mG_m+\lambda_m\mI_{d_{\vz}}
        \right)^{-1}
        \hat\mG_m
        -
        \mI_{d_{\vz}}
    \right]\vbeta^\star
    +
    \left(
        \hat\mG_m+\lambda_m\mI_{d_{\vz}}
    \right)^{-1}
    \hat\delta
    \\
    &=
    \left(
        \hat\mG_m+\lambda_m\mI_{d_{\vz}}
    \right)^{-1}
    \hat\delta
    -
    \lambda_m
    \left(
        \hat\mG_m+\lambda_m\mI_{d_{\vz}}
    \right)^{-1}
    \vbeta^\star.
\end{align*}
Adding and subtracting the population linear term
\(\mC_{\mA}^\top\mG^+\hat\delta\) in the target error gives
\begin{align*}
    \hat{\vtau}_{\mA, m}^{\mathrm{WLS}, \lambda}-\vtau_{\mA}(u)
    &=
    \mC_{\mA}^\top
    \bigl\{
        \hat\vbeta_{m,\lambda}-\vbeta^\star
    \bigr\}
    \\
    &=
    \mC_{\mA}^\top
    \left(
        \hat\mG_m+\lambda_m\mI_{d_{\vz}}
    \right)^{-1}
    \hat\delta
    -
    \lambda_m\mC_{\mA}^\top
    \left(
        \hat\mG_m+\lambda_m\mI_{d_{\vz}}
    \right)^{-1}
    \vbeta^\star
    \\
    &=
    \mC_{\mA}^\top\mG^+\hat\delta
    +
    \underbrace{
    \mC_{\mA}^\top
    \left[
        \left(
            \hat\mG_m+\lambda_m\mI_{d_{\vz}}
        \right)^{-1}
        -
        \mG^+
    \right]\hat\delta
    -
    \lambda_m\mC_{\mA}^\top
    \left(
        \hat\mG_m+\lambda_m\mI_{d_{\vz}}
    \right)^{-1}
    \vbeta^\star
    }_{=:\vr_m}.
\end{align*}
For each coordinate \(j\), \eqref{eq:rr-direct-score-coordinate} gives
\[
    \vc_j^\top\mG^+\hat\delta
    =
    \frac1m\sum_{t=1}^m
    \gamma_{\va_j,q}(S_t)\varepsilon(S_t)
    =
    \frac1m\sum_{t=1}^m
    \psi_{\va_j,q,h^\star}(S_t;u).
\]
Thus each coordinate has the regular expansion with surrogate \(h^\star\), and
\(\vr_m\) is the common remainder vector. It remains to bound
\(\E_q(\|\vr_m\|_2^2)\).

\paragraph{Step 2: Reduce the problem to a whitened coefficient error.}
For every \(S\) with \(\tilde w(S)>0\), we have \(\vz(S)\in\mM\). Hence
\(\operatorname{range}(\hat\mG_m)\subseteq \mM\), while \(\hat\delta\) and
\(\vbeta^\star\) lie in \(\mM\). Let \(r=\operatorname{rank}(\mG)\), and choose
an orthonormal basis \(\mU\in\mathbb R^{d_{\vz}\times r}\) for \(\mM\). Define
\[
    \mG_{\mM}:=\mU^\top\mG\mU,
    \qquad
    \hat\mG_{\mM,m}:=\mU^\top\hat\mG_m\mU .
\]
The matrix \(\mG_{\mM}\) is positive definite, and its eigenvalues are the
nonzero eigenvalues of \(\mG\). Work in this \(r\)-dimensional coordinate
system and define
\[
    \mH_m
    :=
    \mG_{\mM}^{-1/2}\hat\mG_{\mM,m}\mG_{\mM}^{-1/2},
    \qquad
    \vy_m:=\mG_{\mM}^{-1/2}\mU^\top\hat\delta,
    \qquad
    \vtheta_m:=\mG_{\mM}^{1/2}\mU^\top\vbeta^\star.
\]
Let
\[
    \mathcal E_m
    :=
    \left\{
        \|\mH_m-\mI_r\|_{\mathrm{op}}
        \le \frac12
    \right\}.
\]
In these coordinates, define the \(r\)-vector
\[
    \vDelta_m
    :=
    \left[
        \left(
            \mH_m+\lambda_m\mG_{\mM}^{-1}
        \right)^{-1}
        -
        \mI_r
    \right]\vy_m
    -
    \lambda_m
    \left(
        \mH_m+\lambda_m\mG_{\mM}^{-1}
    \right)^{-1}
    \mG_{\mM}^{-1}\vtheta_m.
\]
We now show that
\begin{equation}
    \vr_m
    =
    \mC_{\mA}^\top\mU\mG_{\mM}^{-1/2}\vDelta_m.
    \label{eq:rm-whitened-coordinates}
\end{equation}
Write \(P_{\mM^\perp}:=\mI_{d_{\vz}}-\mU\mU^\top\) for the orthogonal
projection onto \(\mM^\perp\). Since
\(\operatorname{range}(\hat\mG_m)\subseteq\mM\), the ridge matrix decomposes
as
\(
    \hat\mG_m+\lambda_m\mI_{d_{\vz}}
    =
    \mU(\hat\mG_{\mM,m}+\lambda_m\mI_r)\mU^\top
    +
    \lambda_m P_{\mM^\perp}
\),
where the two terms act on the orthogonal subspaces \(\mM\) and
\(\mM^\perp\), respectively. Inverting each block gives
\[
    \left(\hat\mG_m+\lambda_m\mI_{d_{\vz}}\right)^{-1}
    =
    \mU(\hat\mG_{\mM,m}+\lambda_m\mI_r)^{-1}\mU^\top
    +
    \lambda_m^{-1}P_{\mM^\perp}.
\]
Because \(\hat\delta,\vbeta^\star\in\mM\), we have
\(P_{\mM^\perp}\hat\delta=P_{\mM^\perp}\vbeta^\star=0\), and hence
\[
\begin{aligned}
    \left(\hat\mG_m+\lambda_m\mI_{d_{\vz}}\right)^{-1}\hat\delta
    &=
    \mU(\hat\mG_{\mM,m}+\lambda_m\mI_r)^{-1}\mU^\top\hat\delta,\\
    \left(\hat\mG_m+\lambda_m\mI_{d_{\vz}}\right)^{-1}\vbeta^\star
    &=
    \mU(\hat\mG_{\mM,m}+\lambda_m\mI_r)^{-1}\mU^\top\vbeta^\star.
\end{aligned}
\]
Substituting these two identities, together with
\(\mG^+=\mU\mG_{\mM}^{-1}\mU^\top\) and
\((\hat\mG_{\mM,m}+\lambda_m\mI_r)^{-1}=\mG_{\mM}^{-1/2}(\mH_m+\lambda_m\mG_{\mM}^{-1})^{-1}\mG_{\mM}^{-1/2}\),
into the definition of \(\vr_m\) and collecting terms yields
\eqref{eq:rm-whitened-coordinates}.
Consequently,
\[
    \|\vr_m\|_2^2
    \le
    \sigma_{\max}(\mC_{\mA}^\top\mU\mG_{\mM}^{-1}\mU^\top\mC_{\mA})
    \|\vDelta_m\|_2^2
    =
    R_{\mA}^2\|\vDelta_m\|_2^2.
\]
Therefore it suffices to bound \(\E\|\vDelta_m\|_2^2\).

On \(\mathcal E_m\), the matrix \(\mH_m+\lambda_m\mG_{\mM}^{-1}\) has minimum eigenvalue
at least \(1/2\), so
\[
    \left\|
        \left(
            \mH_m+\lambda_m\mG_{\mM}^{-1}
        \right)^{-1}
        -
        \mI_r
    \right\|_{\mathrm{op}}
    \le
    2
    \left\{
        \|\mH_m-\mI_r\|_{\mathrm{op}}
        +
        \lambda_m\|\mG_{\mM}^{-1}\|_{\mathrm{op}}
    \right\},
\]
and
\[
    \left\|
        \left(
            \mH_m+\lambda_m\mG_{\mM}^{-1}
        \right)^{-1}
    \right\|_{\mathrm{op}}
    \le
    2.
\]
Thus, on \(\mathcal E_m\),
\[
    \|\vDelta_m\|_2^2\1(\mathcal E_m)
    \le
    C
    \left[
        \|\mH_m-\mI_r\|_{\mathrm{op}}^2\|\vy_m\|_2^2
        +
        \frac{\lambda^2}{m^2\{\sigma_{\min}^+(\mG)\}^2}\|\vy_m\|_2^2
        +
        \frac{\lambda^2}{m^2\{\sigma_{\min}^+(\mG)\}^2}\|\vtheta_m\|_2^2
    \right].
\]

\paragraph{Step 3: Bound the good event.}
Write
\[
    \vx(S):=\tilde w(S)^{1/2}\mG_{\mM}^{-1/2}\mU^\top\vz(S)
    \in\mathbb R^r .
\]
Then \(\E_q[\vx(S)\vx(S)^\top]=\mI_r\) and
\(\|\vx(S)\|_2^2=\ell(S)\le L\), so \(\mH_m-\mI_r=\frac1m\sum_{t=1}^m\{\vx(S_t)\vx(S_t)^\top-\mI_r\}\)
is of the form covered by Lemma~\ref{lem:isotropic-row-concentration}. Part~(a) of that lemma with \(t=1/2\), together with \(r\le d_{\vz}\), gives
\[
    P(\mathcal E_m^c)
    =
    P\!\left(\|\mH_m-\mI_r\|_{\mathrm{op}}>\tfrac12\right)
    \le
    2r\exp\{-cm/(4L)\}
    \le
    C\exp\{-c m/L+\log d_{\vz}\}.
\]
Also,
\[
    \mH_m-\mI_r
    =
    \frac1m\sum_{t=1}^m
    \{\vx(S_t)\vx(S_t)^\top-\mI_r\},
    \qquad
    \vy_m
    =
    \frac1m\sum_{t=1}^m
    \tilde w(S_t)^{1/2}\vx(S_t)\varepsilon(S_t),
\]
and both summands have mean zero. Applying
Lemma~\ref{lem:same-sample-leverage-product} in dimension \(r\le d_{\vz}\), with
\[
    \vb(S):=\tilde w(S)^{1/2}\vx(S)\varepsilon(S),
    \qquad
    M_{\vb}
    =
    \E_q\!\left[
        \tilde w(S)\ell(S)\varepsilon(S)^2
    \right],
\]
gives
\[
    \E\!\left[
        \|\mH_m-\mI_r\|_{\mathrm{op}}^2
        \|\vy_m\|_2^2
    \right]
    \le
    \frac{C L\log d_{\vz}}{m^2}
    \left(1+\frac{L\log d_{\vz}}{m}\right)
    \E_q\!\left[
        \tilde w(S)\ell(S)\varepsilon(S)^2
    \right].
\]
Moreover,
\[
    \E\|\vy_m\|_2^2
    =
    \frac1m
    \E_q\!\left[
        \tilde w(S)\ell(S)\varepsilon(S)^2
    \right],
    \qquad
    \|\vtheta_m\|_2^2
    =
    (\vbeta^\star)^\top\mG\vbeta^\star
    \le
    \E_q\!\left[
        \tilde w(S)u(S)^2
    \right],
\]
where the last inequality is the Pythagorean property of weighted least
squares. The same projection property also gives
\[
    \E_q\!\left[
        \tilde w(S)\ell(S)\varepsilon(S)^2
    \right]
    \le
    L
    \E_q\!\left[
        \tilde w(S)\varepsilon(S)^2
    \right]
    \le
    L
    \E_q\!\left[
        \tilde w(S)u(S)^2
    \right].
\]
Taking expectations in the bound from Step~2, substituting the last three
displays, and using \(\lambda\asymp \sigma_{\min}^+(\mG)\), we obtain
\[
    \E\!\left[
        \|\vDelta_m\|_2^2\1(\mathcal E_m)
    \right]
    \le
    C
    \left[
        \frac{L\log d_{\vz}}{m^2}
        \left(1+\frac{L\log d_{\vz}}{m}\right)
        \E_q\!\left[
            \tilde w(S)\ell(S)\varepsilon(S)^2
        \right]
        +
        \frac{\|u\|_{2,\tilde w}^2}{m^2}
    \right].
\]
Because
\[
    \E_q\!\left[
        \tilde w(S)\ell(S)\varepsilon(S)^2
    \right]
    \le
    L\|u\|_{2,\tilde w}^2,
\]
the second term is absorbed into the first after adjusting constants. Hence
\[
    \E\!\left[
        \|\vDelta_m\|_2^2\1(\mathcal E_m)
    \right]
    \le
    \frac{C L^2\log d_{\vz}}{m^2}
    \left(1+\frac{L\log d_{\vz}}{m}\right)
    \|u\|_{2,\tilde w}^2.
\]

\paragraph{Step 4: Bound the bad event.}
On \(\mathcal E_m^c\), ridge still gives deterministic stability. Since
\[
    \left\|
        \left(
            \mH_m+\lambda_m\mG_{\mM}^{-1}
        \right)^{-1}
    \right\|_{\mathrm{op}}
    \le
    \frac{\sigma_{\max}(\mG_{\mM})}{\lambda_m}
    =
    \frac{m\sigma_{\max}(\mG)}{\lambda},
\]
we have
\[
    \left\|
        \left(
            \mH_m+\lambda_m\mG_{\mM}^{-1}
        \right)^{-1}
        -
        \mI_r
    \right\|_{\mathrm{op}}
    \le
    1+\frac{m\sigma_{\max}(\mG)}{\lambda}.
\]
Also, for every sample,
\[
    \|\vy_m\|_2
    \le
    \frac1m\sum_{t=1}^m
    \tilde w(S_t)^{1/2}\|\vx(S_t)\|_2|\varepsilon(S_t)|
    \le
    \left\{
        L\sup_{S:q(S)>0}\tilde w(S)\varepsilon(S)^2
    \right\}^{1/2}.
\]
Since \(\varepsilon(S)=u(S)-\vz(S)^\top\vbeta^\star\), for every \(S\) with
\(q(S)>0\),
\[
    \tilde w(S)\varepsilon(S)^2
    \le
    2\tilde w(S)u(S)^2
    +
    2\tilde w(S)\{\vz(S)^\top\vbeta^\star\}^2.
\]
Moreover,
\[
    \tilde w(S)\{\vz(S)^\top\vbeta^\star\}^2
    \le
    \ell(S)(\vbeta^\star)^\top\mG\vbeta^\star
    \le
    L
    \E_q\!\left[
        \tilde w(T)u(T)^2
    \right],
\]
again by the Pythagorean property of weighted least squares. Hence
\[
    \sup_{S:q(S)>0}\tilde w(S)\varepsilon(S)^2
    \le
    2\|u\|_{\infty,\tilde w}^2
    +
    2L\|u\|_{2,\tilde w}^2.
\]
Moreover, \(L\ge1\) in the nonzero-rank case because
\(\E_q\{\ell(S)\}=\operatorname{rank}(\mG)\). Since
\(\|u\|_{2,\tilde w}^2\le \|u\|_{\infty,\tilde w}^2\), the preceding display is
at most
\[
    4L\|u\|_{\infty,\tilde w}^2.
\]
Therefore
\[
\begin{aligned}
    \|\vDelta_m\|_2^2\1(\mathcal E_m^c)
    &\le
    C
    \left\{
        \left(
            1+\frac{m\sigma_{\max}(\mG)}{\lambda}
        \right)^2
        L\sup_{S:q(S)>0}\tilde w(S)\varepsilon(S)^2
        \right.
    \\
    &\hspace{5em}\left.
        +
        \left(
            \frac{\sigma_{\max}(\mG)}{\sigma_{\min}^+(\mG)}
        \right)^2
        \|u\|_{2,\tilde w}^2
    \right\}
    \1(\mathcal E_m^c).
\end{aligned}
\]
Taking expectations, using
\(\|u\|_{2,\tilde w}^2\le \|u\|_{\infty,\tilde w}^2\),
\(\lambda\asymp \sigma_{\min}^+(\mG)\), and the Bernstein bound for
\(P(\mathcal E_m^c)\), we obtain
\[
    \E\!\left[
        \|\vDelta_m\|_2^2\1(\mathcal E_m^c)
    \right]
    \le
    C
    L^2
    e^{-c m/L}\,
    d_{\vz} m^2 \kappa_{\mG}^2
    \|u\|_{\infty,\tilde w}^2,
\]
where the factor \(L^2\) comes from
\[
    L\sup_{S:q(S)>0}\tilde w(S)\varepsilon(S)^2
    \le
    4L^2\|u\|_{\infty,\tilde w}^2 .
\]

Combining the good- and bad-event bounds yields
\[
    \E\|\vDelta_m\|_2^2
    \le
    C
    \left[
        \frac{L^2\log d_{\vz}}{m^2}
        \left(1+\frac{L\log d_{\vz}}{m}\right)
        \|u\|_{2,\tilde w}^2
        +
        L^2
        e^{-c m/L}\,
        d_{\vz} m^2 \kappa_{\mG}^2
        \|u\|_{\infty,\tilde w}^2
    \right].
\]
Finally,
\[
    \E_q\!\left[
        \|\vr_m\|_2^2
    \right]
    \le
    R_{\mA}^2
    \E\|\vDelta_m\|_2^2,
\]
so \eqref{eq:rr-delta-bound} follows.
\end{proof}

\subsubsection{Proof of Lemma~\ref{lem:isotropic-row-concentration}}
\label{app:proof-isotropic-row-concentration}

\begin{proof}
Each \(\mB_t\) is self-adjoint and mean zero, with \(\|\mB_t\|_{\mathrm{op}}\le
L\), since the eigenvalues of \(\vx_t\vx_t^\top-\mI_r\) are \(\|\vx_t\|_2^2-1\)
and \(-1\), and \(L\ge1\). Moreover,
\(
    \E(\mB_t^2)
    =
    \E\{\|\vx_t\|_2^2\vx_t\vx_t^\top\}-\mI_r
    \preceq
    L\mI_r
\),
so for any \(I\subseteq\{1,\dots,m\}\), the matrix Bernstein variance
parameter satisfies \(\|\sum_{i\in I}\E(\mB_i^2)\|_{\mathrm{op}}\le n_IL\). By
the matrix Bernstein inequality for independent, mean-zero, self-adjoint
matrices (e.g.\ \citealp[Theorem~5.4.1]{Vershynin2018-hdp}), for every
\(u\ge0\),
\begin{equation}
    P\!\left(
        \left\|\sum_{i\in I}\mB_i\right\|_{\mathrm{op}}\ge u
    \right)
    \le
    2r\exp\left\{-\frac{u^2/2}{n_IL+Lu/3}\right\}.
    \label{eq:iso-tail}
\end{equation}
Note that the exponent in \eqref{eq:iso-tail} is governed by the variance
parameter \(n_IL\) rather than the uniform bound \(L\).

For part~(a), take \(I=\{1,\dots,m\}\) and \(u=mt\) in \eqref{eq:iso-tail}.
For \(t\in(0,1]\), \(n_IL+Lmt/3\le\frac43mL\), so
\[
    P\!\left(
        \left\|\frac1m\sum_{s=1}^m\mB_s\right\|_{\mathrm{op}}>t
    \right)
    \le
    2r\exp\{-3mt^2/(8L)\}
    \le
    2r\exp\{-cmt^2/L\}.
\]

For part~(b), rewrite \eqref{eq:iso-tail} with \(a:=\log(2er)\): for every
\(s\ge0\),
\[
    P\!\left(
        \left\|\sum_{i\in I}\mB_i\right\|_{\mathrm{op}}
        >
        C_0\{\sqrt{n_IL(a+s)}+L(a+s)\}
    \right)
    \le
    e^{-s}.
\]
Let \(W\sim\operatorname{Exp}(1)\) be independent of the data; this gives the
stochastic domination
\(
    \|\sum_{i\in I}\mB_i\|_{\mathrm{op}}
    \preceq_{\mathrm{st}}
    C_0\{\sqrt{n_IL(a+W)}+L(a+W)\}
\).
Since \((\sqrt p+q)^2\le2p+2q^2\) and \(\E(W)=1\), \(\E(W^2)=2\),
\[
    \E\left\|\sum_{i\in I}\mB_i\right\|_{\mathrm{op}}^2
    \le
    2C_0^2\{n_IL(a+1)+L^2(a^2+2a+2)\}
    \le
    C\{n_ILa+L^2a^2\},
\]
using \(a\ge\log(2e)>1\) in the last step. Since \(a=\log(2er)\le2\log(er)\)
for \(r\ge1\), this proves part~(b).
\end{proof}

\subsubsection{Proof of Lemma~\ref{lem:same-sample-leverage-product}}
\label{app:proof-same-sample-product}

This lemma is a same-sample version of a standard decoupling estimate for
second-order \(U\)-statistic-type sums; see, for example, the decoupling theory
for \(U\)-statistics and \(U\)-processes in \citet[Chapter~3]{de-la-Pena1999-bt}.
We give a self-contained proof because only this elementary split-sample form is
needed here.

\begin{proof}
Write
\[
    \mB_t:=\vx_t\vx_t^\top-\mI_r.
\]
We view \(\bar\mB_m\otimes \bar\vb_m\) as an element of the injective tensor
product with norm
\[
    \|\mT\|_{\vee}
    :=
    \sup_{\|a\|_2=\|b\|_2=\|c\|_2=1}
    |\mT(a,b,c)|.
\]
For a rank-one tensor \(\mB\otimes\vb\), this norm is
\[
    \|\mB\otimes\vb\|_{\vee}
    =
    \|\mB\|_{\mathrm{op}}\|\vb\|_2.
\]
Hence the left-hand side of \eqref{eq:same-sample-product} equals
\[
    \E\|\bar\mB_m\otimes\bar\vb_m\|_{\vee}^2.
\]
Decompose
\[
    \bar\mB_m\otimes\bar\vb_m
    =
    \frac1{m^2}
    \sum_{i\ne j}\mB_i\otimes\vb_j
    +
    \frac1{m^2}
    \sum_{i=1}^m\mB_i\otimes\vb_i
    =:\mU_m+\mD_m.
\]

\paragraph{Step 1: Control the off-diagonal term.}
Let \(\delta_1,\dots,\delta_m\) be iid \(\operatorname{Bernoulli}(1/2)\),
independent of the data, and define
\[
    \mU_m(\delta)
    :=
    \frac1{m^2}
    \sum_{i\ne j}
    \delta_i(1-\delta_j)\mB_i\otimes\vb_j.
\]
Since \(\E_\delta[\delta_i(1-\delta_j)]=1/4\) for \(i\ne j\), conditional on
the data,
\[
    \mU_m=4\E_\delta\{\mU_m(\delta)\}.
\]
Jensen's inequality therefore gives
\[
    \E\|\mU_m\|_{\vee}^2
    \le
    16\E\E_\delta\|\mU_m(\delta)\|_{\vee}^2.
\]
For a fixed split, let \(I=\{i:\delta_i=1\}\) and \(J=\{j:\delta_j=0\}\).
Then \(I\cap J=\varnothing\), so
\[
    \mU_m(\delta)
    =
    \frac1{m^2}
    \left(\sum_{i\in I}\mB_i\right)
    \otimes
    \left(\sum_{j\in J}\vb_j\right).
\]
The two sums involve disjoint observations and are therefore independent.
Thus, for every fixed split,
\[
\begin{aligned}
    \E\|\mU_m(\delta)\|_{\vee}^2
    &=
    \frac1{m^4}
    \E\left\|\sum_{i\in I}\mB_i\right\|_{\mathrm{op}}^2
    \E\left\|\sum_{j\in J}\vb_j\right\|_2^2 .
\end{aligned}
\]
Because \(\E\vb=0\),
\[
    \E\left\|\sum_{j\in J}\vb_j\right\|_2^2
    =
    |J|M_{\vb}
    \le
    mM_{\vb}.
\]
Lemma~\ref{lem:isotropic-row-concentration}(b), applied to the deterministic
subset \(I\), gives
\[
    \E\left\|\sum_{i\in I}\mB_i\right\|_{\mathrm{op}}^2
    \le
    C
    \left\{
        |I|L\log(e r)
        +
        L^2\log^2(e r)
    \right\}
    \le
    C
    \left\{
        mL\log(e r)
        +
        L^2\log^2(e r)
    \right\}.
\]
Consequently, uniformly over the split,
\[
    \E\|\mU_m(\delta)\|_{\vee}^2
    \le
    C
    \left\{
        \frac{L\log(e r)}{m^2}
        +
        \frac{L^2\log^2(e r)}{m^3}
    \right\}
    M_{\vb}.
\]
Averaging over \(\delta\) gives the same bound for \(\E\|\mU_m\|_\vee^2\).

\paragraph{Step 2: Control the diagonal term.}
We next bound \(\mD_m\). First,
\[
    \mD_m
    =
    \frac1m\E(\mB\otimes\vb)
    +
    \frac1{m^2}
    \sum_{i=1}^m
    \left\{
        \mB_i\otimes\vb_i-\E(\mB\otimes\vb)
    \right\}.
\]
For the mean term, for any unit \(a,b,c\),
\[
\begin{aligned}
    \left|
        \E\left[
            (a^\top\mB b)(c^\top\vb)
        \right]
    \right|
    &=
    \left|
        \E\left[
            (a^\top\vx)(b^\top\vx)(c^\top\vb)
        \right]
    \right|
    \\
    &\le
    \left\{
        \E[(a^\top\vx)^2(b^\top\vx)^2]
    \right\}^{1/2}
    \left\{
        \E[(c^\top\vb)^2]
    \right\}^{1/2}
    \le
    (LM_{\vb})^{1/2},
\end{aligned}
\]
where the term involving \(\mI_r\) vanishes because \(\E\vb=0\), and the last
inequality bounds the two factors separately: \((b^\top\vx)^2\le\|\vx\|_2^2\le
L\) and \(\E[(a^\top\vx)^2]=a^\top\E(\vx\vx^\top)a=1\) give
\(\E[(a^\top\vx)^2(b^\top\vx)^2]\le L\), while \(\E[(c^\top\vb)^2]\le\E\|\vb\|_2^2=M_{\vb}\)
since \(c\) is a unit vector. Therefore
\[
    \left\|
        \frac1m\E(\mB\otimes\vb)
    \right\|_{\vee}^2
    \le
    \frac{L M_{\vb}}{m^2}.
\]
For the centered diagonal fluctuation, a direct second-moment bound is enough.
Use \(\|\cdot\|_{\vee}\le \|\cdot\|_F\). Since
\(\E\|\vx\|_2^2=r\) and \(\|\vx\|_2^2\le L\), we have \(r\le L\). Hence
\[
    \|\mB\|_F
    =
    \|\vx\vx^\top-\mI_r\|_F
    \le
    \|\vx\vx^\top\|_F+\|\mI_r\|_F
    \le
    L+\sqrt r
    \le
    C L .
\]
Therefore, using independence and centering in the Hilbert space equipped
with the Frobenius norm,
\[
\begin{aligned}
    \E
    \left\|
        \frac1{m^2}
        \sum_{i=1}^m
        \left\{
            \mB_i\otimes\vb_i-\E(\mB\otimes\vb)
        \right\}
    \right\|_{\vee}^2
    &\le
    \frac1{m^4}
    \E
    \left\|
        \sum_{i=1}^m
        \left\{
            \mB_i\otimes\vb_i-\E(\mB\otimes\vb)
        \right\}
    \right\|_F^2
    \\
    &=
    \frac1{m^3}
    \E
    \left\|
        \mB\otimes\vb-\E(\mB\otimes\vb)
    \right\|_F^2
    \\
    &\le
    \frac1{m^3}
    \E\|\mB\otimes\vb\|_F^2
    =
    \frac1{m^3}
    \E\left(\|\mB\|_F^2\|\vb\|_2^2\right)
    \\
    &\le
    \frac{C L^2 M_{\vb}}{m^3}
    \le
    \frac{
        C L^2\log^2(e r)M_{\vb}
    }{m^3}.
\end{aligned}
\]
Thus
\[
    \E\|\mD_m\|_{\vee}^2
    \le
    C
    \left\{
        \frac{L\log(e r)}{m^2}
        +
        \frac{L^2\log^2(e r)}{m^3}
    \right\}
    M_{\vb},
\]
where we used \(\log(e r)\ge1\).

\paragraph{Step 3: Combine the two pieces.}
Combining the off-diagonal and diagonal bounds proves
\eqref{eq:same-sample-product}.
\end{proof}

\subsection{Proofs of Other Theoretical Results}
\label{app:proof-other-theory}

\subsubsection{Proof of Proposition~\ref{prop:aug-hajek-regular}}

\begin{proof}
First verify that the proposed function belongs to the stated working class.
For \(S\in C_k\), if \(\gamma_{\va,q}\) is not identically zero on \(C_k\),
then the assumption makes \(\gamma_{\va,q}(S)\ne0\) throughout
that cell, and the definition of \(h_{\mathscr C}\) gives
\[
    \gamma_{\va,q}(S)h_{\mathscr C}(S)=\mu_k .
\]
If \(\gamma_{\va,q}\equiv0\) on \(C_k\), then
\(\mu_k=\E_q[\gamma_{\va,q}(S)u(S)\mid S\in C_k]=0\) when
\(\pi_k>0\), and \(\mu_k=0\) by convention when \(\pi_k=0\). In this case
\(\gamma_{\va,q}(S)h_{\mathscr C}(S)=0=\mu_k\). Hence
\[
    \gamma_{\va,q}(S)h_{\mathscr C}(S)
    =
    \sum_{k=1}^K \mu_k\1\{S\in C_k\},
\]
so \(h_{\mathscr C}\in\mathcal H_{\mathscr C,\gamma_{\va,q}}\).

The proof of Theorem~\ref{prop:non-asymptotic-hajek} gives the exact
decomposition
\[
   \hat\tau_{\va, m}^{\mathrm{Hajek}}(u)-\tau_{\va}(u)
    =
    \frac1m\sum_{t=1}^m \psi_{h_{\mathscr C}}(S_t;u)
    +
    r_m
\]
with the remainder bound in \eqref{eq:aug-hajek-finite-remainder}. In the
fixed finite-player setting, \(K_+\), \(\Gamma_{2,\mathscr C}\),
\(\Gamma_\infty\), \(\pi_{\min}\), and \(\|u\|_\infty\) do not vary with
\(m\), and \(\pi_{\min}>0\) by definition over the active cells. Therefore
\[
    \E_q(r_m^2)
    =
    O(m^{-2})+O\{e^{-c m}\}
    =
    o(m^{-1})
\]
for some \(c>0\). This is the regularity condition in
Definition~\ref{def:regular}, relative to
\(\mathcal H_{\mathscr C,\gamma_{\va,q}}\).
\end{proof}

\subsubsection{Proof of Proposition~\ref{prop:rr-regular}}

\begin{proof}
Apply Theorem~\ref{prop:rr-regular-exact} to the scalar target by taking
\(d=1\), \(\mA=\va\), and \(\mC_{\mA}=\vc_{\va}\). The proof of that theorem
constructs the population WLS projection
\[
    h^\star(S)=\vz(S)^\top\vbeta^\star,
\]
where \(\vbeta^\star\) is a minimum-Euclidean-norm population WLS solution, and
establishes the expansion
\[
    \hat{\tau}_{\va, m}^{\mathrm{WLS}, \lambda}(u)-\tau_{\va}(u)
    =
    \frac1m\sum_{t=1}^m
    \psi_{h^\star}(S_t;u)
    +
    r_m .
\]

The function \(h^\star\) is an element of
\(\mathcal H_{\mathrm{WLS}}\), and it is a population WLS projection of \(u\)
onto this feature span. If the population projection is not unique, let
\(h_w(S)=\vz(S)^\top\vbeta_w\) be any population WLS projection. Then
\(\vbeta_w-\vbeta^\star\) lies in the null space of the population Gram matrix,
so
\[
    \sum_{S\subseteq[n]}w(S)
    \{\vz(S)^\top(\vbeta_w-\vbeta^\star)\}^2
    =
    0.
\]
Thus \(h_w(S)=h^\star(S)\) for every \(S\) with \(w(S)>0\). The expression for the weight \(\gamma_{\va,q}\)
derived in the proof of Theorem~\ref{prop:rr-regular-exact},
\[
    \gamma_{\va,q}(S)
    =
    \tilde w(S)\vz(S)^\top\mG^+\vc_{\va},
\]
also implies that \(\gamma_{\va,q}(S)=0\) whenever \(w(S)=0\). Hence
\(\gamma_{\va,q}(S)\{h_w(S)-h^\star(S)\}=0\) for all \(S\) with \(q(S)>0\).
The WLS readout identity gives \(\tau_{\va}(h_w)=\tau_{\va}(h^\star)\).
Consequently,
\[
    \psi_{h_w}(S;u)=\psi_{h^\star}(S;u)
    \qquad q\text{-almost surely}.
\]
The same expansion may therefore be written with \(h_w\).

In the fixed finite-player setting, with the ridge penalty \(\lambda\asymp\sigma_{\min}^+(\mG)\) as in Theorem~\ref{prop:rr-regular-exact}, the displayed bound \eqref{eq:rr-delta-bound} gives
\[
    \E_q(r_m^2)
    =
    O(m^{-2})+O\{m^2 e^{-c m}\}
    =
    o(m^{-1})
\]
for some \(c>0\). Therefore
\(\hat{\tau}_{\va,m}^{\mathrm{WLS},\lambda}\) is regular relative to
\(\mathcal H_{\mathrm{WLS}}\).
\end{proof}

\subsubsection{Proof of Lemma~\ref{lem:regular-mse-vector}}

\begin{proof}
For \(t=1,\dots,m\), define the \(d\)-vector
\[
    \bm{\Psi}_t
    :=
    \bigl(
        \psi_{\va_1,q,h_1}(S_t;u),
        \dots,
        \psi_{\va_d,q,h_d}(S_t;u)
    \bigr)^\top,
\]
and set
\[
    \bar{\bm{\Psi}}_m:=\frac1m\sum_{t=1}^m\bm{\Psi}_t .
\]
Each coordinate has mean zero because
\[
    \E_q\{\psi_{\va_j,q,h_j}(S;u)\}
    =
    \E_q[\gamma_{\va_j,q}(S)\{u(S)-h_j(S)\}]
    +
    \tau_{\va_j}(h_j)
    -
    \tau_{\va_j}(u)
    =
    0.
\]
The coordinatewise expansion \eqref{eq:coordinate-regular-expansion} gives
\[
    \hat{\vtau}_{\mA,m}(u)-\vtau_{\mA}(u)
    =
    \bar{\bm{\Psi}}_m+\vr_m .
\]
Since the samples are iid and each coordinate of \(\bm{\Psi}_t\) is centered,
\[
\begin{aligned}
    \E_q\|\bar{\bm{\Psi}}_m\|_2^2
    &=
    \sum_{j=1}^d
    \E_q\left[
        \left(
            \frac1m\sum_{t=1}^m
            \psi_{\va_j,q,h_j}(S_t;u)
        \right)^2
    \right] \\
    &=
    \frac1m
    \sum_{j=1}^d
    \Var_q[
        \gamma_{\va_j,q}(S)\{u(S)-h_j(S)\}
    ]
    =
    \frac{V(\mA;q,\vh)}{m}.
\end{aligned}
\]
The assumption \(V^*(\mA;q)>0\) ensures that
\(V(\mA;q,\vh)>0\), so the ratio below is well defined.

Using the decomposition
\(\|\bar{\bm{\Psi}}_m+\vr_m\|_2^2-\|\bar{\bm{\Psi}}_m\|_2^2
=2\bar{\bm{\Psi}}_m^\top\vr_m+\|\vr_m\|_2^2\) and Cauchy's inequality,
\[
\begin{aligned}
    \left|
        \E_q\|\bar{\bm{\Psi}}_m+\vr_m\|_2^2
        -
        \E_q\|\bar{\bm{\Psi}}_m\|_2^2
    \right|
    &\le
    2
    \left\{\E_q\|\bar{\bm{\Psi}}_m\|_2^2\right\}^{1/2}
    \left\{\E_q\|\vr_m\|_2^2\right\}^{1/2}
    +
    \E_q\|\vr_m\|_2^2 .
\end{aligned}
\]
Dividing both sides by \(V(\mA;q,\vh)/m\) yields
\eqref{eq:mse-ratio-bound}.

Finally, if
\[
    \E_q\|\vr_m\|_2^2
    \le
    C\epsilon^2\frac{V(\mA;q,\vh)}{m}
\]
with a sufficiently small universal constant \(C\), then the right-hand side of
\eqref{eq:mse-ratio-bound} is at most \(\epsilon\) for every
\(0<\epsilon\le1\). This proves the stated consequence.
\end{proof}

\section{Details for Non-Asymptotic Results} \label{app:non-asymptotic-details}

\subsection{Details for H\'ajek-Type Weighted-Average Estimators} \label{app:non-asymptotic-hajek}

We spell out the specialization of Theorem~\ref{prop:non-asymptotic-hajek} to the Shapley value of a fixed player \(i\). For the Shapley value, the weighted-average coefficient in \eqref{eq:tau-a-linear} is
\begin{equation}
    \rho_i(S)
    =
    \begin{cases}
        \displaystyle
        \frac{1}{|S|\binom{n}{|S|}},
        & i\in S,\ 1\le |S|\le n,\\[2ex]
        \displaystyle
        -\frac{1}{(n-|S|)\binom{n}{|S|}},
        & i\notin S,\ 0\le |S|\le n-1.
    \end{cases}
    \label{eq:app-shapley-rho}
\end{equation}
The endpoint coalitions \(\varnothing\) and \([n]\) are deterministic in the usual Shapley-value implementations, and their contribution \(\{u([n])-u(\varnothing)\}/n\) can be added without sampling error. Hence the rate calculations below concern the random part supported on the interior strata \(1\le |S|\le n-1\).

Both OFA and Stratified SVARM use self-normalization over cells indexed by membership of player \(i\) and coalition size. For \(s=1,\dots,n-1\), define
\[
    C_{i,s}^{+}
    :=
    \{S\subseteq[n]: i\in S,\ |S|=s\},
    \qquad
    C_{i,s}^{-}
    :=
    \{S\subseteq[n]: i\notin S,\ |S|=s\}.
\]
Cells that are empty, such as \(C_{i,n}^{-}\) and \(C_{i,0}^{+}\), are ignored. If the sampling law first draws a size \(s\) with probability \(q_s\) and then draws uniformly among coalitions of that size, then
\[
    q(S)=\frac{q_{|S|}}{\binom{n}{|S|}},
    \qquad
    \pi_{i,s}^{+}:=\PP_q(S\in C_{i,s}^{+})=\frac{s}{n}q_s,
    \qquad
    \pi_{i,s}^{-}:=\PP_q(S\in C_{i,s}^{-})=\frac{n-s}{n}q_s.
\]
For estimating \(\evphi_i\), the weight \(\gamma_{i,q}(S):=\rho_i(S)/q(S)\) is constant on each membership-size cell:
\begin{equation}
    \gamma_{i,q}(S)
    =
    \begin{cases}
        \displaystyle
        \frac{1}{s q_s},
        & S\in C_{i,s}^{+},\\[2ex]
        \displaystyle
        -\frac{1}{(n-s)q_s},
        & S\in C_{i,s}^{-}.
    \end{cases}
    \label{eq:app-hajek-shapley-gamma}
\end{equation}
Thus the condition in Proposition~\ref{prop:aug-hajek-regular} holds for the membership-size partition.

\begin{proof}[Proof of Corollary~\ref{corollary:hajek-ofa-svarm-rates}]
Let
\[
    K_+=|\{C\in\mathscr C:\PP_q(S\in C)>0\}|.
\]
For the membership-size partition above, \(K_+=2(n-1)=O(n)\). We bound the two design quantities in Theorem~\ref{prop:non-asymptotic-hajek},
\[
    \Gamma_{2,\mathscr C}^2
    =
    \frac1{K_+}\sum_{C\in\mathscr C:\PP_q(S\in C)>0}
    \E_q[\gamma_{i,q}(S)^2\mid S\in C],
    \qquad
    \Gamma_\infty
    =
    \sup_{S:q(S)>0}|\gamma_{i,q}(S)|.
\]
Because \(\gamma_{i,q}\) is constant on each cell, the conditional expectations are just the squared cell values in \eqref{eq:app-hajek-shapley-gamma}.

\paragraph{OFA.}
For Shapley-value estimation, the OFA size distribution is
\[
    q_s^{\mathrm{OFA}}
    =
    \frac{\{s(n-s)\}^{-1/2}}
    {\sum_{r=1}^{n-1}\{r(n-r)\}^{-1/2}},
    \qquad s=1,\dots,n-1.
\]
The normalizing denominator is \(\Theta(1)\). Substituting this \(q_s\) into \eqref{eq:app-hajek-shapley-gamma} gives
\[
    |\gamma_{i,q}(S)|^2
    \asymp
    \begin{cases}
        (n-s)/s, & S\in C_{i,s}^{+},\\
        s/(n-s), & S\in C_{i,s}^{-}.
    \end{cases}
\]
Therefore
\[
\begin{aligned}
    \Gamma_{2,\mathscr C}^2
    &\lesssim
    \frac1n
    \sum_{s=1}^{n-1}
    \left\{
        \frac{n-s}{s}
        +
        \frac{s}{n-s}
    \right\}
    =
    O(\log n),
    \\
    \Gamma_\infty^2
    &=
    O(n).
\end{aligned}
\]
Moreover,
\[
    \pi_{\min}
    =
    \min_{1\le s\le n-1}
    \left\{
        \frac{s}{n}q_s^{\mathrm{OFA}},
        \frac{n-s}{n}q_s^{\mathrm{OFA}}
    \right\}
    =
    \Theta(n^{-3/2}).
\]
Applying Theorem~\ref{prop:non-asymptotic-hajek} yields
\[
    \E_q(r_m^2)
    \le
    C_1\|u\|_\infty^2
    \left\{
        \frac{n\log n}{m^2}
        +
        n e^{-C_2m/n^{3/2}}
    \right\},
\]
after absorbing constants into \(C_1,C_2\).

\paragraph{Stratified SVARM.}
Stratified SVARM uses the harmonic size allocation
\[
    q_s^{\mathrm{SVARM}}
    =
    \frac{\{\min(s,n-s)\}^{-1}}
    {\sum_{r=1}^{n-1}\{\min(r,n-r)\}^{-1}},
    \qquad s=1,\dots,n-1.
\]
The denominator is \(\Theta(\log n)\), so
\[
    q_s^{\mathrm{SVARM}}
    \asymp
    \frac{1}{\min(s,n-s)\log n}.
\]
Using \eqref{eq:app-hajek-shapley-gamma},
\[
    |\gamma_{i,q}(S)|
    \lesssim
    \log n
\]
on both \(C_{i,s}^{+}\) and \(C_{i,s}^{-}\). Hence
\[
    \Gamma_{2,\mathscr C}^2=O(\log^2 n),
    \qquad
    \Gamma_\infty^2=O(\log^2 n).
\]
The smallest active cell probability satisfies
\[
    \pi_{\min}
    =
    \min_{1\le s\le n-1}
    \left\{
        \frac{s}{n}q_s^{\mathrm{SVARM}},
        \frac{n-s}{n}q_s^{\mathrm{SVARM}}
    \right\}
    =
    \Theta((n\log n)^{-1}).
\]
Theorem~\ref{prop:non-asymptotic-hajek} therefore gives
\[
    \E_q(r_m^2)
    \le
    C_1\|u\|_\infty^2
    \left\{
        \frac{n\log^2 n}{m^2}
        +
        \log^2 n\, e^{-C_2m/(n\log n)}
    \right\}.
\]
This proves the two displayed bounds in Corollary~\ref{corollary:hajek-ofa-svarm-rates}.
\end{proof}

\subsection{Details for Ridge-Penalized WLS Estimators} \label{app:non-asymptotic-rr}

We next instantiate Theorem~\ref{prop:rr-regular-exact} for the WLS geometry used by KernelSHAP and LeverageSHAP. Let
\[
    \mA:=\mI-\frac1n\vone\vone^\top
\]
be the projection onto the sum-zero subspace, as in Corollary~\ref{corollary:ls-kernel-leverage-rates}. For \(1\le |S|\le n-1\), define the centered coalition feature
\[
    \vz(S)
    :=
    \bigl(\1\{1\in S\},\dots,\1\{n\in S\}\bigr)^\top
    -
    \frac{|S|}{n}\vone
    \in\R^n
\]
and the Shapley kernel weight
\[
    w(S)
    :=
    \frac{n-1}{\binom{n}{|S|}|S|(n-|S|)}.
\]
The population WLS problem with this feature map recovers the centered Shapley vector \(\mA\vphi(u)\); the deterministic efficiency component \(\{u([n])-u(\varnothing)\}\vone/n\) can then be added separately if the full Shapley vector is needed.

For this design,
\begin{equation}
    \mG
    =
    \sum_{1\le |S|\le n-1}
    w(S)\vz(S)\vz(S)^\top
    =
    \frac{n-1}{n}\mA,
    \qquad
    \mG^+
    =
    \frac{n}{n-1}\mA.
    \label{eq:app-kernelshap-gram}
\end{equation}
Indeed, conditional on \(|S|=s\), the centered inclusion vector has covariance
\[
    \frac{s(n-s)}{n(n-1)}\mA,
\]
and summing the weighted contribution over \(s=1,\dots,n-1\) gives \eqref{eq:app-kernelshap-gram}. Consequently, \(\kappa_{\mG}=1\). Also,
\begin{equation}
    \|\vz(S)\|_2^2
    =
    \frac{|S|(n-|S|)}{n},
    \qquad
    \vz(S)^\top\mG^+\vz(S)
    =
    \frac{|S|(n-|S|)}{n-1}.
    \label{eq:app-feature-leverage-core}
\end{equation}
For this target, the readout matrix is \(\mC_{\mA}=\mA\), so \(R_{\mA}^2=\sigma_{\max}(\mA^\top\mG^+\mA)=O(1)\), and the feature dimension is \(d_\vz=n\).

\begin{proof}[Proof of Corollary~\ref{corollary:ls-kernel-leverage-rates}]
We evaluate the quantities in Theorem~\ref{prop:rr-regular-exact} under the two sampling laws.

\paragraph{KernelSHAP.}
KernelSHAP samples coalitions proportionally to the Shapley kernel weight:
\[
    q_{\mathrm{KS}}(S)
    =
    \frac{w(S)}{W_{\mathrm{KS}}},
    \qquad
    W_{\mathrm{KS}}
    :=
    \sum_{1\le |T|\le n-1}w(T)
    =
    \frac{2(n-1)}{n}H_{n-1}
    =
    O(\log n),
\]
where \(H_{n-1}=\sum_{s=1}^{n-1}s^{-1}\). Equivalently,
\[
    q_{\mathrm{KS}}(S)
    =
    \frac{n}{2H_{n-1}}
    \frac{1}{\binom{n}{|S|}|S|(n-|S|)}.
\]
Thus \(\tilde w(S)=w(S)/q_{\mathrm{KS}}(S)=W_{\mathrm{KS}}\). Combining this with \eqref{eq:app-feature-leverage-core},
\[
    \ell(S)
    =
    \tilde w(S)\vz(S)^\top\mG^+\vz(S)
    =
    W_{\mathrm{KS}}
    \frac{|S|(n-|S|)}{n-1},
\]
so
\[
    L_{\mathrm{KS}}
    =
    \sup_S\ell(S)
    =
    O(n\log n).
\]
The weighted norms are bounded by
\[
    \|u\|_{2,\tilde w}^2
    \le
    W_{\mathrm{KS}}\|u\|_\infty^2
    =
    O(\log n)\|u\|_\infty^2,
    \qquad
    \|u\|_{\infty,\tilde w}^2
    \le
    O(\log n)\|u\|_\infty^2.
\]
Substituting \(R_{\mA}^2=O(1)\), \(d_\vz=n\), \(\kappa_{\mG}=1\), \(L_{\mathrm{KS}}=O(n\log n)\), and the preceding norm bounds into Theorem~\ref{prop:rr-regular-exact} gives
\[
\begin{aligned}
    \E_q(\|\vr_m\|_2^2)
    &\le
    C_1\|u\|_\infty^2
    \bigg[
        \frac{n^2\log^4 n}{m^2}
        \left(1+\frac{n\log^2 n}{m}\right)
        \\
        &\hspace{6em}
        +
        n^3m^2\log^3 n\,
        e^{-C_2m/(n\log n)}
    \bigg].
\end{aligned}
\]

\paragraph{LeverageSHAP.}
LeverageSHAP samples proportionally to the WLS leverage contribution
\[
    w(S)\vz(S)^\top\mG^+\vz(S).
\]
By \eqref{eq:app-feature-leverage-core},
\[
    w(S)\vz(S)^\top\mG^+\vz(S)
    =
    \frac{1}{\binom{n}{|S|}}.
\]
The total leverage mass is \(n-1=\operatorname{rank}(\mG)\), so
\[
    q_{\mathrm{Lev}}(S)
    =
    \frac{1}{(n-1)\binom{n}{|S|}},
    \qquad
    1\le |S|\le n-1.
\]
Equivalently, LeverageSHAP draws the size uniformly from \(\{1,\dots,n-1\}\) and then draws a coalition uniformly within that size. For this law,
\[
    \tilde w(S)
    =
    \frac{w(S)}{q_{\mathrm{Lev}}(S)}
    =
    \frac{(n-1)^2}{|S|(n-|S|)}.
\]
Using \eqref{eq:app-feature-leverage-core},
\[
    \ell(S)
    =
    \frac{(n-1)^2}{|S|(n-|S|)}
    \frac{|S|(n-|S|)}{n-1}
    =
    n-1,
\]
and hence \(L_{\mathrm{Lev}}=n-1=O(n)\). The weighted norms satisfy
\[
    \|u\|_{2,\tilde w}^2
    =
    \sum_{1\le |S|\le n-1} w(S)u(S)^2
    \le
    W_{\mathrm{KS}}\|u\|_\infty^2
    =
    O(\log n)\|u\|_\infty^2,
\]
and
\[
    \|u\|_{\infty,\tilde w}^2
    \le
    O(n)\|u\|_\infty^2,
\]
because \(|S|(n-|S|)\ge n-1\) on the interior strata. Substituting
\[
    R_{\mA}^2=O(1),\quad
    d_\vz=n,\quad
    \kappa_{\mG}=1,\quad
    L_{\mathrm{Lev}}=O(n)
\]
into Theorem~\ref{prop:rr-regular-exact} gives
\[
    \E_q(\|\vr_m\|_2^2)
    \le
    C_1\|u\|_\infty^2
    \left[
        \frac{n^2\log^2 n}{m^2}
        \left(1+\frac{n\log n}{m}\right)
        +
        n^4m^2 e^{-C_2m/n}
    \right].
\]
This proves Corollary~\ref{corollary:ls-kernel-leverage-rates}.
\end{proof}

\providecommand{\revised}[1]{{\color{darkpurple}#1}} % marks text revised in this round

\section{Details of Algorithms}\label{app:algorithm}

This appendix gives implementation details for the EASE construction in
\Secref{sec:ease}.  We keep the notation of the main paper: the target vector is
\[
    \vtau_{\mA}(u)
    =
    \bigl(
        \tau_{\va_1}(u),\dots,\tau_{\va_d}(u)
    \bigr)^\top,
    \qquad
    \mA=[\va_1,\dots,\va_d].
\]
Define the vector of weighted-average coefficients
\begin{equation}
    \bm{\rho}_{\mA}(S)
    :=
    \bigl(
        \rho_{\va_1}(S),\dots,\rho_{\va_d}(S)
    \bigr)^\top ,
    \qquad S\subseteq[n].
    \label{eq:algorithm-rho-vector}
\end{equation}
Then
\(
    \vtau_{\mA}(u)=\sum_{S\subseteq[n]}\bm{\rho}_{\mA}(S)u(S)
\).
\Secref{sec:ease} leaves the candidate class \(\mathcal Q\) abstract; \Secref{app:candidate-class} gives the partition-based class our implementation uses. \Secref{app:sampling-details} then gives the reference law \(q_0\) used to initialize Algorithm~\ref{alg:vector-aipw-profiled}, and \Secref{app:flooring-learned-q} gives the floored law that our implementation uses as the Stage-2 sampling law \(\hat q\); all three apply to the general vector-surrogate setting \(\vh=(h_1,\dots,h_d)\). \Secref{app:shared-linear-online} specializes further to the single, shared linear surrogate \(h_{\vbeta}(S)=\vx(S)^\top\vbeta\) that EASE-FO and EASE-SP both use in our numerical experiments, differing only in the feature map \(\vx(S)\). \Secref{app:pilot-pooling} describes how the pilot sample is folded into the final cross-fitted estimate, and \Secref{app:complement-pair-algorithm} specializes complement-pair sampling to sign-symmetric targets.

\subsection{Choice of the Candidate Class \(\mathcal Q\)}
\label{app:candidate-class}

Our implementation instantiates \(\mathcal Q\) as a partition-based class: fix a partition \(\mathscr C=\{C_1,\dots,C_K\}\) of the coalition space \(2^{[n]}\), and consider laws that are constant within each cell,
\begin{equation}
    \mathcal Q_{\mathscr C}
    =
    \Bigl\{
        q_{\bm\nu}
        :
        q_{\bm\nu}(S)=\nu_k \text{ for } S\in C_k,\;
        \nu_k\ge0,\;
        \textstyle\sum_{k=1}^K|C_k|\nu_k=1
    \Bigr\}.
    \label{eq:algorithm-partition-class}
\end{equation}
Restricting to \(\mathcal Q_{\mathscr C}\) reduces the search over distributions on \(2^{[n]}\) to a search over the \(K\) within-cell sampling probabilities \((\nu_1,\dots,\nu_K)\), which is what makes the pilot optimization in \eqref{eq:pilot-design-optimization} tractable.

Throughout our numerical experiments, \(\mathscr C\) is the coalition-size partition, \(C_k=\{S\subseteq[n]:|S|=k\}\) for \(k=0,\dots,n\): the sampling probability of a coalition depends on \(S\) only through \(|S|\), so coalitions of the same size share a sampling probability. The next two subsections fix this class and derive the reference law and the flooring rule within it.

\subsection{Initialization for a Fixed Sampling Partition}
\label{app:sampling-details}

Fix the partition \(\mathscr C=\{C_1,\dots,C_K\}\) from \Secref{app:candidate-class}. For a law that is constant within each cell, write \(q(S)=\nu_k\) for \(S\in C_k\), with \(\sum_{k=1}^K |C_k|\nu_k=1\). For a general vector surrogate \(\vh=(h_1,\dots,h_d)\), the cellwise residual scale is
\begin{equation}
    M_k(\vh)
    :=
    \frac{1}{|C_k|}
    \sum_{S\in C_k}
    \sum_{j=1}^d
    \rho_{\va_j}(S)^2
    \{u(S)-h_j(S)\}^2 .
    \label{eq:algorithm-general-cell-moment}
\end{equation}
For fixed \(\vh\), the uncentered version of the first-order MSE criterion \(V(\mA;q,\vh)\)
within this partition is
\[
    \sum_{k=1}^K
    \frac{|C_k|M_k(\vh)}{\nu_k}.
\]
Minimizing this expression over \((\nu_1,\dots,\nu_K)\) gives
\begin{equation}
    q^\star(S;\vh)=\nu_k^\star(\vh),
    \qquad
    \nu_k^\star(\vh)
    =
    \frac{\sqrt{M_k(\vh)}}
    {\sum_{r=1}^K |C_r|\sqrt{M_r(\vh)}} ,
    \qquad S\in C_k .
    \label{eq:algorithm-oracle-partition-q}
\end{equation}
This is a Neyman-type allocation for stratified sampling, with the cells acting as strata~\citep{Neyman1934-wp}.

Before any informative surrogate has been learned, the residual terms in
\eqref{eq:algorithm-general-cell-moment} are unknown.  A simple initialization is
obtained by treating the residuals \(|u(S)-h_j(S)|\) as uniformly bounded by a
common constant, or more generally as having the same order across cells and
coordinates.  Under either interpretation, the unknown common residual scale
drops out of \eqref{eq:algorithm-oracle-partition-q}.  We therefore initialize
from
\begin{equation}
    M_{k,0}
    :=
    \frac{1}{|C_k|}
    \sum_{S\in C_k}
    \sum_{j=1}^d
    \rho_{\va_j}(S)^2,
    \qquad
    q^{\mathrm{init}}(S)
    =
    \nu_k^{\mathrm{init}}
    :=
    \frac{\sqrt{M_{k,0}}}
    {\sum_{r=1}^K |C_r|\sqrt{M_{r,0}}},
    \quad S\in C_k .
    \label{eq:algorithm-init-law}
\end{equation}
Thus the pilot allocation is determined by the target weights and the chosen partition alone. Our implementation takes \(q_0=q^{\mathrm{init}}\) as the reference law in Algorithm~\ref{alg:vector-aipw-profiled}; the same law initializes the pilot optimization and serves as the flooring anchor in \Secref{app:flooring-learned-q}.

\subsection{Flooring the Learned Sampling Law}
\label{app:flooring-learned-q}

Let \(\tilde q\) denote the sampling-law component of a solution to the pilot optimization problem in \eqref{eq:pilot-design-optimization}. In finite samples, \(\tilde q\) can be too aggressive: pilot noise may lead to a very small probability for a target-relevant cell, which in turn inflates the correction weights \(\rho_{\va_j}(S)/q(S)\). For the partition-based class \(\mathcal Q_{\mathscr C}\), we therefore floor \(\tilde q\) toward the initialization law \(q^{\mathrm{init}}\) from \eqref{eq:algorithm-init-law}:
\begin{equation}
    q^{\mathrm{floor}}(S)
    =
    (1-\epsilon)\tilde q(S)
    +
    \epsilon\, q^{\mathrm{init}}(S),
    \qquad S\in C_k,
    \quad \epsilon\in(0,1).
    \label{eq:algorithm-floored-q}
\end{equation}
Every coalition \(S\in C_k\) therefore receives probability at least \(\epsilon\, q^{\mathrm{init}}(S)\) under \(q^{\mathrm{floor}}\). The floor prevents pilot error from producing nearly infinite inverse-probability weights while preserving the learned allocation when \(\epsilon\) is small. Flooring toward \(q^{\mathrm{init}}\), rather than toward a cell-uniform law, keeps the floor itself target-informed: it redistributes mass back toward the same \(\sqrt{M_{k,0}}\)-weighted allocation used to initialize the pilot, rather than toward cells the target weights \(\bm\rho_{\mA}\) may rule out entirely. Our implementation uses \(\epsilon=10^{-8}\) throughout and takes \(\hat q=q^{\mathrm{floor}}\) as the Stage-2 sampling law in Algorithm~\ref{alg:vector-aipw-profiled}.

\subsection{Shared Linear Surrogate Learned On the Fly}
\label{app:shared-linear-online}

We now specialize to a practical regime in which a single surrogate is shared
across all coordinates and is linear in a chosen feature vector:
\[
    h_{\vbeta}(S)=\vx(S)^\top\vbeta,
    \qquad
    \vx(S)\in\mathbb R^p .
\]

Assume throughout this subsection that the sampling law \(q\) is fixed and
positive on the support of \(\bm{\rho}_{\mA}\).  For observations
\(S_t\sim q\), set
\[
    \bm{\omega}_t
    :=
    \frac{\bm{\rho}_{\mA}(S_t)}{q(S_t)},
    \qquad
    a_t:=\|\bm{\omega}_t\|_2^2,
    \qquad
    \vx_t:=\vx(S_t).
\]
Given the first \(t\) observations, the profiled shared-linear objective is
\begin{equation}
    L_t(\vbeta,\vmu)
    :=
    \sum_{s=1}^t
    \left\|
        \bm{\omega}_s
        \{Y_s-\vx_s^\top\vbeta\}
        -
        \vmu
    \right\|_2^2
    +
    \lambda\|\vbeta\|_2^2,
    \qquad \lambda>0 .
    \label{eq:algorithm-shared-linear-loss}
\end{equation}
This is the empirical profiled criterion from
\eqref{eq:empirical-profiled-training} specialized to a shared linear surrogate
under a fixed sampling law.

The minimizer can be updated from four additive sufficient statistics:
\begin{align}
    \mR_t
    &:=
    \sum_{s=1}^t
    a_s\vx_s\vx_s^\top,
    &
    \vd_t
    &:=
    \sum_{s=1}^t
    a_s\vx_sY_s,
    \label{eq:algorithm-online-stats-1}
    \\
    \mU_t
    &:=
    \sum_{s=1}^t
    \bm{\omega}_s\vx_s^\top,
    &
    \vv_t
    &:=
    \sum_{s=1}^t
    \bm{\omega}_sY_s .
    \label{eq:algorithm-online-stats-2}
\end{align}
Profiling out \(\vmu\) gives
\begin{equation}
    \hat\vbeta_t
    =
    \left(
        \mR_t
        -
        \frac{1}{t}\mU_t^\top\mU_t
        +
        \lambda\mI_p
    \right)^{-1}
    \left(
        \vd_t
        -
        \frac{1}{t}\mU_t^\top\vv_t
    \right),
    \label{eq:algorithm-online-beta}
\end{equation}
and
\[
    \hat\vmu_t
    =
    \frac{1}{t}
    \bigl(
        \vv_t-\mU_t\hat\vbeta_t
    \bigr).
\]
Each statistic in
\eqref{eq:algorithm-online-stats-1}--\eqref{eq:algorithm-online-stats-2}
updates by addition when a new utility value is observed.  Hence the shared
linear surrogate can be updated on the fly without resolving the full least
squares problem from scratch.

This is the surrogate class used by both EASE variants in our numerical experiments: EASE-FO takes \(\vx(S)\) to be an intercept, the player indicators \(\mathbb I(i\in S)\), and the coalition-size terms \(\log(1+|S|)\) and \((|S|/n)^2\); EASE-SP takes \(\vx(S)\) to be the player-by-coalition-size indicators \(\mathbb I(|S|=s,i\in S)\). The two differ only in this feature map, not in how the surrogate is fit.

\subsection{Incorporating the Pilot Sample into the Final Estimate}
\label{app:pilot-pooling}

Algorithm~\ref{alg:vector-aipw-profiled} as stated draws an independent estimation sample of size \(m\) from \(\hat q\) for Stage~2 and cross-fits only on those \(m\) observations. Our implementation instead pools the \(m_0\) pilot observations \(\{(S_t^{(0)},Y_t^{(0)})\}_{t=1}^{m_0}\), drawn from \(q_0\), together with the \(m\) Stage-2 observations \(\{(S_t,Y_t)\}_{t=1}^{m}\), drawn from \(\hat q\), into a single pooled sample of size \(m_0+m\), and uses all of it in the final cross-fitted estimate. Write \(q_t=q_0\) for the pilot indices and \(q_t=\hat q\) for the Stage-2 indices, so the pooled sample is \(\{(S_t,Y_t,q_t)\}_{t=1}^{m_0+m}\). The fold partition \(\{I_b\}_{b=1}^B\) of \(\{1,\dots,m_0+m\}\) is drawn over these pooled indices, so pilot observations, like Stage-2 observations, serve as training data for some folds and as held-out data for others. On each training fold \(\mathcal T_b=\{1,\dots,m_0+m\}\setminus I_b\), EASE solves
\begin{equation}
    (\hat\vh^{(-b)},\hat\vmu^{(-b)})
    \in
    \argmin_{\vh\in\mathcal H,\,\vmu\in\R^d}
    \sum_{t\in\mathcal T_b}
    \sum_{j=1}^d
    \left[
        \frac{\rho_{\va_j}(S_t)}{q_t(S_t)}
        \{Y_t-h_j(S_t)\}
        -
        \evmu_j
    \right]^2,
    \label{eq:algorithm-pooled-training}
\end{equation}
and, for each \(j=1,\dots,d\), sets
\begin{equation}
    \hat\tau_{\va_j}^{(b)}
    =
    \tau_{\va_j}(\hat h_j^{(-b)})
    +
    \frac{1}{|I_b|}
    \sum_{t\in I_b}
    \frac{\rho_{\va_j}(S_t)}{q_t(S_t)}
    \{Y_t-\hat h_j^{(-b)}(S_t)\}.
    \label{eq:algorithm-pooled-fold-estimate}
\end{equation}
The pooled cross-fitted estimator is
\begin{equation}
    \hat\vtau_{\mA}
    =
    \sum_{b=1}^B
    \frac{|I_b|}{m_0+m}
    \bigl(
        \hat\tau_{\va_1}^{(b)},\dots,\hat\tau_{\va_d}^{(b)}
    \bigr)^\top,
    \label{eq:algorithm-pooled-estimator}
\end{equation}
the same combination rule as in Algorithm~\ref{alg:vector-aipw-profiled}, but with \(m\) replaced by \(m_0+m\) and \(\hat q(S_t)\) replaced by \(q_t(S_t)\). Because each observation retains the law under which it was actually drawn, \(q_t\), rather than being reweighted by a single \(\hat q\), every term in \eqref{eq:algorithm-pooled-training}--\eqref{eq:algorithm-pooled-estimator} is unbiased for its own coordinate of \(\vtau_{\mA}(u)\) regardless of which stage generated it. Pooling the two samples is therefore free of bias; it only affects variance, and typically reduces it by using \(m_0+m\) rather than \(m\) observations in the final cross-fit.

\subsection{Complement-Pair Sampling}
\label{app:complement-pair-algorithm}

Complement-pair sampling observes \(u(S)\) and \(u(S^c)\) together.  In this subsection, we only
consider the common sign-symmetric case in which, for every target coordinate
and every coalition,
\begin{equation}
    \rho_{\va_j}(S^c)=-\rho_{\va_j}(S),
    \qquad j=1,\dots,d.
    \label{eq:algorithm-complement-antisymmetry}
\end{equation}
With the sign convention in \eqref{eq:tau-a-linear}, this is the relation
satisfied by the Shapley values, Beta Shapley values with symmetric weights, and
Banzhaf values. In practice, complement pair sampling is a natural choice for sign-symmetric targets~\citep{Witter2025-gz}. 

Under \eqref{eq:algorithm-complement-antisymmetry}, the target can be written as
a sum over complement contrasts:
\[
    \vtau_{\mA}(u)
    =
    \frac{1}{2}
    \sum_{S\subseteq[n]}
    \bm{\rho}_{\mA}(S)
    \{u(S)-u(S^c)\}.
\]
Suppose an ordered coalition \(S\) is drawn from a law \(q\), but the
observation reveals the pair \((S,S^c)\).  The unordered pair mass is
\[
    q_{\mathrm{pair}}(S)
    :=
    q(S)+q(S^c),
\]
which is the same for \(S\) and \(S^c\).  For a general vector surrogate
\(\vh=(h_1,\dots,h_d)\), define the coordinatewise residual contrast
\[
    \Delta_{\vh,j}(S)
    :=
    \{u(S)-h_j(S)\}
    -
    \{u(S^c)-h_j(S^c)\},
    \qquad j=1,\dots,d,
\]
and let
\(
    \bm{\Delta}_{\vh}(S)
    =
    (\Delta_{\vh,1}(S),\dots,\Delta_{\vh,d}(S))^\top
\).
Then the paired AIPW term
\begin{equation}
    \widehat{\vtau}^{\mathrm{pair}}(S;\vh)
    :=
    \vtau_{\mA}(\vh)
    +
    \frac{
        \bm{\rho}_{\mA}(S)\odot\bm{\Delta}_{\vh}(S)
    }{
        q_{\mathrm{pair}}(S)
    }
    \label{eq:algorithm-complement-score}
\end{equation}
is unbiased when \(S\sim q\), provided \(q_{\mathrm{pair}}(S)>0\) whenever
\(\bm{\rho}_{\mA}(S)\neq\vzero\).  Indeed, the two ordered representatives
\(S\) and \(S^c\) of the same pair contribute the same contrast because
\(\bm{\rho}_{\mA}(S^c)=-\bm{\rho}_{\mA}(S)\) and
\(\bm{\Delta}_{\vh}(S^c)=-\bm{\Delta}_{\vh}(S)\).

The same residual-aware sampling idea can be applied at the pair level.  For a
partition \(\mathscr P=\{P_1,\dots,P_L\}\) of complement pairs, define
\[
    M_\ell^{\mathrm{pair}}(\vh)
    :=
    \frac{1}{|P_\ell|}
    \sum_{\{S,S^c\}\in P_\ell}
    \left\|
        \bm{\rho}_{\mA}(S)\odot\bm{\Delta}_{\vh}(S)
    \right\|_2^2 .
\]
The corresponding plug-in allocation assigns more pair mass to cells with large
estimated residual contrasts.  As in the single-coalition case, a floor should
be added to the learned pair law to keep \(q_{\mathrm{pair}}(S)\) bounded away
from zero on active pairs.

\section{Additional Experimental Details and Results}
\label{app:additional-experiments}

\subsection{SOU Games: Closed-Form Values and Construction of \(\mathscr S_{\mathrm{high}}\)}
\label{app:sou-closed-form}

\paragraph{Closed-form probabilistic values.} For a sum-of-unanimity game \(u(S)=\sum_{T\in\mathscr S}\omega_T\mathbb I(T\subseteq S)\), the probabilistic value of a unanimity game \(u_T\) is well known to depend on \(T\) only through \(|T|\) and whether \(i\in T\): \(\phi_i(u_T)=c_{|T|}\) if \(i\in T\) and \(0\) otherwise, for a constant \(c_s\) that depends on the target probabilistic value. By linearity, \(\phi_i(u)=\sum_{T\in\mathscr S:\,i\in T}\omega_T c_{|T|}\), which is what we use as ground truth. The coefficients \(c_s\) for the targets considered in this paper are
\[
    c_s=
    \begin{cases}
        1/s, & \text{Shapley value},\\[2pt]
        p^{\,s-1}, & \text{weighted Banzhaf value with parameter } p,\\[2pt]
        B(\alpha,\beta+s-1)/B(\alpha,\beta), & \text{Beta Shapley value with parameters } (\alpha,\beta),
    \end{cases}
\]
where \(B(\cdot,\cdot)\) is the Beta function.

\paragraph{Generating \(\mathscr S_{\mathrm{high}}\).} Each of the \(n^2\) coalitions in \(\mathscr S_{\mathrm{high}}\) is generated independently: its cardinality is drawn uniformly from \(\{3,\dots,\min(40,n-1)\}\), and, given the cardinality, its members are drawn uniformly without replacement from \([n]\). Duplicate coalitions across draws are permitted; because each draw also receives an independent weight, a repeated coalition simply contributes the sum of its independent weights. The coefficient variance \(\sigma^2\) defaults to \(|\mathscr S_{\mathrm{low}}|+|\mathscr S_{\mathrm{high}}|\), which is the value used throughout.

\subsection{Additional AUCC Results for the SOU Benchmark}
\label{app:additional-aucc-sou}

The main text reports the SOU AUCC benchmark (Figure~\ref{fig:sou-aucc-benchmark}) for \(\eta=0.25\) and \(n=40\). This subsection extends the comparison along two axes: larger \(\eta\) at \(n=40\), and larger \(n\) across all three \(\eta\) values.

\paragraph{Larger \(\eta\).}
Figure~\ref{fig:sou-aucc-benchmark-additional} reports \(\eta\in\{0.5,0.75\}\) at \(n=40\). The qualitative pattern matches the \(\eta=0.25\) benchmark in Figure~\ref{fig:sou-aucc-benchmark}: EASE-FO and EASE-SP remain among the most sample-efficient estimators, with EASE-SP typically strongest for the Shapley and Beta-Shapley targets. Increasing \(\eta\) from \(0.5\) to \(0.75\) changes the scale of the errors but not the ranking.

\begin{figure}[!htbp]
    \centering
    \includegraphics[width=\linewidth]{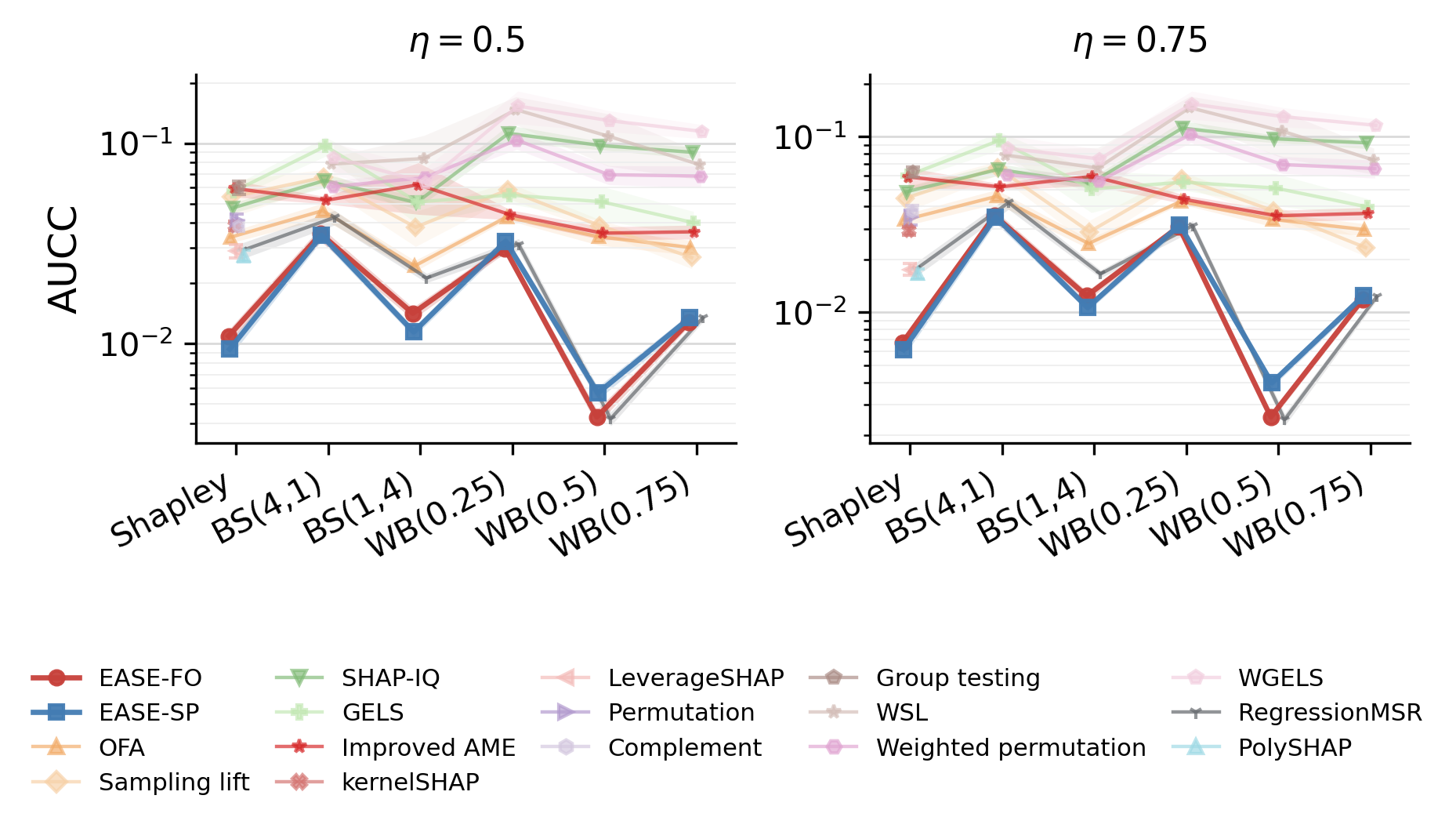}
    \caption{AUCC benchmark on the SOU games for \(\eta = 0.5\) and \(\eta = 0.75\) at \(n=40\). Each x-axis label is a target probabilistic value. Lower AUCC indicates better performance.}
    \label{fig:sou-aucc-benchmark-additional}
\end{figure}

\paragraph{Larger \(n\).}
Figures~\ref{fig:sou-aucc-benchmark-n80} and \ref{fig:sou-aucc-benchmark-n160} report the same benchmark at \(n=80\) and \(n=160\), across all three \(\eta\) values, using the same total budget of \(160{,}000\) utility evaluations as \(n=40\). Because this budget is not rescaled with \(n\), the average number of evaluations per player falls from \(4000\) at \(n=40\) to \(2000\) at \(n=80\) and \(1000\) at \(n=160\). The EASE variants remain at or near the lowest AUCC across all \((n,\eta)\) combinations, and for the Shapley and both Beta-Shapley targets, EASE-FO and EASE-SP are visibly ahead of every baseline. A second pattern emerges as \(n\) grows: EASE-SP, nearly indistinguishable from EASE-FO at \(n=40\), falls increasingly behind it at \(n=80\) and \(n=160\). This is consistent with a bias-variance trade-off in the surrogate class: EASE-SP's player-by-coalition-size surrogate has \(O(n^2)\) parameters, against \(O(n)\) for EASE-FO, so as the per-player budget shrinks with growing \(n\), the richer surrogate becomes correspondingly harder to fit well.

\begin{figure}[!htbp]
    \centering
    \includegraphics[width=\linewidth]{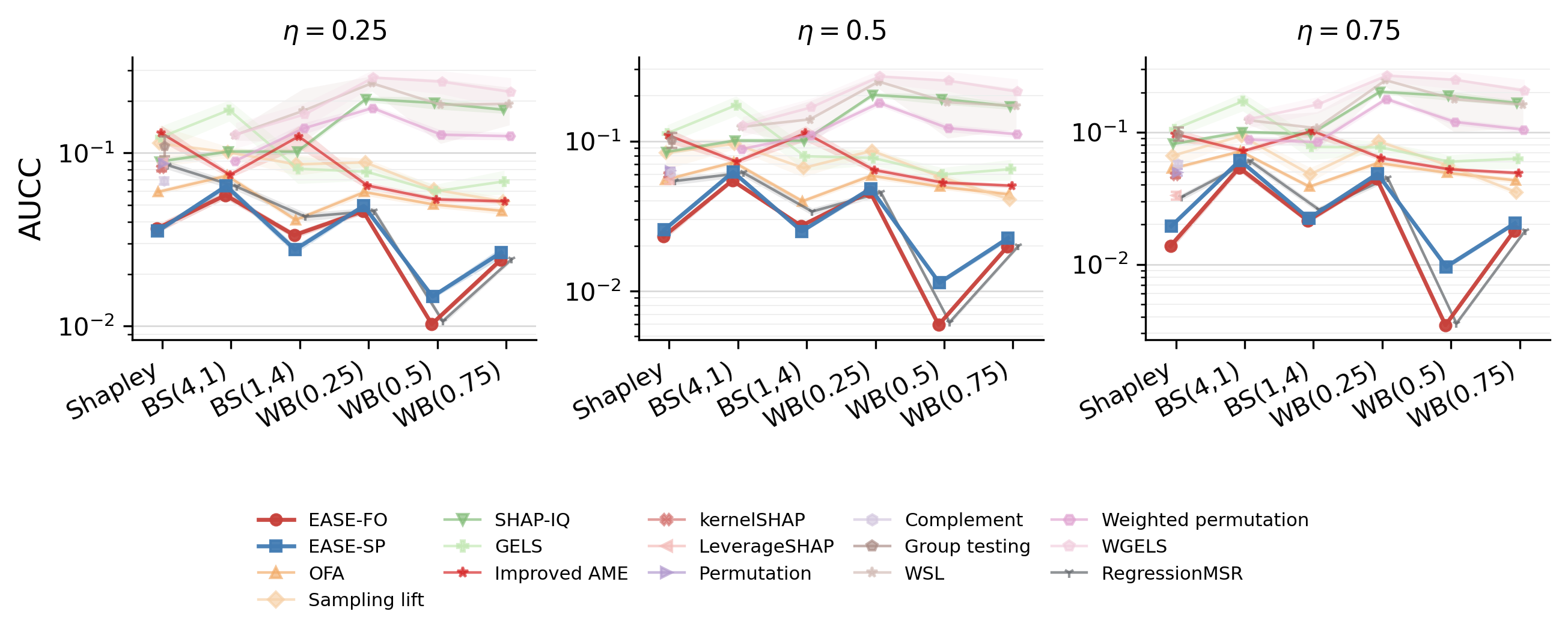}
    \caption{AUCC benchmark on the SOU games for \(n=80\), across \(\eta\in\{0.25,0.5,0.75\}\). Each x-axis label is a target probabilistic value. Lower AUCC indicates better performance.}
    \label{fig:sou-aucc-benchmark-n80}
\end{figure}

\begin{figure}[!htbp]
    \centering
    \includegraphics[width=\linewidth]{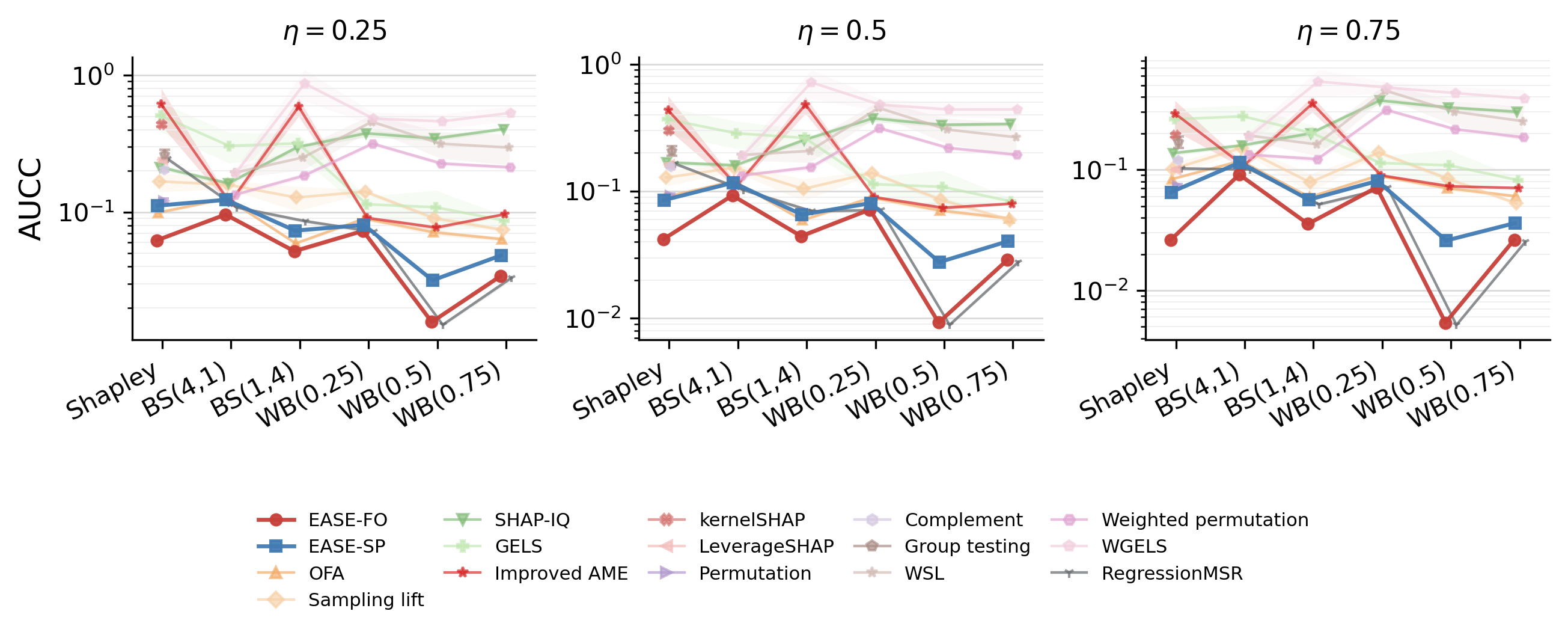}
    \caption{AUCC benchmark on the SOU games for \(n=160\), across \(\eta\in\{0.25,0.5,0.75\}\). Each x-axis label is a target probabilistic value. Lower AUCC indicates better performance.}
    \label{fig:sou-aucc-benchmark-n160}
\end{figure}

\subsection{Sensitivity to the Pilot Fraction}
\label{app:pilot-fraction-sensitivity}

The main text allocates \(20\%\) of the utility-evaluation budget to the pilot stage. Figure~\ref{fig:pilot-fraction-sensitivity} sweeps the pilot fraction over \(\{0.05,0.10,0.20,0.40\}\) for EASE-FO on the four real-world feature-attribution benchmarks (Shapley value), holding the total budget and all other settings fixed at the main-text protocol. Final performance is not sensitive to the pilot fraction once it is at least \(10\%\): across \(\{0.10,0.20,0.40\}\), the relative \(L_2\) error at the largest budget varies by no more than about \(9\%\) on any dataset. Only the smallest fraction, \(5\%\), is consistently worse, trailing the best of the other three by \(10\)--\(17\%\). The \(20\%\) default used throughout the main text therefore sits well within this insensitive range.

\begin{figure}[!htbp]
    \centering
    \includegraphics[width=\linewidth]{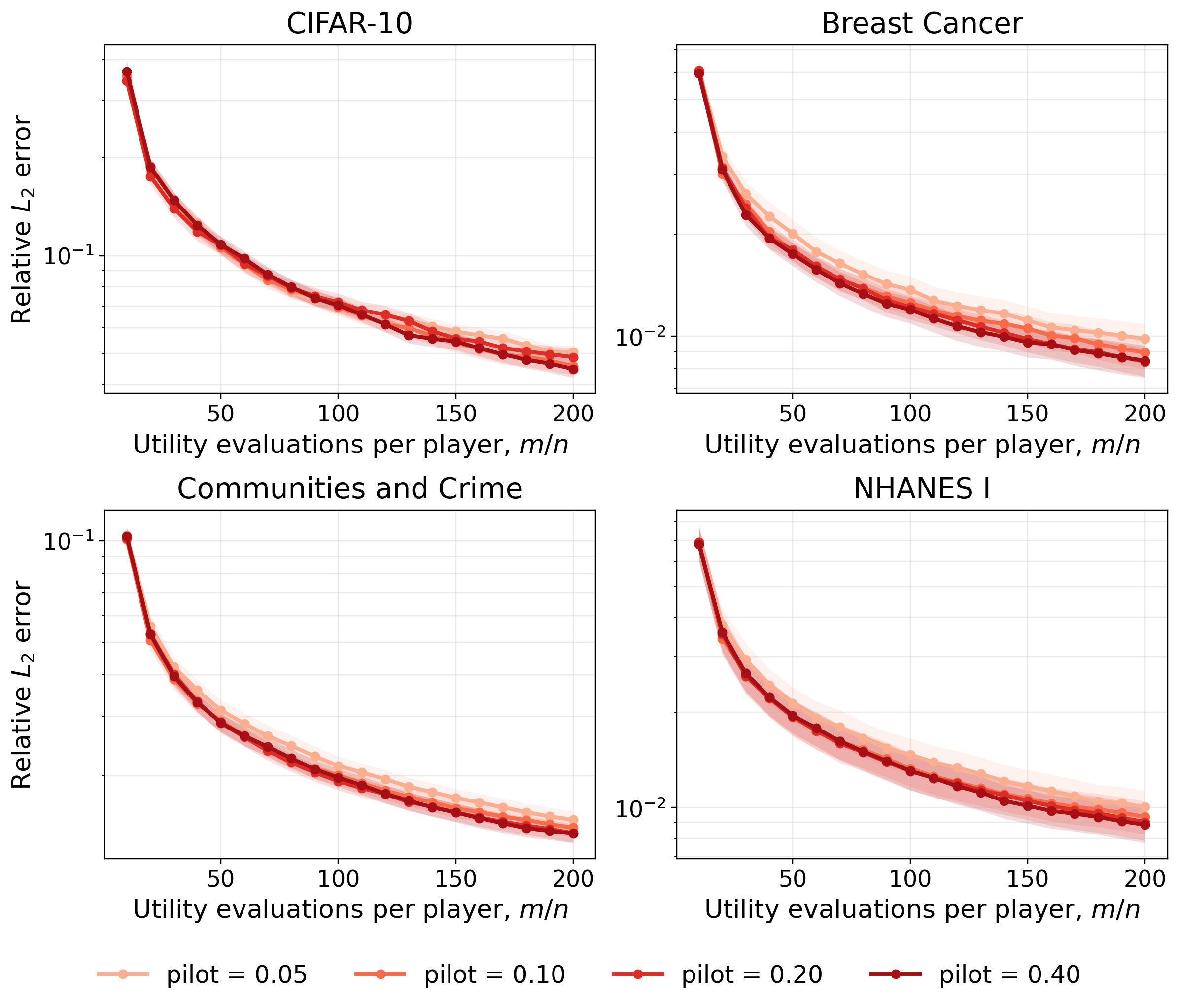}
    \caption{Sensitivity of EASE-FO to the pilot fraction on the four real-world feature-attribution benchmarks (Shapley value), at the main-text total budget of \(m=200n\). Shaded bands denote \(\pm 1\) standard error over explained instances.}
    \label{fig:pilot-fraction-sensitivity}
\end{figure}

\subsection{Sampling-Allocation Diagnostics}
\label{app:allocation-diagnostics}

Our real-world results show a larger relative gain for EASE-FO on data attribution (ACSIncome) than on feature attribution. This subsection reports a diagnostic that traces this pattern to how far EASE-FO's learned sampling law actually moves away from non-adaptive baseline laws on each task.

The diagnostic compares three coalition-size sampling laws: \(\hat q\), the sampling law learned by the pilot stage of EASE-FO under the main-text protocol; \(q^{\mathrm{init}}\), the target-informed initialization law of \eqref{eq:algorithm-init-law} from which the pilot updates start; and \(q^{\mathrm{lev}}\), the size-uniform law used by LeverageSHAP (Appendix~\ref{app:non-asymptotic-rr}). Under the coalition-size partition of Appendix~\ref{app:candidate-class}, each law is determined by the probability mass it assigns to the coalition sizes; we therefore view the laws as probability distributions over the sampled coalition sizes and quantify their disagreement by the total-variation (TV) distance.

Table~\ref{tab:allocation-tv} reports \(\operatorname{TV}(\hat q,q^{\mathrm{lev}})\) and \(\operatorname{TV}(\hat q,q^{\mathrm{init}})\) for the ACSIncome data-attribution game with ridge logistic regression and for the three tabular feature-attribution benchmarks. For ACSIncome, we repeat the diagnostic over ten independent pilot draws and report the mean with its Monte Carlo standard error. For each feature-attribution dataset, we run ten independent pilot draws for each of ten explained instances, average within each instance, and report the mean with the standard error across the ten instance means.

\begin{table}[!htbp]
    \centering
    \small
    \renewcommand{\arraystretch}{1.12}
    \begin{tabularx}{\linewidth}{@{}Y c c@{}}
        \toprule
        Task & \(\operatorname{TV}(\hat q,\,q^{\mathrm{lev}})\) & \(\operatorname{TV}(\hat q,\,q^{\mathrm{init}})\) \\
        \midrule
        ACSIncome (data attribution) & 0.457 (0.005) & 0.347 (0.005) \\
        Breast Cancer (feature attribution) & 0.062 (0.004) & 0.077 (0.006) \\
        NHANES I (feature attribution) & 0.067 (0.005) & 0.137 (0.010) \\
        Communities and Crime (feature attribution) & 0.061 (0.001) & 0.131 (0.007) \\
        \bottomrule
    \end{tabularx}
    \caption{Total-variation distance from the sampling allocation \(\hat q\) learned by EASE-FO to the size-uniform law \(q^{\mathrm{lev}}\) used by LeverageSHAP and to the initialization \(q^{\mathrm{init}}\), with all laws viewed as distributions over the sampled coalition sizes. Entries are means with standard errors in parentheses, computed over ten repeated pilot draws for ACSIncome and across ten explained instances for the feature-attribution tasks.}
    \label{tab:allocation-tv}
\end{table}

On ACSIncome, the learned allocation moves far from both baselines: its TV distance is \(0.457\) from the size-uniform law and \(0.347\) from its own initialization. The pilot detects pronounced size-dependence in the residuals of the first-order surrogate, so the learned design departs substantially from any fixed baseline allocation---room that the adaptive design can exploit, consistent with the larger gains EASE-FO shows on this task. On the three feature-attribution benchmarks, by contrast, \(\hat q\) ends close to size-uniform, with TV distances of at most \(0.067\). Notably, on these tasks \(\operatorname{TV}(\hat q,q^{\mathrm{init}})\) exceeds \(\operatorname{TV}(\hat q,q^{\mathrm{lev}})\): the three updates move the allocation away from its target-informed initialization and toward the near-uniform law. The near-uniform endpoint therefore reflects what the pilot residuals support---residual variation that is nearly homogeneous across coalition sizes---rather than inertia in the update procedure. With little size structure to exploit, residual-aware sampling has correspondingly less to offer, consistent with the smaller gains of EASE-FO over RegressionMSR on these benchmarks.

\section{Monte Carlo Estimation with Bundled Samples}\label{app:bundled}
\renewcommand{\theHdefinition}{appbundled.definition.\arabic{definition}}
\renewcommand{\theHexample}{appbundled.example.\arabic{example}}
\renewcommand{\theHproposition}{appbundled.proposition.\arabic{proposition}}

\subsection{Bundled Sampling and Target Representation}

The main text studies single-coalition sampling, where Monte Carlo draw $t=1,\dots,m$ returns a pair $(S_t,u(S_t))$ with $S_t\subseteq[n]$. This appendix considers the more general case in which one draw reveals several coupled coalition values. Specifically, draw $t$ returns $(\vT_t,u(\vT_t))$, where
\[
    \vT_t=(S_{t,1},\dots,S_{t,r})
\]
is a random bundle of $r$ coalitions and
\[
    u(\vT_t)=(u(S_{t,1}),\dots,u(S_{t,r}))
\]
is the corresponding vector of utility values. The bundle is sampled from a known law $Q$ on $(2^{[n]})^r$. We refer to this observation model as \emph{bundled sampling}.

Throughout this appendix, \(m\) counts bundle draws, not utility evaluations. Each draw reveals \(r\) utility values (\(r=2\) for complement or edge pairs, and \(r=n+1\) for permutation paths), so \(m\) bundle draws consume \(rm\) utility evaluations, whereas \(m\) single-coalition draws in the main text consume only \(m\). Within this appendix the distinction is immaterial: every comparison holds the bundle law \(Q\), and hence \(r\), fixed, so per-draw and per-evaluation rankings coincide. The distinction matters only when designs with different bundle sizes are compared, for example bundled against single-coalition sampling; such comparisons should be made at a matched utility-evaluation budget.

Bundled sampling covers several common Monte Carlo designs. Complement-pair sampling returns a random coalition $S$ together with its complement $S^c$~\citep{Zhang2023-ne}. Edge-pair sampling returns a random coalition $S$ and its one-player extension $S\cup\{i\}$ for some player $i$~\citep{Moehle2021-cm,Kwon2021-bw}. Permutation-path sampling returns the full sequence of coalitions generated along a random permutation of the players~\citep{Castro2009-yu,Castro2017-pf,Maleki2013-zq,Wu2023-hb}.

These designs differ in the bundle they observe, but the target remains the
same linear functional
\(\tau_{\va}(u)=\sum_{S\subseteq[n]}\rho_{\va}(S)u(S)\). Under a fixed bundle
law \(Q\), the basic question is how to assign weights to the observed bundle
entries so that their expectation recovers this target. We formalize this
requirement below.

\begin{definition}[Feasible bundle weighting]\label{def:feasible-bundle-weighting}
Fix the target $\tau_{\va}(u)=\sum_{S\subseteq[n]}\rho_{\va}(S)u(S)$ and a bundle law $Q$ on $(2^{[n]})^r$. An $r$-dimensional function $\vgamma_{\va,Q}:(2^{[n]})^r\to\mathbb{R}^r$ is a \emph{feasible bundle weighting} under $Q$ if, for every utility function $u:2^{[n]}\to\mathbb{R}$,
\[
    \tau_{\va}(u)
    =
    \E_Q\!\left[\vgamma_{\va,Q}(\vT)^\top u(\vT)\right].
\]
We denote the class of all feasible bundle weightings by $\Gamma_{\va}(Q)$.
\end{definition}

Single-coalition sampling is the special case in which the feasible weighting is unique. When $r=1$, the definition reduces to the inverse-probability weighting $\vgamma_{\va,Q}(S)=\rho_{\va}(S)/Q(S)$ for every $S\subseteq[n]$ with $Q(S)>0$. When $r\geq2$, feasibility instead imposes a system of linear constraints on a vector-valued weighting, and that system can have multiple solutions.

To illustrate the nonuniqueness, consider the complement-pair design as an example, where $\vT=(S,S^c)$ with $S\sim q$. Write the weighting as $\vgamma_{\va,Q}(S,S^c)=(\gamma_1(S,S^c),\gamma_2(S,S^c))$. Matching the coefficient of each $u(S)$ in Definition~\ref{def:feasible-bundle-weighting} gives the linear constraints
\[
    q(S)\gamma_1(S,S^c)+q(S^c)\gamma_2(S^c,S)=\rho_{\va}(S),
    \qquad S\subseteq[n].
\]
Thus, whenever both $q(S)$ and $q(S^c)$ are positive, $\gamma_1(S,S^c)$ can be chosen freely and $\gamma_2(S^c,S)$ adjusted to equal $\{\rho_{\va}(S)-q(S)\gamma_1(S,S^c)\}/q(S^c)$. Equivalently, if $q$ has full support, any allocation function $\zeta:2^{[n]}\to\mathbb R$ induces a feasible weighting through
\[
    \gamma_1(S,S^c)=\frac{\zeta(S)\rho_{\va}(S)}{q(S)},
    \qquad
    \gamma_2(S,S^c)=\frac{\{1-\zeta(S^c)\}\rho_{\va}(S^c)}{q(S)}.
\]
Different choices of $\zeta$ distribute the same target coefficient across the two possible appearances of a coalition in the ordered complement pair. For instance, $\zeta\equiv1$ assigns each coefficient to the first component when the coalition is sampled directly, whereas $\zeta\equiv0$ assigns it to the second component when its complement is sampled. Hence the same target functional $\tau_{\va}(u)$ can be represented by many feasible bundle weightings.

\subsection{Bundle-Regular Estimators}

The same regularity idea used in the main text extends once the single
coalition weighting is replaced by a feasible bundle weighting. We first record
the coefficient form of Definition~\ref{def:feasible-bundle-weighting}. For
\(\vgamma_{\va,Q}\in\Gamma_{\va}(Q)\), feasibility is equivalent to
\begin{equation}
    \rho_{\va}(S)
    =
    \E_Q\!\left[
        \sum_{\ell=1}^r
        \evgamma_{\va,Q,\ell}(\vT)\1\{S_\ell=S\}
    \right],
    \qquad S\subseteq[n].
    \label{eq:bundle-id}
\end{equation}
This form is useful because it separates the sampling law \(Q\) from the
choice of weighting used for the entries of a bundle.

For a feasible weighting \(\vgamma\in\Gamma_{\va}(Q)\) and a surrogate
\(h:2^{[n]}\to\mathbb R\), define the bundle influence term
\begin{equation}
    \psi_{\vgamma,h}(\vT;u)
    :=
    \vgamma(\vT)^\top\{u(\vT)-h(\vT)\}
    +
    \tau_{\va}(h)-\tau_{\va}(u).
    \label{eq:bundle-psi}
\end{equation}
By feasibility, \(\E_Q\{\psi_{\vgamma,h}(\vT;u)\}=0\). This gives the
following analogue of Definition~\ref{def:regular}.

\begin{definition}[Bundle-regular estimator sequence]\label{def:bundle-regular}
Fix a bundle law \(Q\), a feasible weighting \(\vgamma\in\Gamma_{\va}(Q)\), and a
working surrogate class \(\mathcal H\) on \(2^{[n]}\). A sequence of estimators
\(\{\hat\tau_{\va,m}\}_{m\ge1}\) is \emph{bundle-regular relative to
    \((\vgamma,\mathcal H)\) under \(Q\)} if there exist nonrandom functions
    \(h_m\in\mathcal H\) such that, for i.i.d. bundles
    \(\vT_1,\dots,\vT_m\sim Q\),
    \begin{equation}
	    \hat\tau_{\va,m}(u)-\tau_{\va}(u)
	    =
	    \frac1m\sum_{t=1}^m
	    \psi_{\vgamma,h_m}(\vT_t;u)
	    +
	    r_m,
    \qquad
    \E_Q(r_m^2)=o(m^{-1}).
    \label{eq:bundle-regular-expansion}
\end{equation}
\end{definition}

Definition~\ref{def:bundle-regular} keeps the same separation as in the
single-coalition case: the weighting describes the target representation,
whereas \(h_m\) describes the first-order surrogate adjustment. 

This class includes several bundled estimators in the literature. Edge-pair
methods such as sampling lift~\citep{Moehle2021-cm} and weighted sampling
lift~\citep{Kwon2021-bw} are HT-type weighted averages with a design-specific
\(\vgamma\) and the working class \(\mathcal H=\{0\}\). Permutation-based
estimators and their stratified refinements~\citep{Castro2009-yu,Castro2017-pf,Maleki2013-zq,Wu2023-hb}
can be viewed as bundled H{\'a}jek-type weighted averages, where the
normalization over permutation strata induces the corresponding working class.
Paired or complement-paired variants of WLS estimators, including
LeverageSHAP~\citep{Musco2024-cf} and PolySHAP~\citep{Fumagalli2026-hu}, fit
the same definition, with \(\mathcal H\) given by the WLS feature span and the feasible weighting induced jointly by the bundled sampling law and the WLS design.

\begin{example}[Complement-pair sampling]\label{ex:bundle-complement}
Complementary-contribution sampling~\citep{Zhang2023-ne} evaluates coalition
values in complement pairs. One draw returns \(\vT=(S,S^c)\) with \(S\sim q\),
so the same utility query batch contains both \(u(S)\) and \(u(S^c)\). This
coupling is useful for semivalues whose weights are symmetric with respect to coalition size, for which the coefficients of a
coalition and its complement have opposite signs, and the estimator can be
written in terms of complementary differences. In our notation, a canonical
feasible weighting for this design is
\begin{equation}
    \evgamma^{\mathrm{comp}}_1(S,S^c)
    :=
    \frac{\rho_{\va}(S)}{q(S)+q(S^c)},
    \qquad
    \evgamma^{\mathrm{comp}}_2(S,S^c)
    :=
    \frac{\rho_{\va}(S^c)}{q(S)+q(S^c)},
    \label{eq:comp-canonical-weighting}
\end{equation}
on pairs for which \(q(S)+q(S^c)>0\). Indeed, for every \(S\subseteq[n]\),
\[
    q(S)\evgamma^{\mathrm{comp}}_1(S,S^c)
    +
    q(S^c)\evgamma^{\mathrm{comp}}_2(S^c,S)
    =
    \rho_{\va}(S).
\]
Thus the contribution of a coalition can be assigned across its two possible
appearances in an ordered complement pair. The construction in
\citet{Zhang2023-ne} uses this paired observation structure, together with
allocation across coalition sizes, to reduce the variance of direct
Monte Carlo estimation.
\end{example}

\begin{example}[Edge-pair sampling]\label{ex:bundle-edge}
Sampling lift~\citep{Moehle2021-cm} and weighted sampling
lift~\citep{Kwon2021-bw} evaluate marginal-contribution pairs. Fix a player
\(i\), let \(q_i\) be a law on \(2^{[n]\setminus\{i\}}\), and sample
\(\vT=(S,S\cup\{i\})\) with \(S\sim q_i\). The draw directly reveals the
difference \(u(S\cup\{i\})-u(S)\). For the coordinate target
\(\tau_{\ve_i}(u)=\evphi_i(u)\), the marginal-contribution representation gives
\[
    \evphi_i(u)
    =
    \sum_{S\subseteq[n]\setminus\{i\}}
    \alpha_i^{(n)}(S)
    \{u(S\cup\{i\})-u(S)\}.
\]
Thus a feasible edge weighting is
\begin{equation}
    \evgamma^{\mathrm{edge},i}_1(S,S\cup\{i\})
    :=
    -\frac{\alpha_i^{(n)}(S)}{q_i(S)},
    \qquad
    \evgamma^{\mathrm{edge},i}_2(S,S\cup\{i\})
    :=
    \frac{\alpha_i^{(n)}(S)}{q_i(S)}.
    \label{eq:edge-weighting}
\end{equation}
The original sampling-lift construction uses this edge-pair observation to
estimate Shapley-type targets from sampled marginal contributions, while
weighted sampling lift modifies the sampling and weighting scheme to cover
Beta Shapley and related probabilistic values.
\end{example}

More general edge-pair designs can sample both the base coalition and the
added player. If the base coalition--player pair is sampled from a law \(q_{\mathcal E}\) on
\(\mathcal E_n=\{(S,i):S\subseteq[n],\,i\notin S\}\), then feasibility requires
\begin{equation}
    \rho_{\va}(S)
    =
    \sum_{i\notin S}
        q_{\mathcal E}(S,i)\evgamma_1(S,i)
    +
    \sum_{i\in S}
        q_{\mathcal E}(S\setminus\{i\},i)
        \evgamma_2(S\setminus\{i\},i),
    \qquad S\subseteq[n].
    \label{eq:edge-id}
\end{equation}
Unlike the fixed-player representation in Example~\ref{ex:bundle-edge}, this
system couples the weighting choices across many directed edges.

\subsection{First-Order Risk with a Fixed Bundle Weighting}

For a bundle-regular estimator, the leading MSE is governed by the variance of
the bundle influence term. Define
\begin{equation}
    V_Q(\vgamma,h)
    :=
    \Var_Q\!\left\{
        \psi_{\vgamma,h}(\vT;u)
    \right\}
    =
    \Var_Q\!\left[
        \vgamma(\vT)^\top\{u(\vT)-h(\vT)\}
    \right].
    \label{eq:bundle-risk}
\end{equation}
Because the bundle influence terms \(\psi_{\vgamma,h_m}(\vT_t;u)\) are i.i.d.\ and mean zero, the non-asymptotic ratio bound of Lemma~\ref{lem:regular-mse-vector} applies verbatim with \(V(\mA;q,\vh)/m\) replaced by \(V_Q(\vgamma,h_m)/m\).
In particular, if \eqref{eq:bundle-regular-expansion} holds and
\(\E_Q(r_m^2)=o\{V_Q(\vgamma,h_m)/m\}\), then
\begin{equation}
    \E_Q\!\left[
        \{\hat\tau_{\va,m}(u)-\tau_{\va}(u)\}^2
    \right]
    =
    \frac{V_Q(\vgamma,h_m)}{m}\{1+o(1)\}.
    \label{eq:bundle-first-order-mse}
\end{equation}
Equivalently,
\begin{equation}
    V_Q(\vgamma,h)
    =
    \inf_{t\in\mathbb R}
    \E_Q\!\left[
        \left\{
            \vgamma(\vT)^\top\{u(\vT)-h(\vT)\}
            -
            t
        \right\}^2
    \right].
    \label{eq:bundle-profiled-risk}
\end{equation}
Thus, unlike the single-coalition case, the first-order criterion depends on
three objects: the bundle law \(Q\), the feasible weighting \(\vgamma\), and the
surrogate \(h\).
The rate in \eqref{eq:bundle-first-order-mse} is stated per bundle draw. On the utility-evaluation scale, a budget of \(B=rm\) evaluations yields leading MSE \(r\,V_Q(\vgamma,h_m)/B\). Hence, at matched evaluation budgets, the quantity comparable to the single-coalition criterion \(\Var_q[\gamma_{\va,q}(S)\{u(S)-h(S)\}]\) of Section~\ref{subsec:first-order-risk} is \(r\,V_Q(\vgamma,h)\) rather than \(V_Q(\vgamma,h)\) itself.
For a fixed weighting, the oracle surrogate problem is
\[
    \inf_{h\in\mathcal H} V_Q(\vgamma,h).
\]
At the fully joint level one may instead consider
\[
    \inf_{\vgamma\in\Gamma_{\va}(Q),\ h\in\mathcal H}
    V_Q(\vgamma,h),
\]
but the weighting optimization is generally a constrained problem over functions
on the whole bundle space. Equation~\eqref{eq:edge-id} illustrates why this
optimization can be difficult: even for edge pairs, the feasibility constraints
are globally coupled.

\subsection{Rao--Blackwell Weighting for Complement Pairs}

Complement-pair sampling is an important exception in which the weighting
optimization has an explicit solution. Let \(\varepsilon_h(S):=u(S)-h(S)\). For the
design \(\vT=(S,S^c)\), \(S\sim q\), any feasible weighting
\((\evgamma_1,\evgamma_2)\) must satisfy
\begin{equation}
    q(S)\evgamma_1(S,S^c)
    +
    q(S^c)\evgamma_2(S^c,S)
    =
    \rho_{\va}(S),
    \qquad S\subseteq[n].
    \label{eq:comp-id}
\end{equation}
The canonical weighting in \eqref{eq:comp-canonical-weighting} is the
Rao--Blackwellized version of any such feasible ordered-pair weighting.

\begin{proposition}[Rao--Blackwell optimality for complement pairs]\label{prop:comp-rb}
Assume the complement-pair design \(\vT=(S,S^c)\), \(S\sim q\), and suppose
\(q(S)+q(S^c)>0\) whenever the pair \(\{S,S^c\}\) can contribute to the target.
Let \((\evgamma_1,\evgamma_2)\) be any feasible weighting satisfying
\eqref{eq:comp-id}. Then, for every fixed surrogate \(h\),
\begin{align*}
&\Var_q\!\left[
    \evgamma^{\mathrm{comp}}_1(S,S^c)\varepsilon_h(S)
    +
    \evgamma^{\mathrm{comp}}_2(S,S^c)\varepsilon_h(S^c)
\right]
\\
&\qquad\le
\Var_q\!\left[
    \evgamma_1(S,S^c)\varepsilon_h(S)
    +
    \evgamma_2(S,S^c)\varepsilon_h(S^c)
\right].
\end{align*}
\end{proposition}

\begin{proof}
Let \(S_0\sim q\) denote the sampled base coalition, let
\(\mathcal U(S_0)=\{S_0,S_0^c\}\) denote the unordered complement pair, and
define
\[
    X_\gamma(S_0)
    :=
    \evgamma_1(S_0,S_0^c)\varepsilon_h(S_0)
    +
    \evgamma_2(S_0,S_0^c)\varepsilon_h(S_0^c).
\]
Conditioning on \(\mathcal U(S_0)=\{S,S^c\}\), the two possible ordered draws
are \((S,S^c)\) and \((S^c,S)\), with probabilities proportional to \(q(S)\)
and \(q(S^c)\). Therefore,
\begin{align*}
    \E_q\{X_\gamma(S_0)\mid \mathcal U(S_0)=\{S,S^c\}\}
    &=
    \frac{
        \{q(S)\evgamma_1(S,S^c)+q(S^c)\evgamma_2(S^c,S)\}\varepsilon_h(S)
    }{q(S)+q(S^c)}
    \\
    &\quad+
    \frac{
        \{q(S)\evgamma_2(S,S^c)+q(S^c)\evgamma_1(S^c,S)\}\varepsilon_h(S^c)
    }{q(S)+q(S^c)}
    \\
    &=
    \frac{\rho_{\va}(S)\varepsilon_h(S)+\rho_{\va}(S^c)\varepsilon_h(S^c)}
    {q(S)+q(S^c)}
    \\
    &=
    \evgamma^{\mathrm{comp}}_1(S,S^c)\varepsilon_h(S)
    +
    \evgamma^{\mathrm{comp}}_2(S,S^c)\varepsilon_h(S^c),
\end{align*}
where the second equality uses \eqref{eq:comp-id} for \(S\) and \(S^c\). The
last display is measurable with respect to the unordered pair
\(\mathcal U(S_0)\). The variance inequality follows from the law of total
variance:
\[
    \Var_q\{X_\gamma(S_0)\}
    =
    \Var_q\!\left[\E_q\{X_\gamma(S_0)\mid\mathcal U(S_0)\}\right]
    +
    \E_q\!\left[\Var_q\{X_\gamma(S_0)\mid\mathcal U(S_0)\}\right].
\]
\end{proof}

Proposition~\ref{prop:comp-rb} shows that complement pairs differ from general
bundled designs: one can remove the artificial randomness coming from the
orientation of the ordered pair without changing the target. For semivalues whose weights are symmetric with respect to coalition size,
including the Shapley value, \(\rho_i(S^c)=-\rho_i(S)\) for a
single coordinate target; this is the sign-symmetric case treated algorithmically in Appendix~\ref{app:complement-pair-algorithm}. Hence the Rao--Blackwellized weighting reduces to a
weighted difference,
\[
    \frac{\rho_i(S)}
    {q(S)+q(S^c)}
    \left[
        \{u(S)-h(S)\}-\{u(S^c)-h(S^c)\}
    \right].
\]
When \(h\equiv0\), this is precisely the common complementary-contribution
form based on \(u(S)-u(S^c)\). From the present viewpoint, that form is not only
natural for such sign-symmetric targets; among feasible ordered complement-pair
weightings, it is the first-order variance-minimizing choice.

\end{document}